\documentclass[11pt]{article}
\usepackage[margin=1.3in]{geometry}
\usepackage{comment}

\usepackage{comment}
\usepackage{amsfonts, amssymb,amsmath,amsthm}
\RequirePackage{natbib}
\RequirePackage[colorlinks,citecolor=blue,urlcolor=blue]{hyperref}
\RequirePackage{graphicx}
\usepackage{txfonts}
\usepackage{xcolor}

\usepackage[utf8]{inputenc} 
\usepackage[T1]{fontenc}    
\usepackage{hyperref}       
\usepackage{url}            
\usepackage{booktabs}       
\usepackage{nicefrac}       
\usepackage{microtype}      
\usepackage{centernot}
\usepackage[boxruled,
linesnumbered,
commentsnumbered]{algorithm2e}
\usepackage[sans]{dsfont}

\usepackage{mathtools}
\usepackage{graphicx}
\usepackage{enumerate,xspace}
\usepackage{fancyhdr}
\usepackage{comment}
\usepackage{parskip}
\usepackage{float}
\usepackage{longtable}
\usepackage{tablefootnote}
\usepackage{url}
\usepackage{standalone}
\restylefloat{table}

\newtheorem{remark}{Remark}

\usepackage{tikz}
\usetikzlibrary{matrix,arrows,calc,shapes,backgrounds}
\usetikzlibrary{shapes.callouts,decorations.text}
\usetikzlibrary{shapes.misc}

\newcommand{\ddk}{\Delta d_k}
\newcommand{\drk}{\Delta R_k}

\newcommand{\bnorm}[1]{\left\|{#1} \right\|}

\newcommand{\pth}[1]{\left( #1 \right)}
\newcommand{\qth}[1]{\left[ #1 \right]}

\newcommand{\iprod}[2]{\left \langle #1, #2 \right\rangle}

\newcommand{\E}{\mathbb{E}}

\newcommand{\n}[1]{%
\ifthenelse{\equal{#1}{}}{\frac{1}{n}}{\frac{1}{n^{#1}}}%
}

\newcommand{\Trans}{\mathrm{T}}

\renewcommand{\epsilon}{\varepsilon}

\newcommand{\bX}{\mathbf{X}}

\newcommand{\bx}{\mathbf{x}}
\newcommand{\by}{\mathbf{y}}

\newcommand{\N}{\mathcal{N}}

\newcommand{\R}{\mathbb{R}}

\DeclareMathOperator*{\argmin}{argmin}

\SetCommentSty{commfont}
\SetKwComment{Comm}{$\rhd\ $}{}
\newcommand{\hfstack}{\hat f_{\text{stack}}}
\newcommand{\hfbest}{\hat f_{\text{best}}}
\newcommand{\bhfstack}{\mathbf{\hat f}_{\text{stack}}}
\newcommand{\bhfbest}{\mathbf{\hat f}_{\text{best}}}
\newcommand{\bff}{\mathbf{f}}
\newcommand{\biprodl}{\iprod{\by}{\boldsymbol{\psi}_l}}

\newcommand{\cov}{\text{cov}}

\usepackage{prettyref,xspace}
\newrefformat{eq}{(\ref{#1})}
\newrefformat{thm}{Theorem~\ref{#1}}
\newrefformat{fact}{Fact~\ref{#1}}
\newrefformat{sec}{Section~\ref{#1}}
\newrefformat{algo}{Algorithm~\ref{#1}}
\newrefformat{fig}{Fig.~\ref{#1}}
\newrefformat{tab}{Table~\ref{#1}}
\newrefformat{rmk}{Remark~\ref{#1}}
\newrefformat{def}{Definition~\ref{#1}}
\newrefformat{cor}{Corollary~\ref{#1}}
\newrefformat{lmm}{Lemma~\ref{#1}}
\newrefformat{prop}{Proposition~\ref{#1}}
\newrefformat{app}{Appendix~\ref{#1}}
\newrefformat{ex}{Example~\ref{#1}}
\newrefformat{ass}{Assumption~\ref{#1}}

\mathtoolsset{showonlyrefs,showmanualtags}

\def\bx{\mathbf{x}}

\def\bX{\mathbf{X}}

\DeclarePairedDelimiterX{\infdivx}[2]{(}{)}{%
  #1\;\delimsize\|\;#2%
}

\usepackage{mathrsfs}

\newtheorem{theorem}{Theorem}[section]
\newtheorem{lemma}[theorem]{Lemma}
\theoremstyle{remark}

\usepackage{stmaryrd}

\usepackage{comment}
\usepackage{subcaption}

\begin{document}

\title{Error Reduction from Stacked Regressions\thanks{
    The authors would like to thank Matias Cattaneo, Jianqing Fan, Giles Hooker, Vladimir Koltchinskii, Cheng Mao, Jonathan Siegel, Min Xu, and Dana Yang for insightful discussions. Klusowski gratefully acknowledges financial support from the National Science Foundation through CAREER DMS-2239448, DMS-2054808, and HDR TRIPODS CCF-1934924. Tan gratefully acknowledges support from NUS Start-up Grant A-8000448-00-00.
}}
\author{Xin Chen\thanks{Department of Operations Research and Financial Engineering, Princeton University} \and
  Jason Klusowski\footnotemark[2]  \and
  Yan Shuo Tan\thanks{Department of Statistics and Data Science,
National University of Singapore}
}
\maketitle

\begin{abstract}
  Stacking regressions is an ensemble technique that forms linear combinations of different regression estimators to enhance predictive accuracy. The conventional approach uses cross-validation data to generate predictions from the constituent estimators, and least-squares with nonnegativity constraints to learn the combination weights. In this paper, we learn these weights analogously by minimizing a regularized version of the empirical risk subject to a nonnegativity constraint. When the constituent estimators are linear least-squares projections onto nested subspaces separated by at least three dimensions, we show that thanks to an adaptive shrinkage effect, the resulting stacked estimator has strictly smaller population risk than best single estimator among them, with more significant gains when the signal-to-noise ratio is small. Here ``best'' refers to an estimator that minimizes a model selection criterion such as AIC or BIC. In other words, in this setting, the best single estimator is inadmissible. Because the optimization problem can be reformulated as isotonic regression, the stacked estimator requires the same order of computation as the best single estimator, making it an attractive alternative in terms of both performance and implementation.
\end{abstract}

\section{Introduction}

When performing regression, an analyst rarely knows the true model a priori, and instead starts by deriving a collection of candidate models $ \hat\mu_1, \hat\mu_2, \dots, \hat\mu_M$.
Classically, the next step in the process is to select the best model based on data-dependent criteria such as complexity (e.g., AIC or BIC) or out-of-sample error (e.g., cross-validation).
Such procedures come under the umbrella of model selection, and have been extensively studied \citep{Hastie-Tibshirani-Friedman2009_book}.

If predictive performance is the primary consideration,
\citet{wolpert1992stacked} realized that an alternate approach may be more fruitful.
Instead of selecting the best single model, one may use predictions from these estimators as inputs for another (combined) model, a scheme he called \emph{stacked generalizations}.
\cite{breiman1996stacking} was able to operationalize Wolpert's stacking idea by restricting the combined models to be of the form
\begin{equation} \label{eq:stacking}
    \hat f_{\text{stack}}(\bx) = \sum_{k=1}^M \hat\alpha_k\hat\mu_k(\bx).
\end{equation}
and learning the weights $\hat\alpha_1,\hat\alpha_2,\ldots,\hat\alpha_M$ using cross-validation.
Through extensive experiments on real-world and simulated data, he consistently observed the following:
\begin{enumerate}[(i)]
    \item The stacked model \eqref{eq:stacking} has lower mean squared error than the single model having lowest cross-validation error.
    \item The optimal weights sum to approximately one, despite no explicit constraint to do so.
\end{enumerate}
These initial results have proved to be quite robust and have been reproduced in many different settings. Indeed, stacking has found widespread applications in industry (often going by the name \emph{blending}), and has also been a component of several successful solutions to data science competitions, including some on the platform Kaggle and, perhaps most famously, the Netflix Prize \citep{koren2009bellkor}.
It is therefore unfortunate that the beneficial impact of stacking has yet to be properly understood theoretically.

In our paper, we focus on a special setting in order to better illustrate these stylized features of stacking. 
We will focus on the case when the base estimators comprise a nested sequence of regression models.
Such a sequence arises naturally in the context of sieve estimation (series, spline, polynomial, wavelet), certain autoregressive models, stepwise regression, and decision tree pruning (e.g., CART).
Hence, it is fairly general.
In such a setting, we can rigorously prove both of Breiman's empirical observations (i) and (ii) for a variant of stacking that estimates the stacking weights via a regularized mean squared error. 
We show in fact that \emph{stacking performs model selection and shrinkage}. 
This allows us to lower bound the reduction in population risk of the stacked model, relative to that of the best single model selected using an AIC-like criterion, and to quantify how this depends on the signal strength and sample size.

Our proof relies on a novel connection between the optimization program and isotonic regression.
This connection also allows our version of the stacked regression problem to be solved efficiently.

The problem of how to linearly combine multiple estimators of the same target, especially in the context of nonparametric regression, has been studied in the statistical literature under the name of \emph{aggregation} \citep{juditsky2000functional,nemirovski2000topics,tsybakov2003optimal}.
Due perhaps to the difficulty of analyzing $\hat f_{\textnormal{stack}}$, past works have defined and studied other formulations of the weights $\hat\alpha_1,\ldots,\hat\alpha_M$ and have focused on showing that the aggregate model does not perform too much worse compared to an oracle benchmark.
They hence neither study stacking as used in practice, nor do they explain how stacking, or aggregating models more generally, can uniformly improve upon model selection.
Furthermore, in the context of nested regressions, we provide both simulation results showing that the stacked model outperforms other aggregation techniques as well as mathematical explanations for why this might be the case.

\section{Preliminaries}
Throughout the paper, we consider a nonparametric regression model. Suppose we have collected data $(\bx_1, y_1), (\bx_2, y_2), \dots, (\bx_n, y_n) $ satisfying
\begin{equation}\label{eq:model}
    y_i=f(\bx_i) + \sigma \varepsilon_i, \quad i = 1, 2, \dots, n, 
\end{equation}

where $y_i \in \mathbb R$ is the label of the $i$-th observation, $\bx_i\in \R^d$ is the $d$-dimensional feature vector of the $i$-th observation, $\varepsilon_i \overset{\mathrm{iid}}{\sim} \N(0,1) $ is the $i$-th unobserved noise variable following the standard normal distribution, and $f$ is the unknown regression function to estimate. We consider the fixed design problem where the $\bx_i$'s are deterministic. We also assume that the noise level $ \sigma^2 $ is known a priori.

For a real-valued function $ f $, we define
$ \mathbf{f} = (f(\bx_1), \dots, f(\bx_n))^{\Trans} $
to be the $ n\times 1$ vector of $ f $
evaluated at the design points $\mathbf{X}=(\bx_1, \dots, \bx_n)^\Trans \in
\mathbb{R}^{n\times d} $. For real-valued functions $f$ and $g$, we let
\begin{align*}
    \bnorm{\mathbf{f}}^2= \frac{1}{n}\sum_{i=1}^n\pth{f(\bx_i)}^2, \qquad \iprod{\mathbf{f}}{\mathbf{g}} = \frac{1}{n}\sum_{i=1}^nf(\bx_i)g(\bx_i)
\end{align*}
denote the squared empirical norm and empirical inner product, respectively. The response vector $\mathbf{y} = (y_1,y_2,\dots,y_n)^\Trans$ is viewed as a relation, defined on the design matrix $ \mathbf{X}$, that associates $\bx_i$ with $y(\bx_i)=y_i$. Thus, we write $\bnorm{\mathbf{y}-\mathbf{f}}^2 = \frac{1}{n}\sum_{i=1}^n \pth{y_i - f(\bx_i)}^2$ and $\iprod{\mathbf{y}}{\mathbf{f}} = \frac{1}{n}\sum_{i=1}^ny_if(\bx_i)$. 
The $\ell_0$ norm of a vector, denoted by $ \|\cdot\|_{\ell_0} $, is defined as the number of its nonzero components. 
We write $ (z)_{+} = \max\{ 0, z\}  $ for the positive part of a real number $ z $.

\subsection{Nested Regression Models}

In this paper, we investigate the problem of approximating a target variable $y$ by stacking a sequence of least squares projections onto nested subspaces corresponding to a different level of approximation.
Specifically, we consider a fixed sequence of 
orthonormal basis functions $\{\psi_l\}_{l\in\mathbb{N}}$, that is, $ \|\boldsymbol{\psi}_l\|^2 = 1 $ and $ \langle \boldsymbol{\psi}_l, \boldsymbol{\psi}_{l'}\rangle = 0  $ for all $ l \neq l '$.
Let $ A_1 \subseteq A_2 \subseteq \cdots \subseteq A_M $ be a nested sequence of index sets and define the $ k $-th linear subspace $ \mathcal{A}_k $ as $ \text{span}(\{\psi_l: l \in A_k\}) $. The nested structure implies that each subsequent subspace of dimension $ d_k = |A_k| $ expands the representation space of the previous one. Let $\hat \mu_k(\bx)$ represent the projection of the response values $\mathbf{y} = (y_1, y_2, \dots, y_n)^{\Trans}$ onto the $k$-th linear subspace $ \mathcal{A}_k $, given explicitly by
\begin{equation} \label{eq:expansion}
\hat \mu_k(\bx) = \sum_{l\in A_k}\langle \mathbf{y}, \boldsymbol{\psi}_l \rangle \psi_l(\bx).
\end{equation}
We define the mean of $\hat\mu_k$ by $f_k(\bx) = \sum_{l\in A_k}\langle \mathbf{f}, \boldsymbol{\psi}_l \rangle \psi_l(\bx)$ and the empirical risk of $\hat\mu_k $ by $ R_k = \|\mathbf{y}-\hat{\boldsymbol{\mu}}_k\|^2$. For notational convenience in some of the upcoming expressions, we define $ d_0 = 0 $, $ R_0 = \|\mathbf{y}\|^2 $, and $ \hat\mu_0(\bx) \equiv 0 $, corresponding to the \emph{null model}.
Without loss of generality, we shall assume that the models are distinct, so that the set inclusions above are strict. In this case, because the representation spaces are nested, we have $ R_0 \geq R_1 \geq \cdots \geq R_M $, with strict inequality holding almost surely.\footnote{This is due to the fact that $(n/\sigma^2)(R_0-R_k) \sim \chi^2(d_k, \theta)$ follows a (continuous) noncentral chi-squared distribution with $ d_k $ degrees of freedom and noncentrality parameter $ \theta = n\big(\|\mathbf{f}\|^2-\|\mathbf{f}-\mathbf{f}_k\|^2\big)/\sigma^2 $.} We therefore assume throughout the paper that $R_k < R_{k-1} $ for all $ k$, an event that holds with probability one.

\subsection{Examples of Nested Regressions}

While there are many families of nested regression models that one could potentially stack (as previously mentioned), Breiman focused on two canonical forms that we now describe. 

\subsubsection{Stepwise Regressions}

Using validation data, perform stepwise deletion (elimination) on a linear model with $ d $ variables. That is, start with all $ d $ variables and then discard the one that reduces the empirical risk the least, in succession, until we are left with a simple linear model. This produces $ d $ (nested) linear models such that the $k$-th model has $d-k+1$ nonzero coefficients. These nested regressions can then be stacked together using the training data. 

\subsubsection{Decision Trees} \label{sec:tree}

A decision tree model can be viewed as a least squares projection onto a set of piecewise constant functions. The set of piecewise-constant functions is defined in terms of a (possibly data-dependent) partition of the input space induced by the splits in the tree. The tree output can therefore be written in the form \eqref{eq:expansion}, where the basis elements $ \psi_l $ are indexed by the internal nodes of the tree and depend only on the induced tree partition; see for example, \citep[Lemma 2.1]{cattaneo2022convergence}. 

Nested decision trees serve as a canonical example of nested regressions, since a refinement of a partition naturally defines an ordering. Each tree $T_k$ is then a subtree of its successor: $T_1  \preceq T_2  \preceq \cdots  \preceq T_M$, where $ T^{\prime}  \preceq T $ means that $T^{\prime}$ can be obtained from $ T $ by pruning.  As each tree model is linear, $d_k $ is equal to the number of internal nodes in its corresponding tree $ T_k $.

To form the nested sequence of trees in practice, \cite{breiman1996stacking} recommended growing a large tree $ T_M $ with $ M $ terminal nodes ($ M-1$ internal nodes) and then pruning upwards so that $ T_k $ is the subtree of $ T_M $ having the smallest empirical risk among all subtrees of $ T_M $ with $ k $ terminal nodes ($k-1$ internal nodes).
This construction implies that the number of models in the stack could potentially be quite large, as $M$ can grow polynomially with $n$.

\subsection{Accounting for Data-adaptive Subspaces}
To preserve the integrity of our theory, we require the sequence of subspaces to be statistically independent of the responses.
On the other hand, subset regression and decision trees as implemented in practice both select these subspaces data-adaptively, which in the case of decision trees corresponds to selecting splits based on a data-dependent criterion, such as impurity decrease for CART.
In order to maintain independence, one may perform data splitting, using a portion of the data to select the subspaces and the rest of the data to fit the stacked model.
This assumption, sometimes known as the honesty condition, is often made in the literature for theoretical tractability \citep{athey2016recursive}.

\section{Learning the Stacking Weights}

For a sequence of weights $\boldsymbol{\alpha} = (\alpha_1,\alpha_2,\ldots,\alpha_M)$, the empirical risk of the corresponding stacked model is given by
\begin{equation} \label{eq:emp_risk}
    R(\boldsymbol{\alpha}) = \Bigg\|\mathbf{y} - \sum_{k=1}^M \alpha_k \boldsymbol{\hat \mu}_k \Bigg\|^2.
\end{equation}

\subsection{Breiman's Stacking}

It is easy to see that directly optimizing \eqref{eq:emp_risk} leads to a weight vector that is a point mass on the most complex model $\hat\mu_M$, and so \citet{breiman1996stacking} advocated for optimizing the weights with respect to cross-validation error. 
Even so, optimizing over $\boldsymbol{\alpha} $ without constraints often leads to stacked models that overfit. Recognizing this, Breiman advocated for some form of regularization. Because the $ \hat\mu_k $ values are highly correlated—since they aim to predict the same outcome—he initially experimented with ridge constraints, i.e., $ \sum_{k=1}^M \alpha^2_k = t $. Ultimately, however, he settled on nonnegative constraints, i.e., $ \alpha_k \geq 0 $ for $ k = 1, 2, \dots, M $, for their seemingly best performance. Breiman argued that the reason nonnegative weights (that also sum to one) work well is because they make the stacked estimator ``interpolating'' in the sense that $ \min_k \hat\mu_k(\bx) \leq \hat f_{\text{stack}}(\bx) \leq \max_k \hat\mu_k(\bx) $.
On the other hand, he did not impose the constraint that the weights had to sum to unity as this seemed to have little effect on performance.

While Breiman used cross-validation to learn the weights of combination, we can quickly see that such an approach is not appropriate for the fixed design setting because the cross-validation error does not give an unbiased estimate of the in-sample error.
As such, we will instead study a variant of stacking that optimizes a regularized version of \eqref{eq:emp_risk} based on model degrees of freedom and dimension.
We believe this to be a small departure from Breiman's formulation of stacking, as the asymptotic model selection properties of cross-validation and penalized empirical risk have been shown to be similar under certain conditions \citep{stone1977asymptotic, shao1997asymptotic}. Because of this, our theory can and indeed does explain some of the empirical findings of Breiman.

\subsection{Dimension of a Model}

We define the dimension of a function $ \mu $ by
$$
\text{dim}(\mu) = \min_{A \subseteq \mathbb{N}}\big\{ |A| : \mu \in \text{span}(\{\psi_l : l \in A \})\big\},
$$
that is, the fewest number of basis elements needed to represent the function. If no such representation exists, then $ \text{dim}(\mu) = \infty $, and if $ \mu \equiv 0 $, then $ \text{dim}(\mu) = 0 $.
Using this definition and setting $ \alpha_0 = 1 $, it is easy to see that a stacked model of nested regressions with fixed weight vector $ \boldsymbol{\alpha} $ has dimension 
$ \max_{k = 0, 1, \dots, M}\big\{ d_k : \alpha_k \neq 0\big \}.
$

\subsection{Degrees of Freedom of a Stacked Model}

In fixed design regression, the degrees of freedom of an estimator $\hat f$ is defined as
$$
\text{df}(\hat f) = \sum_{i=1}^n \frac{\text{cov}(y_i, \hat f(\bx_i))}{\sigma^2}.
$$

It quantifies the optimism bias of the empirical risk with respect to the mean squared error \citep[Section 7.5 \& 7.7]{Hastie-Tibshirani-Friedman2009_book}, seen through the relation
\begin{equation}\label{eq:unbiased_mse_est}
    \mathbb{E}\big[\|\mathbf{f}-\hat{\mathbf{f}}\|^2\big] = \E\big[\|\mathbf{y}-\hat{\mathbf{f}}\|^2\big] + \frac{2\sigma^2\text{df}(\hat f)}{n} - \sigma^2.
\end{equation}
Although it is generally a population level quantity, in the case of linear regression, it is equivalent to the number of regressors which is known a priori; in particular, $ \text{df}(\hat \mu_k) = d_k $ for each $k=1,2,\ldots,M$.
Moreover, because covariance is linear in each of its arguments, the degrees of freedom for a stacked model with fixed weight vector $\boldsymbol{\alpha}$ likewise has a known formula.
It is equal to
\begin{equation}
    \text{df}(\boldsymbol{\alpha}) = \sum_{k=1}^M \alpha_k d_k.
\end{equation}

\subsection{Stacking via Complexity Penalization}

The identity \eqref{eq:unbiased_mse_est} immediately suggests that an unbiased estimator of the population risk of the stacked model is $ R(\boldsymbol{\alpha})+(2\sigma^2/n)\text{df}(\boldsymbol{\alpha}) $, which we may then optimize, subject to $ \boldsymbol{\alpha} \geq \mathbf{0} $, to learn the stacking weights. However, a major problem with this estimator is that unbiasedness holds as long as the stacking weights are nonadaptive. So there is no guarantee that  $ R(\hat{\boldsymbol{\alpha}})+(2\sigma^2/n)\text{df}(\hat{\boldsymbol{\alpha}}) $ is an unbiased estimator of the expected population risk of the stacked model with adaptive weights $\hat{\boldsymbol{\alpha}} $, where $ \hat{\boldsymbol{\alpha}} $ minimizes $ R(\boldsymbol{\alpha})+(2\sigma^2/n)\text{df}(\boldsymbol{\alpha}) $ subject to $ \boldsymbol{\alpha} \geq \mathbf{0} $. In fact, for the stacked model with these weights, it turns out that
\begin{equation} \label{loss:full}
\mathbb{E}\big[\|\mathbf{f}-\hat{\mathbf{f}}_{\text{stack}}\|^2\big] = \mathbb{E}\Bigg[R(\hat{\boldsymbol{\alpha}}) + \frac{2\sigma^2}{n}\text{df}(\hat{\boldsymbol{\alpha}})+ \frac{4\sigma^2}{n}\|\hat{\boldsymbol{\alpha}}\|_{\ell_0} - \frac{4\sigma^2}{n} \hat{\boldsymbol{\alpha}}^{\Trans}\hat{\boldsymbol{\alpha}}_0 - \sigma^2  \Bigg],
\end{equation}
where $ \hat{\boldsymbol{\alpha}}_0(k) = \#\{j \leq k: \hat\alpha_j \neq 0 \} $. It therefore becomes clear that a good proxy for the population risk of a stacked model with adaptive weights should involve more than just the empirical risk  $R(\boldsymbol{\alpha})$ and degrees of freedom $\text{df}(\boldsymbol{\alpha})$, and even be discontinuous in the weights. Inspired by this observation, we propose learning the weights via solving the following program:
\begin{equation} \label{loss:lasso}
\begin{aligned}
& \text{minimize} \quad R(\boldsymbol{\alpha}) + \frac{2\tau\sigma^2}{n}\text{df}(\boldsymbol{\alpha}) + \frac{(\lambda-\tau)^2_{+}}{\lambda}\frac{\sigma^2}{n}\text{dim}(\boldsymbol{\alpha})  \\
& \text{subject to} \quad \boldsymbol{\alpha} \geq \mathbf{0}, \\
\end{aligned}
\end{equation}
where $ \text{dim}(\boldsymbol{\alpha}) = \max_{k = 0, 1, \dots, M}\big\{ d_k : \alpha_k \neq 0\big\} $ with $ \alpha_0 = 1 $, and $ \tau, \lambda > 0 $ are tuning parameters chosen by the user. It turns out that the solution $\hat{\boldsymbol{\alpha}}$ to program \eqref{loss:lasso} satisfies $ \sum_{k=1}^M \hat\alpha_k \leq 1 $, in which case $ \text{dim}(\hat{\boldsymbol{\alpha}}) \geq \max\{ \|\hat{\boldsymbol{\alpha}}\|_{\ell_0}, \text{df}(\hat{\boldsymbol{\alpha}}) \} $. Thus, among weight sequences satisfying $ \sum_{k=1}^M \alpha_k \leq 1 $, we see that $\text{dim}(\boldsymbol{\alpha})$ is a stronger penalty than either $ \|\boldsymbol{\alpha}\|_{\ell_0} $ or $ \text{df}(\boldsymbol{\alpha}) $.

The objective function \eqref{loss:lasso} modulates the empirical risk of the stacked ensemble in two distinct ways. Firstly, similar to $\ell_1$ regularization in nonnegative Lasso, it shrinks the weights according to the size of the models, as reflected by the term $\text{df}(\boldsymbol{\alpha})$. Secondly, similar to $\ell_0$ regularization, the objective function also takes into account the number of basis functions needed to represent the model through the term $\text{dim}(\boldsymbol{\alpha})$. This promotes model parsimony by encouraging a sparse selection of models.

Crucially, while program \eqref{loss:lasso} is nonconvex due to $ \text{dim}(\boldsymbol{\alpha}) $, we can leverage the nested structure of the estimators and reduce it to an isotonic regression problem, which can be solved in $O(M)$ time, i.e., linear time in the number of models. Note that this is orderwise the same complexity as finding the data-selected best single model \eqref{eq:best}. This means that we could potentially stack thousands of models and not worry about computational issues.

It is worth noting that, like conventional stacking, we do not impose any sum constraint on the weights, i.e., stacking forms a conical combination of models. Indeed, as we shall see, such a constraint is essentially superfluous as the unconstrained solution always satisfies $ \sum_{k=1}^M \hat\alpha_k 
 < 1 $ with near equality in most cases. This fact also corroborates with Breiman's experiments, in particular, when stacking nested decision trees or linear regressions determined by stepwise deletion. We have also found both empirically and theoretically that enforcing the equality constraint can lead to inferior performance, as otherwise one is limiting the potentially beneficial effects of shrinkage.

\subsection{Data-selected Best Single Model}

In our analysis, we will compare the performance of the stacked model \eqref{loss:lasso} to that of the \emph{data-selected best single model}, defined as $ \hat f_{\text{best}}(\bx) = \hat \mu_{\hat m}(\bx) $, where
\begin{equation}\label{eq:best}
\hat m \in \argmin_{k = 0, 1, \dots, M} \; \|\mathbf{y}-\hat{\boldsymbol{\mu}}_k\|^2 + \lambda\frac{\sigma^2 d_k}{n},
\end{equation}
where $\lambda > 0$ is a tuning parameter whose value we will take to be the same in \eqref{loss:lasso} and \eqref{eq:best}.
If ties exist in \eqref{eq:best}, we choose the smallest $ \hat m $.
Note that we have also included the null model (i.e., $\hat\mu_0 \equiv 0 $) as a candidate model, which will simplify some of the forthcoming theory. The presence of the penalty term $ \lambda\sigma^2 d_k/n $ in \eqref{eq:best} is crucial for modulating the complexity of the chosen model. For example, without it, due to the nested structure of the models, the data-selected best single model is the most complex model, $ \hat \mu_M(\bx) $, as it has the highest goodness-of-fit. In this case, the population risk $ \mathbb{E}\big[\|\mathbf{f}-\hat{\mathbf{f}}_{\text{best}}\|^2\big] = \|\mathbf{f}-\mathbf{f}_M\|^2 + \sigma^2 d_M/n $ can be very large if $ d_M$ has the same order as $ n$, even though $ \|\mathbf{f}-\mathbf{f}_M\|^2 $ may be very small.

\begin{remark}
An equivalent way of formulating program \eqref{loss:lasso} is to minimize $ R(\boldsymbol{\alpha}) + (2\tau\sigma^2/n)\text{df}(\boldsymbol{\alpha}) $ subject to the constraints $ \boldsymbol{\alpha} \geq \mathbf{0} $ and $ \text{dim}(\boldsymbol{\alpha}) \leq \text{dim}(\hat f_{\text{best}}) = d_{\hat m} $. If $\tau = 0$, stacking becomes data-selected best single model selection according to \eqref{eq:best}.
\end{remark}

The criterion $ R_k + \lambda \sigma^2 d_k/n $ holds significance in various model selection contexts. Specifically, for $ \lambda = 2 $, it matches Mallows's $C_p$. In the present Gaussian linear regression setting, if $ \lambda = 2 $, it is equivalent to the Akaike Information Criterion (AIC) and Stein's Unbiased Risk Estimate (SURE) \citep{stein1981estimation}, and if $ \lambda = \log(n) $, it corresponds to the Bayesian Information Criterion (BIC). These model selection criteria operate on a single model at a time, seeking the one that optimally balances goodness-of-fit and complexity. In contrast, stacking leverages the diverse strengths of multiple models through an optimal linear combination.

Despite potential performance limitations in certain settings (e.g., when the signal-to-noise ratio $\|\mathbf{f}\|^2/\sigma^2$ is small), AIC and BIC remain widely used across many applied disciplines. For example, AIC and BIC continue to be standard model selection criteria in econometrics \citep{zhang2015cross}, biology \citep{dziak2019sensitivity}, and psychology \citep{vrieze2012model}. These fields, including many others, place high value on model interpretability, which can be diminished with model combining approaches. Our forthcoming results provide a quantitative basis for navigating the inherent trade-off between accuracy and interpretability when deciding between single model selection and ensemble methods.

\subsection{Stacking Versus Pruning Decision Trees}

When the nested subspaces arise from a sequence of decision trees obtained via Breiman's approach (see Section \ref{sec:tree}), model selection via \eqref{eq:best} corresponds exactly to cost-complexity pruning
 \citep[Section 9.2.2]{Hastie-Tibshirani-Friedman2009_book}. 
Meanwhile, the stacked model $\hat f_{\text{stack}}$ also takes the form of a decision tree with the same structure as the largest tree in the stack. 
It consists of $\text{dim}(\hat f_{\text{stack}})$ internal nodes and outputs $\hat f_{\text{stack}}(\bx) = \sum_{k=1}^M\hat\alpha_k \overline y_k$, where $\overline y_k$ represents the output of tree $T_k$ at $\bx$, which is the sample mean of observations in the cell containing $\bx$.
Since for $k=1,2,\ldots,M$, the leaf node of $T_k$ containing $\bx$ is an ancestor of that containing $\bx$ in $T_M$, this has the effect of shrinking the predictions over each leaf in $T_M$ to its ancestors' values, similar to the method of \citet{agarwal2022hierarchical}.

\section{Main Results}

Our main result is that stacked model with weights from \eqref{loss:lasso} \emph{strictly} outperforms the data-selected best single model \eqref{eq:best} provided the dimensions of the individual models differ by a constant. Put another way, in the parlance of statistical decision theory, the data-selected best single model $\hat f_{\text{best}}$ (and by implication, the estimator selected by AIC and BIC) is \emph{inadmissible}. 

We also provide an explicit lower bound on the difference between their population risks, which reveals that more substantial performance gains occur in settings with weak signals and small sample sizes. Evidently, this gap tends to be larger when the signal-to-noise ratio $ \|\mathbf{f}\|^2/\sigma^2 $ or sample size $n$ is small—agreeing with the adage that ensemble methods tend to work particularly well in these settings, whereas the data-selected best single model could potentially be unstable.

We emphasize again that both $\hat f_{\text{best}}$ and $\hat f_{\text{stack}}$ have the same computational complexity, that is, $O(M)$, which means stacking offers improved performance at no additional computational cost.

\begin{theorem}\label{thm:main}
Suppose $ 0 < \tau < 2 $ and $d_k \geq d_{k-1} + 4/(2-\tau)$ for all $ k $. The population risk of the stacked model with weights from \eqref{loss:lasso} is strictly less than the population risk of the data-selected best single model \eqref{eq:best}; furthermore, if $d_k \geq d_{k-1} + 5/(2-\tau)$ for all $ k $, there exists a universal constant $ C > 0 $ such that 
\begin{equation} \label{eq:main_inequality}
\mathbb{E}\big[\|\mathbf{f}-\hat{\mathbf{f}}_{\text{stack}}\|^2\big] \leq \mathbb{E}\big[\| \mathbf{f} - \hat{\mathbf{f}}_{\text{best}}\|^2\big] - C\frac{\sigma^2}{n}\frac{d_1^2\tau(2-\tau)}{d_1+n(\|\mathbf{f}\|^2/\sigma^2)}.
\end{equation}
\end{theorem}

\begin{remark}\label{rmk:ddk}
Since the $ d_k $ are integers, the hypotheses of Theorem \ref{thm:main} become $ d_k \geq d_{k-1} + 3 $ when $ 0 < \tau \leq 1/3 $. Thus, a sufficient condition for the stacked model to be provably superior to the data-selected best single model is that the constituent models differ by at least three dimensions. We presently do not know if the condition is also necessary. If we knew the theoretically optimal $ \tau $, which involves unknown population level quantities, then the condition can be improved to $ d_k \geq d_{k-1} + 2 $.
\end{remark}

The proof hinges on a few novel ingredients that we briefly introduce here, and will explain further in the next section.
First, the nestedness of the base estimators and the nonnegative weight constraints allow us to recast the optimization program \eqref{loss:lasso} as isotonic regression (see Lemma \ref{lem:equi}).
Second, we then build on known properties of solutions to isotonic regression to obtain an explicit representation of $\hfstack$ and compare it to that of $\hfbest$ (see Theorem \ref{thm:fstack}).
To compare the population risks of the two representations, we apply an extension of Stein's Lemma for discontinuous functions \citep{tibshirani2015degrees}, which allows us to decompose the model degrees of freedom into two portions, one of which is attributable to the model selection mechanism (Tibshirani called this contribution the \emph{search degrees of freedom}). This proof strategy yields the following lower bound for the population risk gap
\begin{equation} \label{eq:improve}
\begin{aligned}
& \mathbb{E}\big[\| \mathbf{f} - \hat{\mathbf{f}}_{\text{best}}\|^2\big] -  \mathbb{E}\big[\|\mathbf{f}-\hat{\mathbf{f}}_{\text{stack}}\|^2\big] \\ & \qquad \geq 
\frac{\sigma^2\tau(2-\tau)}{n}\mathbb{E}\Bigg[\min_{1 \leq k \leq M}\frac{(d_k-4k/(2-\tau))^2}{(n/\sigma^2)(R_0-R_{k})}\Bigg] + 2\min\Big\{1, \frac{\tau}{\lambda}\Big\}\frac{\sigma^2}{n}\text{sdf}(\hat f_{\text{best}}).
\end{aligned}
\end{equation}
Third, to reveal how the population risk gap depends on the model parameters, we apply a martingale argument to further lower bound \eqref{eq:improve}, yielding \eqref{eq:main_inequality} in Theorem \ref{thm:main}. To see why martingale theory enters the discussion, note that, by Jensen's inequality, 
\begin{equation} \label{eq:maxbound1}
\mathbb{E}\Bigg[\min_{1 \leq k \leq M}\frac{(d_k-4k/(2-\tau))^2}{(n/\sigma^2)(R_0-R_{k})}\Bigg] \geq \Bigg(\mathbb{E}\Bigg[\max_{1 \leq k \leq M}\frac{(n/\sigma^2)(R_0-R_{k})}{(d_k-4k/(2-\tau))^2}\Bigg]\Bigg)^{-1}.
\end{equation}
We then use a powerful inequality due to \cite{shorack1976inequalities} to upper bound the maximum of the zero mean process $\frac{(n/\sigma^2)(R_0-R_{k}-\mathbb{E}[R_0-R_{k}])}{(d_k-4k/(2-\tau))^2}$ by the maximum of a positive submartingale. Finally, we apply Doob's maximal inequality \citep[Theorem 4.4.4]{Durrett_2019} on this submartingale to show that 
\begin{equation} \label{eq:maxbound2}
\mathbb{E}\Bigg[\max_{1 \leq k \leq M}\frac{(n/\sigma^2)(R_0-R_{k})}{(d_k-4k/(2-\tau))^2}\Bigg] \leq \frac{d_1+n(\|\mathbf{f}\|^2/\sigma^2)}{Cd_1^2},
\end{equation}
for some universal constant $C > 0$. The lower bound \eqref{eq:main_inequality} is obtained by combining \eqref{eq:improve}, \eqref{eq:maxbound1}, and \eqref{eq:maxbound2}, since $\text{sdf}(\hat f_{\text{best}}) \geq 0$.

\subsection{Connections to James-Stein Shrinkage}
The second term in \eqref{eq:improve} is in terms of the (strictly positive) search degrees of freedom of $\hat f_{\text{best}}$, which is due to the model selection mechanism of stacking.
Meanwhile, the first term
is reminiscent of the improvement that would arise from applying positive-part James-Stein shrinkage factors \citep{stein1956inadmissibility,james1961estimation, baranchik1964multiple} across the successive model subspace differences $\mathcal{A}_k\setminus \mathcal{A}_{k-1}$, $ k = 1, 2, \dots, M$, viz.,
\begin{equation} \label{eq:JS}
\Bigg(1-\frac{(d_k-d_{k-1})-2}{(n/\sigma^2)(R_{k-1}-R_k)}\Bigg)_{+}\;\big(\hat\mu_k(\bx)-\hat\mu_{k-1}(\bx)\big).
\end{equation}

These similarities are not a coincidence.
Our proof strategy shows that one may think of stacking, in the nested regressions setting, as essentially performing shrinkage to an adaptively selected subset of models. The condition $ d_k-d_{k-1} \geq 3$ needed in Theorem \ref{thm:main} (see Remark \ref{rmk:ddk}) is akin to the necessary and sufficient condition $d_k-d_{k-1} \geq 3$ for \eqref{eq:JS} to lead to an improvement in population risk over $ \hat\mu_k(\bx)-\hat\mu_{k-1}(\bx) $.

However, despite these parallels, stacking does \emph{not} apply statistically independent shrinkage factors, like the ones in \eqref{eq:JS}, across the successive model subspace differences $\mathcal{A}_k\setminus \mathcal{A}_{k-1}$. Instead, stacking applies dependent shrinkage factors across adaptively chosen model subspace differences $\mathcal{A}_k\setminus \mathcal{A}_l$, where $ k > l$; see the forthcoming representation in \eqref{eq:stack}. We will discuss the adaptive shrinkage properties of stacking further in the next few sections.

\section{Reduction to Isotonic Regression}

Due to the nested structure of the models and the nonnegative weight constraint, the problem of determining the stacking weights from problem \eqref{loss:lasso} can be recast as isotonic regression. To see this, let $ \Delta d_k = d_k - d_{k-1} $ and $ \Delta R_k = R_{k-1}-R_k $. Recall that $ R_k < R_{k-1}$ for all $ k $, implying that $ \Delta R_k > 0 $ for all $ k $. Making the change of variables $ \alpha_k = \beta_{k+1} - \beta_k $, for $ k = 1, 2, \dots, M $, with $\beta_{M+1} = 1$, the nonnegativity constraint on the weights $ \alpha_k $ becomes an isotonic constraint on the $\beta_k$. 
Some further algebraic manipulation gives the following equivalence of optimization programs.
\begin{lemma}\label{lem:equi}
    Program \eqref{loss:lasso} is equivalent to:
    \begin{equation} \label{loss:1}
    \begin{aligned}
    & \text{minimize} \quad \sum_{k=1}^M w_k(z_k - \beta_k)^2 + \xi\sum_{k=1}^M \Delta d_k\mathbf{1}(\beta_k \neq 1) \\
    &\text{subject to} \quad \beta_1 \leq \beta_2 \leq \cdots \leq \beta_M \leq 1,
    \end{aligned}
    \end{equation}
    where $ w_k = \Delta R_k > 0  $, $ z_k = (\tau\sigma^2/n)(\Delta d_k/\Delta R_k) > 0 $, and $ \xi = (\sigma^2/n)((\lambda-\tau)^2_{+}/\lambda) $.
\end{lemma}

We shall henceforth refer to the stacking weights $\hat{\boldsymbol{\alpha}}$ as \emph{the} solution to program \eqref{loss:lasso}, since it turns out to be almost surely unique. Crucially, using known properties of isotonic regression, the solution admits a closed form expression in terms of the following nondecreasing minimax sequence \eqref{eq:minmax}, which allows us to connect $\hat f_{\text{stack}}$ and $ \hat f_{\text{best}}$:
\begin{equation} \label{eq:minmax}
\hat \gamma_k = \frac{\sigma^2}{n}\min_{k\leq i\leq M}\max_{0 \leq j < k}\frac{d_{i}-d_{j}}{R_{j}-R_{i}}, \quad k = 1, 2, \dots, M.
\end{equation}
The vector of this sequence is denoted as $\hat{\boldsymbol{\gamma}} = (\hat\gamma_1, \hat\gamma_2, \dots, \hat\gamma_M)^{\Trans}$. In what follows, it will be convenient to define $ \hat\gamma_0 = 0 $ and $ \hat\gamma_{M+1} = \infty $.

\begin{theorem}\label{thm:fstack}
The (almost surely) unique solution to program \eqref{loss:lasso} is
\begin{equation} \label{eq:solution}
\hat\alpha_k = (1-\tau\hat\gamma_k)\mathbf{1}(\hat\gamma_k < \gamma) - (1-\tau\hat\gamma_{k+1})\mathbf{1}(\hat\gamma_{k+1} < \gamma), \quad k = 1, 2, \dots, M,
\end{equation}
where $ \gamma = \min\{ 1/\tau, 1/\lambda\} $, and, consequently, $ \sum_{k=1}^M \hat\alpha_k = (1-\tau\hat\gamma_1)\mathbf{1}(\hat\gamma_1 < \gamma) < 1 $. Furthermore, the stacked model can be written as
\begin{equation} \label{eq:stack}
\hat f_{\text{stack}}(\bx) = \sum_{k=1}^M (\hat\mu_{k}(\bx)-\hat\mu_{k-1}(\bx))(1-\tau\hat \gamma_k)\mathbf{1}(\hat\gamma_k< \gamma),
\end{equation}
and the data-selected best single model can be written as
\begin{equation} \label{eq:best_form}
\hat f_{\text{best}}(\bx) = \sum_{k=1}^M (\hat\mu_{k}(\bx)-\hat\mu_{k-1}(\bx))\mathbf{1}(\hat\gamma_k < 1/\lambda).
\end{equation}
\end{theorem}
According to \eqref{eq:best_form}, the data-selected best single model can be written as a telescoping sum of predictive differences $ \hat\mu_k(\bx)-\hat\mu_{k-1}(\bx) $ across successive submodels, up to the selected model $ \hat \mu_{\hat m}$. In contrast, according to \eqref{eq:stack}, stacking additionally shrinks these predictive differences towards zero, and because $\{ \hat\gamma_k\} $ is an increasing sequence, larger models are shrunk more than smaller models.
The adaptive shrinkage factor $ 1-\tau\hat\gamma_k $ in \eqref{eq:stack} is the driving force behind its superior performance. So in summary, in the setting of nested regressions, \emph{stacking performs both model selection and shrinkage simultaneously.}

\subsection{Implementation}

Because we know the general form of the solution \eqref{eq:solution} in terms of $ \hat{\boldsymbol{\gamma}} $, a quantity free of $ \lambda $ and $ \tau $, we simply need solve for $ \hat{\boldsymbol{\gamma}} $ using the related program
\begin{equation} \label{loss:iso}
\begin{aligned}
& \text{minimize} \quad \sum_{k=1}^M w_k(z_k - \beta_k)^2 \\
&\text{subject to} \quad \beta_1 \leq \beta_2 \leq \cdots \leq \beta_M,
\end{aligned}
\end{equation}
which we recognize as a weighted isotonic regression problem. The solution admits a closed form expression $ \hat{\boldsymbol{\beta}} = \hat{\boldsymbol{\gamma}} $, the minimax sequence \eqref{eq:minmax}, from which the general solution \eqref{eq:solution} can be directly obtained. Program \eqref{loss:iso} can be solved in $O(M)$ time using the Pooled Adjacent Violators Algorithm (PAVA) \citep{barlow1972statistical}, thereby permitting $ M $ to be in the thousands without incurring worrisome computational burden. The \texttt{isoreg()} function from base R
will implement program \eqref{loss:iso}.

Concurrent work by \cite{peng2023model} also explores the connection between aggregation (when the weights sum to one) and isotonic regression in the nested setting. However, it appears that the author was unaware of the explicit expression \eqref{eq:minmax} for weighted isotonic regression, and so pursued a different direction by relaxing the monotonicity constraint in \eqref{loss:iso}, whereby the $ \beta_j $ are decoupled and belong to a fixed interval, such as $ [0, 1] $. The resulting estimator applies statistically independent shrinkage factors (see \eqref{eq:JS}) across the successive model subspace differences $\mathcal{A}_k\setminus \mathcal{A}_{k-1}$, $ k = 1, 2, \dots, M$, in contrast to stacking \eqref{eq:stack}, which applies statistically dependent shrinkage factors across adaptively chosen model subspace differences $\mathcal{A}_k\setminus \mathcal{A}_l$, where $ k > l$. These dependencies make our analysis considerably more delicate. \cite{peng2023model} also focuses on demonstrating asymptotic optimality relative to the oracle best convex combination model (see Section \ref{sec:aggregation}), and does not make comparisons with the data-selected best single model \eqref{eq:best} as we do in the present work.

\section{Model Selection and Shrinkage in Stacking}
\label{sec:model_selection_and_shrinkage}
 
According to Theorem \ref{thm:fstack}, the version of stacking we analyzed performs both model selection and shrinkage simultaneously.
As discussed after Theorem \ref{thm:main} and as will be proved in Appendix \ref{app:proofs} in the supplementary material, the improvement in prediction performance over the data-selected best single model derives almost entirely from shrinking the weights in the telescoping sum \eqref{eq:best_form}.
Since the shrinkage factors are optimized, they naturally resemble James-Stein shrinkage factors, leading to a similar formula for the population risk gap \eqref{eq:improve}.
In this section, we compare shrinkage from stacking with prior work on shrinkage and discuss the generality of this interpretation of stacking.

\citet{agarwal2022hierarchical} introduced a hierarchical shrinkage procedure to regularize decision tree models.
When viewing the tree model as a linear regression onto orthogonal basis elements $\psi_l$ indexed by the internal nodes of the tree, they showed that their procedure was equivalent to performing ridge regression instead of linear regression on this basis.
Since the amount of shrinkage in ridge regression is controlled by a single parameter, theirs is therefore a more constrained form of shrinkage compared to what arises from stacking nested subtrees.
The latter optimizes the amount of shrinkage over $M$ different parameters, each controlling the amount of shrinkage over a separate block of the regression coefficients.
On the other hand, hierarchical shrinkage allows the basis elements $\psi_l$ to be unnormalized, with $\|\psi_l\|^2 = N_l/n$, where $N_l$ is the number of samples contained in node $l$. The ridge regression shrinkage factor for the coefficient of $\psi_l$ is then $\frac{N_l}{N_l + \lambda}$, which means that splits on nodes with fewer samples are penalized more strongly.
A theoretical or empirical performance comparison of the two forms of shrinkage is left for future work.

\citet{agarwal2022hierarchical} also showed empirically that the best performing tree model after shrinkage is usually much deeper than the best performing tree model prior to shrinkage.
In other words, by allowing a more nuanced bias-variance tradeoff, shrinkage allows for a ``larger model'' and makes use of features whose inclusion would not be justified otherwise.
\citet{breiman1996stacking}'s empirical findings for stacked subtrees present a similar story. 
In one experiment he found that, among a collection of $M=50$ subtrees, the stacked tree had $33$ internal nodes (the \emph{dimension of the model} in our terminology), whereas the (underperforming) best single tree had only $22$. 
Breiman was surprised by this, and explained it as the prediction in the stacked tree ``gathering strength'' from the data among ancestor terminal nodes in the smaller stacked subtrees.

Our work formalizes Breiman's notion of ``gathering strength'' as shrinkage, and rigorously proves its benefits.
On the other hand, we are not yet able to show theoretically how the prediction performance of stacking benefits from allowing for a larger model.
The optimization program \eqref{loss:lasso} was explicitly designed so that the additive representations \eqref{eq:stack} and \eqref{eq:best_form} for $\hat f_{\text{stack}}$ and $\hat f_{\text{best}}$ truncate at the same level.
This property is crucial in our proof of Theorem \ref{thm:main} because it allows us to compare the search degrees of freedom of the two estimators, $ \text{sdf}(\hat f_{\text{stack}})$ and  $ \text{sdf}(\hat f_{\text{best}})$, which would otherwise be challenging.
Recall that these quantities arise because we use an extension of Stein's Lemma for discontinuous functions \citep{tibshirani2015degrees}.
We need this to cope with the fact that both $\hat f_{\text{stack}}$ and $\hat f_{\text{best}}$ are discontinuous with respect to the $y$-data---they perform hard thresholding on a data-dependent quantity, arising from the model selection mechanism. 
In general, bounding the search degrees of freedom (e.g., from best subset selection) is known to be quite challenging \citep{mikkelsen2018best, tibshirani2019excess}. 
Truncating the additive representations for $\hat f_{\text{stack}}$ and $\hat f_{\text{best}}$ at the same level places their discontinuities (with respect to the $y$-data) at exactly the same locations, which causes their search degrees of freedom 
to be positive and proportional to each other, that is, $ \text{sdf}(\hat f_{\text{stack}}) = (1-\tau/\lambda)_{+}\text{sdf}(\hat f_{\text{best}}) \geq 0 $. 
Nonetheless, we conjecture that stacking nested regressions without this artificial constraint does lead to a larger and more predictive model.
Some partial results in this direction are presented in the next section.

Furthermore, while we have analyzed stacking in the context of nested regressions, the benefits of stacking have been empirically observed in much more general settings without any nested structure.
The shrinkage effect of stacking likely continues to play some role in improving prediction performance, but it is surely not the only factor at play.
We hypothesize that stacking in these settings benefits from combining different models, some more complex than $\hat f_{\text{best}}$, where each captures separate nuances in the data.

\section{Size of the Stacked Model}

In this section, we will provide theoretical evidence of the aforementioned phenomenon observed by Breiman and show that, in addition to having a sum less than one, for certain stacking weights, the stacked model puts nonzero weight on some constituent models that have higher complexity than the data-selected best single model when $ \lambda = 2 $ (e.g., the model selected by AIC, SURE, or Mallows's $C_p$). To this end, note that the quantity inside the expectation in the right hand side of \eqref{loss:full} is an unbiased estimator of the expected prediction error of the stacked model with adaptive weights $\hat{\boldsymbol{\alpha}} $, where $ \hat{\boldsymbol{\alpha}} $ minimizes $ R(\boldsymbol{\alpha})+(2\sigma^2/n)\text{df}(\boldsymbol{\alpha}) $ subject to $ \boldsymbol{\alpha} \geq \mathbf{0} $. Dropping the negative term involving $ 0 \leq \hat{\boldsymbol{\alpha}}^{\Trans}\hat{\boldsymbol{\alpha}}_0 \ll \|\hat{\boldsymbol{\alpha}}\|_{\ell_0} $, we may thus use 
$$
R(\boldsymbol{\alpha})+\frac{2\sigma^2}{n}\text{df}(\boldsymbol{\alpha}) + \frac{4\sigma^2}{n}\|\boldsymbol{\alpha}\|_{\ell_0},
$$
to hopefully better estimate the population risk of an (adaptive) stacked model, and solve the program
\begin{equation} \label{loss:l0}
\begin{aligned}
& \text{minimize} \quad R(\boldsymbol{\alpha}) + \frac{2\sigma^2}{n}\text{df}(\boldsymbol{\alpha}) + \frac{4\sigma^2}{n}\|\boldsymbol{\alpha}\|_{\ell_0}  \\
& \text{subject to} \quad \boldsymbol{\alpha} \geq \mathbf{0}.
\end{aligned}
\end{equation}

We see that the objective function penalizes the number of included models through $\|\boldsymbol{\alpha}\|_{\ell_0}$ as well as the complexity of the included models through $ \text{df}(\boldsymbol{\alpha}) $.
The optimization problem \eqref{loss:l0} is equivalent to (weighted) reduced isotonic regression \citep{gao2020estimation, haiminen2008algorithms} in which one fits an isotonic sequence subject to a constraint on the number of distinct values it can assume, i.e.,
\begin{equation} \label{loss:l0-equiv}
\begin{aligned}
& \text{minimize} \quad \sum_{k=1}^M w_k(z_k - \beta_k)^2 + \frac{4\sigma^2}{n}\sum_{k=1}^M \mathbf{1}(\beta_k \neq \beta_{k+1}) \\
&\text{subject to} \quad \beta_1 \leq \beta_2 \leq \cdots \leq \beta_M \leq 1,
\end{aligned}
\end{equation}
where $ w_k = \Delta R_k > 0  $ and $ z_k = (\sigma^2/n)(\Delta d_k/\Delta R_k) > 0 $.

Reduced isotonic regression can be implemented in $ O(M^3) $ time using a two-stage approach consisting of Pooled Adjacent Violators Algorithm (PAVA) \citep{barlow1972statistical} followed by dynamic programming for the $K$-segmentation problem \citep{bellman1961approximation}, e.g., see the procedure of \citep{haiminen2008algorithms}. The \texttt{scar} package in R provides a function for reduced isotonic regression.

It turns out that the stacked model with weights from solving \eqref{loss:l0} will always assign nonzero weight to models with dimension at least as large as that of the data-selected best single model, agreeing with what Breiman observed using cross-validated data. 

\begin{theorem}\label{thm:complexity}
The dimension of the stacked model with weights from \eqref{loss:l0} is no less than the dimension of the data-selected best single model \eqref{eq:best} with $ \lambda = 2$; that is,
$$
\text{dim}(\hat f_{\text{stack}}) \geq \text{dim}(\hat f_{\text{best}}).
$$
Furthermore, the weights sum to less than one, i.e., $ \sum_{k=1}^M \hat\alpha_k < 1 $.
\end{theorem}

\section{Relationship to Aggregation and Oracle Inequalities}
\label{sec:aggregation}

\subsection{Aggregation}

Stacking can be viewed as a method for solving the problem of aggregation, which was first defined in \cite{juditsky2000functional} and \cite{nemirovski2000topics}.
Here, we are given data as in \eqref{eq:model}, nonparametric estimators $\hat\mu_1,\hat\mu_2,\ldots,\hat\mu_M$ of the regression function $f$, and the goal is, quoting \cite{tsybakov2003optimal}, ``to construct a new estimate of $f$ (called \emph{aggregate}) that mimics in a certain sense the behavior of the best of the estimators.''
In other words, the performance of the aggregate estimator $\hat f_{\text{agg}}$ is compared with that of an oracle best-in-class estimator.
While several different benchmarks were introduced, the two that have received the most attention are the oracle best single model\footnote{This is also referred to in the literature as the \emph{model selection oracle}.} and the oracle best convex combination model, which translate respectively into the following desired oracle inequalities (or their counterpart deviation inequalities that hold with high probability):
\begin{equation}
\label{eq:aggregation_bs}
    \mathbb{E}\Big[\big\|\mathbf{f}-\hat{\mathbf{f}}_{\text{agg}}\big\|^2\Big] \leq \min_{1 \leq k \leq M} \mathbb{E}\big[\|\mathbf{f}-\hat{\boldsymbol{\mu}}_k\|^2\big] + \Delta_{n,M}^{MS},
\end{equation}
\begin{equation}
\label{eq:aggregation_cvx}
    \mathbb{E}\Big[\big\|\mathbf{f}-\hat{\mathbf{f}}_{\text{agg}}\big\|^2\Big] 
    \leq \min_{\alpha_1 \geq 0,\, \alpha_2 \geq 0, \,\dots,\alpha_M \geq 0:\;\sum_{k=1}^M \alpha_k = 1} \mathbb{E}\Bigg[\Bigg\|\mathbf{f}-\sum_{k=1}^M \alpha_k \hat{\boldsymbol{\mu}}_k\Bigg\|^2\Bigg] + \Delta_{n,M}^{C}.
\end{equation}

Works studying the aggregation problem have sought to:
\begin{enumerate}[(i)]
    \item Obtain minimax lower bounds for the remainder terms $\Delta_{n,M}^{MS}$ and $\Delta_{n,M}^{C}$, called the \emph{price to pay for aggregation}, i.e., for not knowing the oracle model;
    \item Construct (and analyze) estimators  $\hat f_{\text{agg}}$ with price values that are minimax rate-optimal (in terms of $M$ and $n$).
\end{enumerate}
To clarify, $\hat f_{\text{agg}}$ is allowed to be any estimator, although the constructions proposed by various works have taken the form of convex combinations of $\hat\mu_1,\hat\mu_2,\ldots,\hat\mu_M$.

For the sake of analytic tractability, most of these works studied the setting where the estimators $\hat\mu_1,\hat\mu_2,\ldots\hat\mu_M$ are deterministic functions (i.e., independent of $\varepsilon_1,\varepsilon_2,\ldots,\varepsilon_n$), rather than data-dependent estimators (see \cite{juditsky2000functional,tsybakov2003optimal,bunea2007aggregation,dalalyan2008aggregation,lecue2009aggregation,rigollet2011exponential,rigollet2012sparse} as well as the references therein). 
This assumption is justified by sample splitting, whereby a fraction of the available data is used to fit $\hat\mu_1,\hat\mu_2,\ldots,\hat\mu_M$ and the remainder is used to perform aggregation, conditioning on the first subsample.
While not affecting the minimax rates, this sample splitting procedure is an inefficient use of data and rarely performed in practice.

Closer to our setting are works that study aggregation when $\hat\mu_1,\hat\mu_2, \ldots\hat\mu_M$ are linear smoothers on observed labels $y_1,y_2,\ldots,y_n$.
When the smoothing matrices are orthogonal projections, \cite{barron2006information} proposed an estimator based on exponential weights and showed that it achieved a price bound of
\begin{equation}
\label{eq:barron-bound}
\Delta_{n,M}^{MS} = O\left(\frac{\sigma^2\log(M)}{n}\right)    
\end{equation}
when compared with the oracle best single model.

If the estimators are \emph{ordered} linear smoothers \citep{kneip1994ordered}, which includes the nested linear least squares models examined in this paper, then \cite{chernousova2013ordered} have shown that the price has an improved upper bound of
\begin{equation}
\label{eq:bellec-bound}
    \Delta^{MS}_{n,M} = O\left(\frac{\sigma^2\log\left(1+(n/\sigma^2)\min_{1 \leq k \leq M} \mathbb{E}\big[\|\mathbf{f}-\hat{\boldsymbol{\mu}}_k\|^2\big]\right)}{n}\right).
\end{equation}
Meanwhile, a different estimator based on the $Q$-aggregation procedure for deterministic functions \citep{dai2012deviation, rigollet2012kullback}, was investigated by \cite{bellec2018optimal} and \cite{bellec2020cost}. \cite{bellec2018optimal} provided a deviation version of \eqref{eq:aggregation_bs} for general linear smoothers $\hat\mu_k$, under a mild condition on the operator norms of the smoothing matrices. \cite{bellec2020cost} showed that if the estimators are ordered linear smoothers, then the $\log(M)$ factor in \eqref{eq:barron-bound} can be completely removed, i.e., $\Delta^{MS}_{n,M} = O(\sigma^2/n) $. 

Furthermore, using a slightly different aggregation scheme, \cite{bellec2018optimal} (Proposition 7.2) also obtained a price bound of
\begin{equation}
    \label{eq:bellec}
    \Delta_{n,M}^C = O\left(\sigma^2\sqrt{\frac{\log(1+M/\sqrt{n})}{n}}\right)
\end{equation}
for the oracle best convex combination benchmark \eqref{eq:aggregation_cvx}.

\subsection{Stacking as a Solution to Aggregation}

Although the aggregation literature is extensive, researchers in this area seem to have been unaware of \cite{breiman1996stacking} and \cite{wolpert1992stacked}'s work.
Indeed, the stacking estimator or even any alternative inspired by it has never been previously studied as a candidate solution to the aggregation problem.
While this is not the main focus of our paper, we are able to add to the existing literature by proving a price upper bound for the stacking estimator with respect to the oracle best \emph{conical} combination model. Note that this is a stronger benchmark than the oracle best convex combination model in \eqref{eq:aggregation_cvx}, since, with stacking, we do not constrain the weights to sum to one.

\begin{theorem}\label{thm:l0}
The population risk of the stacked model with weights from \eqref{loss:lasso} when $ \tau = \lambda = 1 $ satisfies the following oracle inequality:
\begin{equation} \label{eq:stack-oracle}
\mathbb{E}\big[\|\mathbf{f}-\hat{\mathbf{f}}_{\text{stack}}\|^2\big] \leq \min_{\alpha_1\geq 0,\, \alpha_2 \geq 0,\, \dots,\, \alpha_M \geq 0}\mathbb{E}\Bigg[\Bigg\|\mathbf{f}-\sum_{k=1}^M\alpha_k \hat{\boldsymbol{\mu}}_k \Bigg\|^2\Bigg] + \frac{4\sigma^2\mathbb{E}[\|\hat{\boldsymbol{\alpha}}\|_{\ell_0}]}{n}.
\end{equation}
\end{theorem}

Theorem \ref{thm:l0} says that the price is controlled by the expected number of stacked models, or $ \mathbb{E}[\|\hat{\boldsymbol{\alpha}}\|_{\ell_0}] $.

Breiman observed that this quantity was surprisingly small. For example, in his experiments with stacked (nested) regressions obtained from forward selection ($M=40$), he found that $ \mathbb{E}[\|\hat{\boldsymbol{\alpha}}\|_{\ell_0}] \approx 3.1 $.  The size of $ \|\hat{\boldsymbol{\alpha}}\|_{\ell_0} $ is equal to the number of unique elements of the isotonic sequence $\hat{\boldsymbol{\gamma}}$ that are less than one, defined in \eqref{eq:minmax} below. Finding useful bounds (other than $\mathbb{E}[\|\hat{\boldsymbol{\alpha}}\|_{\ell_0}] \leq M$) appears to be nontrivial. While previous shape-constrained estimation literature has shown that similar isotonic minimax sequences can have, on average, $o(M)$ distinct elements \citep{meyer2000shape}, we have yet to establish this for $ \hat{\boldsymbol{\gamma}} $. 

Nonetheless, the naive upper bound $\mathbb{E}[\|\hat{\boldsymbol{\alpha}}\|_{\ell_0}] \leq M$ already leads to an interesting conclusion about the minimax rate for $\Delta_{n,M}^{C}$. Note that when $\hat\mu_1,\hat\mu_2,\ldots,\hat\mu_M$ are instead deterministic, it is well-known that the minimax optimal rate for $\Delta_{n,M}^{C}$ is \begin{equation} \label{eq:d-deviation}
\sigma^2\left(\frac{M}{n} \wedge \sqrt{{\frac{\log(1+M/\sqrt{n})}{n}}}\right);
\end{equation}
see \citep{bunea2007aggregation} as well as Theorem 5.3 in \citep{rigollet2011exponential}. 

Determining the minimax rate for more general linear smoothers, which include deterministic functions, is an open question according to \cite{bellec2018optimal}. \cite{bellec2018optimal} (see Proposition 7.2) established a price bound $\Delta_{n,M}^{C}$ of order $\sigma^2\sqrt{\smash[b]{\log(1+M/\sqrt{n})/n}}$, which according to \eqref{eq:d-deviation} is optimal when $M > \sqrt{n}$. However, he was unable to achieve the $\sigma^2M/n$ rate when $M \leq \sqrt{n}$. Combining Bellec's result, which is applicable to nested models, with our price bound $\Delta_{n,M}^{C} = 4\sigma^2\mathbb{E}[\|\hat{\boldsymbol{\alpha}}\|_{\ell_0}]/n \leq 4\sigma^2M/n$ from \eqref{eq:stack-oracle} demonstrates that the price
from aggregating nested models is at most the optimal price \eqref{eq:d-deviation} from aggregating deterministic functions. While \eqref{eq:d-deviation} is achievable, we do not know if it is the optimal price for nested models, since nested models cannot include more than one deterministic function.

\subsection{Oracle Versus Data-selected Model}\label{sec:oracle_vs_single}

In contrast to the aggregation literature's focus on oracle inequalities, our main result 
compares the stacking estimator with a data-dependent quantity: the data-selected best single model (selected by an AIC-like model selection criterion).
Rather than being an intentionally weaker benchmark, our choice follows that of \cite{breiman1996stacking}, who was more interested in demonstrating the practical benefits of stacking.
From this perspective, the data-selected best single model, being used in practice, is more relevant than an oracle model that cannot be implemented.

Despite potential performance limitations in certain settings, AIC and BIC remain widely used across many applied disciplines, which means that data analysts often encounter situations where they have to choose between performing model selection via these criteria or using an aggregation approach (via stacking or otherwise).
\emph{Our results give actionable advice} on how to make that choice under a nested regressions scenario:
We show that the stacked model always has smaller population risk than the data-selected best single model and furthermore that the performance gap scales inversely with the sample size and the SNR.
This carefully quantifies the trade-off between the two desiderata of prediction accuracy and interpretability.

By framing performance evaluation in terms of oracle inequalities, the aggregation literature has struggled to make the same point.
\cite{lecue2009aggregation} seem to obtain a negative price bound, $\Delta^C_{n,M} \leq 0$, that holds with high probability instead of in expectation (c.f., their value of $\beta$ in the proof of Theorem 4.2 on page 606), but their result requires a lower bound on the diameter of the set of ``almost minimizers'' in $\lbrace\hat\mu_1,\hat\mu_2,\ldots,\hat\mu_M \rbrace$, and it is unclear when this condition is satisfied, if ever. Furthermore, these high probability bounds cannot be translated back into bounds that hold in expectation without incurring an additional positive price.
\cite{rigollet2012sparse} show that any estimator restricted to be a single element of $\hat\mu_1,\hat\mu_2,\ldots,\hat\mu_M$ cannot have a price bound \eqref{eq:aggregation_bs} satisfying the minimax rate (see Theorem 2.1 therein).
However, this is not a head-to-head comparison and is less convincing because minimax rate comparisons are asymptotic and moreover are based on a worst case analysis, c.f., our Theorem \ref{thm:main} which holds for \emph{any} regression function $ f $, noise level $ \sigma^2 $, and sample size $ n $.
Most importantly, both of these results were established only for deterministic functions and the oracle benchmark, which we have already discussed as being less practically informative.
We are also unaware of any result from the aggregation literature showing a dependence of price bounds on the SNR.

Second, our choice of benchmark allows us to make scientifically interesting claims about how stacking improves prediction performance that would be difficult to prove otherwise.
Recall that we show that stacking reduces the population risk relative to model selection by performing adaptive shrinkage.
As explained in Section \ref{sec:model_selection_and_shrinkage}, our ability to make this precise argument relies on the stacking estimator and the data-selected best single model having the same points of discontinuity.
This calculation is hence extremely delicate and difficult to extend to a comparison with an oracle.
A minimax framework is also inconvenient for demonstrating the benefits of shrinkage, which tends to make constant factor improvements rather than alter rates.
In addition, our results regarding shrinkage corroborates empirical observations that ensemble strategies such as stacking tend to be most beneficial in low SNR settings \citep{mentch2020randomization} and adds to the discussion regarding this phenomenon.

Setting aside the question of which benchmark (oracle or data-dependent) is more meaningful, it is also unclear which one is more conservative. Even within the nested regression setting, it is still not known whether the inequality
\begin{equation} \label{eq:compare1}
\min_{1 \leq k \leq M} \mathbb{E}\big[\|\mathbf{f}-\hat{\boldsymbol{\mu}}_k\|^2\big] \leq \mathbb{E}\big[\|\mathbf{f}-\hat{\mathbf{f}}_{\text{best}}\|^2\big],
\end{equation}
holds in general.
In fact, \eqref{eq:compare1} has only been established for specific (worst-case) constructions when $M=2$, where the gap can be of order $ \sigma/\sqrt{n}$; see \cite{bellec2018optimal} and the references therein. In the other direction, when $\lambda = 2 $ and the base estimators are ordered linear smoothers, \cite{kneip1994ordered} showed the bound
$$
\mathbb{E}\big[\|\mathbf{f}-\hat{\mathbf{f}}_{\text{best}}\|^2\big] \leq \min_{1 \leq k \leq M} \mathbb{E}\big[\|\mathbf{f}-\hat{\boldsymbol{\mu}}_k\|^2\big] + \frac{C\sigma^2\sqrt{1+(n/\sigma^2)\min_{1 \leq k \leq M} \mathbb{E}\big[\|\mathbf{f}-\hat{\boldsymbol{\mu}}_k\|^2\big]}}{n},
$$
where $C$ is a universal positive constant, and thus the data-selected best single model is never worse than the oracle by an order $\sigma/\sqrt{n}$ error, i.e., for this estimator, $\Delta^{MS}_{n,M} = O(\sigma/\sqrt{n})$. However, empirically, we often find that $\mathbb{E}\big[\|\mathbf{f}-\hat{\mathbf{f}}_{\text{best}}\|^2\big]$ and $\min_{1 \leq k \leq M} \mathbb{E}\big[\|\mathbf{f}-\hat{\boldsymbol{\mu}}_k\|^2\big]$ are much closer than these particular constructions and bounds would suggest (see Section \ref{sec:numerical} for numerical evidence). 

Proving \eqref{eq:compare1} poses an even greater challenge than the already difficult problem of determining whether the excess optimism (the difference between apparent and true error) is positive. To see why, let us revisit \cite{tibshirani2019excess, mikkelsen2018best}. Using sophisticated tools from differential geometry, they established that the excess optimism of subset regression estimators is positive, i.e.,
\begin{equation} \label{eq:compare2}
\mathbb{E}\Big[\hat R_{\hat m}\Big] = \mathbb{E}\Bigg[\|\mathbf{y}-\hat{\boldsymbol{\mu}}_{\hat m}\|^2 + \frac{2\sigma^2d_{\hat m}}{n}-\sigma^2\Bigg] \leq \mathbb{E}\big[\|\mathbf{f}-\hat{\boldsymbol{\mu}}_{\hat m}\|^2\big] = \mathbb{E}\big[\|\mathbf{f}-\hat{\mathbf{f}}_{\text{best}}\|^2\big],
\end{equation}
where $\hat m$ is the index of the data-selected best single model, defined in \eqref{eq:best}.
Since $ \mathbb{E}\big[\hat R_{\hat m}\big] \leq \min_{1 \leq k \leq M} \mathbb{E}\big[\|\mathbf{f}-\hat{\boldsymbol{\mu}}_k\|^2\big] $, we note that \eqref{eq:compare1} implies \eqref{eq:compare2}, but not vice versa.

\section{Comparisons with Other Aggregation Estimators}

While stacking enjoys widespread use in practical data analysis, other aggregation estimators such as \cite{barron2006information}'s exponential weights method and \cite{bellec2018optimal}'s $Q$-aggregation method are less popular, despite their theoretical guarantees.
It is difficult to explain this phenomenon from the lens of price bounds alone, especially since these other methods already achieve minimax rates.
By considering our specialized setting of nested regressions, however, we are able to perform detailed calculations that reveal why these other methods may be inferior.
In Section \ref{sec:numerical}, we support these arguments using numerical simulations.

\subsection{Aggregation with Exponential Weights}

\cite{barron2006information} start with an unbiased estimator of the population risk, $\hat R_k = R_k+2\sigma^2d_k/n-\sigma^2$, and form
exponential weights $\check\alpha_k \propto \pi_k\exp\big(-\frac{n\beta\hat R_k}{2\sigma^2}\big)$, where $\beta > 0 $ and $\pi_k$ is a deterministic sequence in the probability simplex, usually taken to be uniform, i.e., $\pi_k = 1/M$. These weights can be shown to minimize the separable objective
\begin{equation} \label{eq:boltzmann}
\sum_{k=1}^M \alpha_k\hat R_k + \frac{2\sigma^2}{n\beta}\sum_{k=1}^M \alpha_k\log(\alpha_k/\pi_k)
\end{equation}
over the probability simplex, similar to how the Gibbs distribution maximizes entropy, $ \sum_{k=1}^M \alpha_k\log(1/\alpha_k) $, subject to a constraint on the linear function $ \sum_{k=1}^M \alpha_k\hat R_k $.
Note that to achieve the bound \eqref{eq:barron-bound}, they set $\beta = 1/2$.

Because this aggregation estimator is based on exponential weights that optimize an objective \eqref{eq:boltzmann} not designed to estimate the population risk of the aggregation estimator, we expect its performance to be worse than stacking, which more directly targets the population risk of the combined models. 

Specifically, the entropy term in the separable objective \eqref{eq:boltzmann} encourages spreading weight across models, giving non-zero weights to all models, even poor performers. This may not be optimal if a small subset of models performs much better than the rest. Stacking does not have this built-in bias towards using all models, and can aggressively downweight or exclude poor models—the stacking weights $\hat\alpha_k $ in \eqref{eq:solution} are zero whenever the isotonic solution $\hat\gamma_k$ in \eqref{eq:minmax} is less than $\gamma$ or constant across two successive values. 

Moreover, the separable objective only accounts for the empirical risks, $R_k = \|\mathbf{y}-\hat{\boldsymbol{\mu}}_k\|^2$, of each model and, as such, fails to account for interactions, $\langle \mathbf{y}-\hat{\boldsymbol{\mu}}_k,\mathbf{y}-\hat{\boldsymbol{\mu}}_l \rangle  $, $k\neq l$, between the different models being combined. In other words, it treats the contribution of each model independently. Stacking, on the other hand, optimizes a quadratic form that includes these interaction terms between the models, enabling more adaptive shrinkage and allowing it to better capture how the models complement or overlap with each other. 

We confirm empirically that stacking outperforms aggregation with exponential weights in Section \ref{sec:numerical}, with the latter having performance similar to (and sometimes worse than) the data-selected best single model.

\begin{remark}
    Despite solving a more complex optimization problem, stacking can be solved (order-wise) in the same amount of time as forming the aggregation estimator with exponential weights, namely, $O(M)$.
\end{remark}

\subsection{$Q$-Aggregation}
\cite{bellec2018optimal}, inspired by the $Q$-aggregation procedure for deterministic functions \citep{dai2012deviation, rigollet2012kullback}, considers minimizing the quadratic objective  
\begin{equation} \label{eq:bellec-program}
R(\boldsymbol{\alpha}) + \frac{2\sigma^2}{n}\text{df}(\boldsymbol{\alpha})+ \eta \sum_{m=1}^M \alpha_m
\Bigg\|\hat{\boldsymbol{\mu}}_m-\sum_{k=1}^M \alpha_k \hat{\boldsymbol{\mu}}_k\Bigg\|^2, \quad \eta > 0,
\end{equation}
over the probability simplex. Note that to achieve the bound \eqref{eq:barron-bound}, he sets $\eta = 1/2$.

Bellec's quadratic program \eqref{eq:bellec-program}, which in general takes $O(M^3)$ operations to compute via interior-point methods, can be reduced to $O(M)$ for nested regressions. Simple algebra (see Appendix \ref{app:proofs} in the supplementary material) shows that the program \eqref{eq:bellec-program} is equivalent to
\begin{equation}
\begin{aligned}\label{loss:q-agg}
& \text{minimize} \quad \sum_{k=1}^M w_k(z_k - \beta_k)^2 \\
&\text{subject to} \quad  0 = \beta_1 \leq \beta_2 \leq \cdots\leq \beta_M \leq 1,
\end{aligned}
\end{equation}
where $ w_k = \Delta R_k $ and $ z_k = (1/(1-\eta))(\eta/2 +(\sigma^2/n)(\Delta d_k/\Delta R_k))$. Setting $$ \check\gamma_1 = 0, \quad \check\gamma_k = \frac{\sigma^2}{n}\min_{k\leq i\leq M}\max_{1\leq j < k}\frac{d_i - d_j}{R_j-R_i}, \quad k = 2, 3, \dots, M, $$ c.f., $\hat\gamma_k$ in \eqref{eq:minmax},
the solution has a closed form expression:
\begin{equation} \label{eq:q-sol}
\check\beta_k = 1-\phi\Bigg(\frac{1-\eta/2-\check\gamma_k}{1-\eta}\Bigg), \quad k = 1, 2, \dots, M,
\end{equation}
where $ \phi(z) = \min\{1, \max\{0, z\}\} $ is the clip function for $ z $ at $0$ and $1$. Thus, the solution to the program \eqref{eq:bellec-program} is
$$
\check\alpha_k = \check\beta_{k+1}-\check\beta_k = \phi\Bigg(\frac{1-\eta/2-\check\gamma_k}{1-\eta}\Bigg) - \phi\Bigg(\frac{1-\eta/2-\check\gamma_{k+1}}{1-\eta}\Bigg), \quad k = 1, 2, \dots, M.
$$

Similar to \eqref{eq:solution}, the isotonic vector $\check{\boldsymbol{\gamma}}$ can be computed with $O(M)$ operations via PAVA.
However, as with Leung and Barron's solution to \eqref{eq:boltzmann}, this solution differs from the stacking solution \eqref{eq:solution} in multiple ways. First, it is continuous in the labels $\mathbf{y}$. Second, there is a bias term, $-\eta/2$, in the shrinkage factor. Third, due to the nature of the clip function, $ \check\beta_k $ is more likely to be pushed into the corners $ \{0, 1\} $ whenever $\check\gamma_k \leq \eta/2$ \emph{or} $\check\gamma_k \geq 1-\eta/2$, as opposed to only when $\hat\gamma_k \geq \gamma$ for \eqref{eq:solution}, resulting in $ \check\alpha_k = \check\beta_{k+1} - \check\beta_k = 0 $. Thus, the weights will tend to be sparser than the stacking weights \eqref{eq:solution}. These differences suggest that Bellec's estimator that optimizes \eqref{eq:bellec-program} may not fully capitalize on the benefits of shrinkage and the more targeted adaptivity offered by stacking. 

\section{Bayesian Model Averaging}
There is an interesting connection between stacked regressions and Bayesian model averaging that we now explore. To see this, we construct a prior distribution on the regularization parameter $\lambda$. Let $ \Lambda $ have a mixed discrete-continuous distribution with Pareto-like distribution function $ F_{\Lambda}(\lambda') = (1-\tau/\lambda')\mathbf{1}(\lambda' \geq 1/\gamma) $, where $ \tau $ and $\gamma = \min\{1/\tau, 1/\lambda \}$ are as in Theorem \ref{thm:fstack}. The density function can be written as $ f_{\Lambda}(\lambda') = (1-\tau\gamma)\delta(\lambda'-1/\gamma) + \tau/(\lambda')^2\mathbf{1}(\lambda' > 1/\gamma) $, where $ \delta $ is the Dirac delta function. Note that this distribution does not have a finite mean. Draw $\Lambda$ from this prior distribution and set $$
\hat m_{\Lambda} \in \argmin_{k = 0, 1, \dots, M} \; \|\bold{y} - \hat{\boldsymbol{\mu}}_k\|^2 + \Lambda\frac{\sigma^2 d_k}{n}.
$$
Using the form \eqref{eq:best_form} of the data-selected best single model with regularization parameter $ \Lambda = \lambda $, we have the posterior predictive mean
\begin{equation} \label{eq:ensemble}
\hat f_{\text{Bayes}}(\bx) = \mathbb{E}_{\Lambda}\big[\hat\mu_{\hat m_{\Lambda}}(\bx)\big] = \sum_{k=1}^M(\hat\mu_k(\bx)-\hat\mu_{k-1}(\bx))\mathbb{P}_{\Lambda}\big(\hat\gamma_k < 1/\Lambda\big),
\end{equation}
where $ \mathbb{P}_{\Lambda} $ and $ \mathbb{E}_{\Lambda} $ denotes the probability measure and expectation operator of the prior, respectively. It is easy to see that $ \mathbb{P}_{\Lambda}\big(\hat\gamma_k < 1/\Lambda\big) = F_{\Lambda}(1/\hat\gamma_k-) = (1-\tau\hat\gamma_k)\mathbf{1}(\hat\gamma_k < \gamma) $, so that according to \eqref{eq:stack}, $ \hat f_{\text{Bayes}} $ and $ \hat f_{\text{stack}} $ coincide exactly.

As can be seen from comparing \eqref{eq:ensemble} and \eqref{eq:stack}, the posterior predictive mean and stacking both reduce the variance of the data-selected best single model in the same way, but do so from different perspectives. Recall the form of the data-selected best single model \eqref{eq:best_form} in which the successive predictive differences $ \hat\mu_k(\bx)-\hat\mu_{k-1}(\bx) $ across successive submodels are multiplied by $ \mathbf{1}(\hat\gamma_k < 1/\lambda) $ and then summed up to the selected model index $ \hat m $. The posterior predictive mean smooths the jump discontinuity in the indicator function $ \mathbf{1}(\hat\gamma_k < 1/\lambda) $ by averaging the indicator functions $\mathbf{1}(\hat\gamma_k < 1/\Lambda)$ resulting from random perturbations around the data-selected best single model at $\Lambda = \lambda $. On the other hand, stacking shrinks $ \mathbf{1}(\hat\gamma_k < 1/\lambda) $ towards zero by minimizing a regularized version of the empirical risk over a set of weights.


\section{Numerical Experiments} \label{sec:numerical}
In this section, we present a
numerical evaluation of the performance of the stacked model with weights from \eqref{loss:lasso}. We compare its performance against two benchmarks: the data-selected best single model \eqref{eq:best} and the oracle best single model $\hat\mu_{m^*}$, where $m^* \in \argmin_{1 \leq k \leq M}\mathbb{E}[\|\mathbf{f}-\hat{\boldsymbol{\mu}}_k\|^2]$, as well as two aggregation estimators: the Leung and Barron estimator \citep{barron2006information} and the Bellec estimator \citep{bellec2018optimal}.

We consider subset regression with a fixed design, where the orthogonal data matrix $\bX \in \mathbb{R}^{n\times d}$ satisfies $\bX^\top \bX = \mathbf{I}$. We set $n = 1000$ and $d = 150$. Nested regressions are constructed by regressing (without intercept) on the first $3, 6, \dots, 150$ columns of $\bX$. In other words, we set $d_k = 3k$ for $k=1,2,\dots,50$. For Figure \ref{fig:linear_first35} and Figure \ref{fig:linear_first60}, the function $f$ is linear and depends on the first 35 and 60 covariates, respectively. Let the covariates be denoted as $x_1, x_2, \dots, x_{150}$. For Figure \ref{fig:linear_first35}, the coefficients are 35 values randomly selected from a uniform distribution over $[-5, 5]$, denoted  as $c_1, c_2, \dots, c_{35}$. We set the true function as $f_1 = \alpha\sum_{j=1}^{35} c_jx_j$, where $\alpha$ controls the signal strength $\|\mathbf{f}_1\|$. For Figure \ref{fig:linear_first60}, the true function $f_2$ is constructed linearly in the same manner, this time using 60 randomly selected coefficients. Figure \ref{fig:nonlinear_first35} and Figure \ref{fig:nonlinear_first60} use nonlinear functions. The true function in Figure \ref{fig:nonlinear_first35} is defined as  $f_3 = f_1 + \alpha\big(0.3c_1x_1^2 + 0.1c_1\sin(\pi x_1x_2)\big)$. Similarly, the true function in Figure \ref{fig:nonlinear_first60}, is defined as  $f_4 = f_2 + \alpha\big(0.3c_1x_1^2 + 0.1c_1\sin(\pi x_1x_2)\big)$. In all cases, the noise level $\sigma$ is set to be 3, and the signal strength $\|\mathbf{f}\|$ ranges from 1 to 5 in increments of 0.5. For stacking, we use $\tau=1$ and $\lambda=2$, so the data-selected best single model corresponds to the model selected by the AIC criterion. For the Leung and Barron estimator, we use $\beta = 1/2$ and take $\pi_k$ to be uniform, as they recommend for subset regressions (see Section V, Part A in \citep{barron2006information}). We conducted $3,000$ replications for each experiment, using the same noise realizations $\varepsilon_1, \varepsilon_2, \dots, \varepsilon_n$ across different choices of $f$ for additional stability.

The figures corroborate our theory presented in Theorem \ref{thm:main}, showing that stacking performs particularly well when $\|\mathbf{f}\|$ is small, outperforming other benchmarks, including the oracle best single model. As $\|\mathbf{f}\|$ increases, the advantage of stacking diminishes, with other estimators showing comparable performance. Nevertheless, stacking consistently achieves the lowest MSE across nearly all cases when compared to the Leung and Barron estimator, the Bellec estimator, and the data-selected best single model. Remarkably, when $\|\mathbf{f}\|$ is small, stacking even surpasses the oracle best single model—something neither the Leung and Barron nor the Bellec estimator manages to achieve. Additional experiments show similar patterns with non-Gaussian and random design data; see Appendix \ref{app:numeric_additional} in the supplementary material.

\begin{figure}[htbp]
    \centering
    \begin{minipage}[b]{0.48\textwidth}
        \centering
        \includegraphics[width=\textwidth]{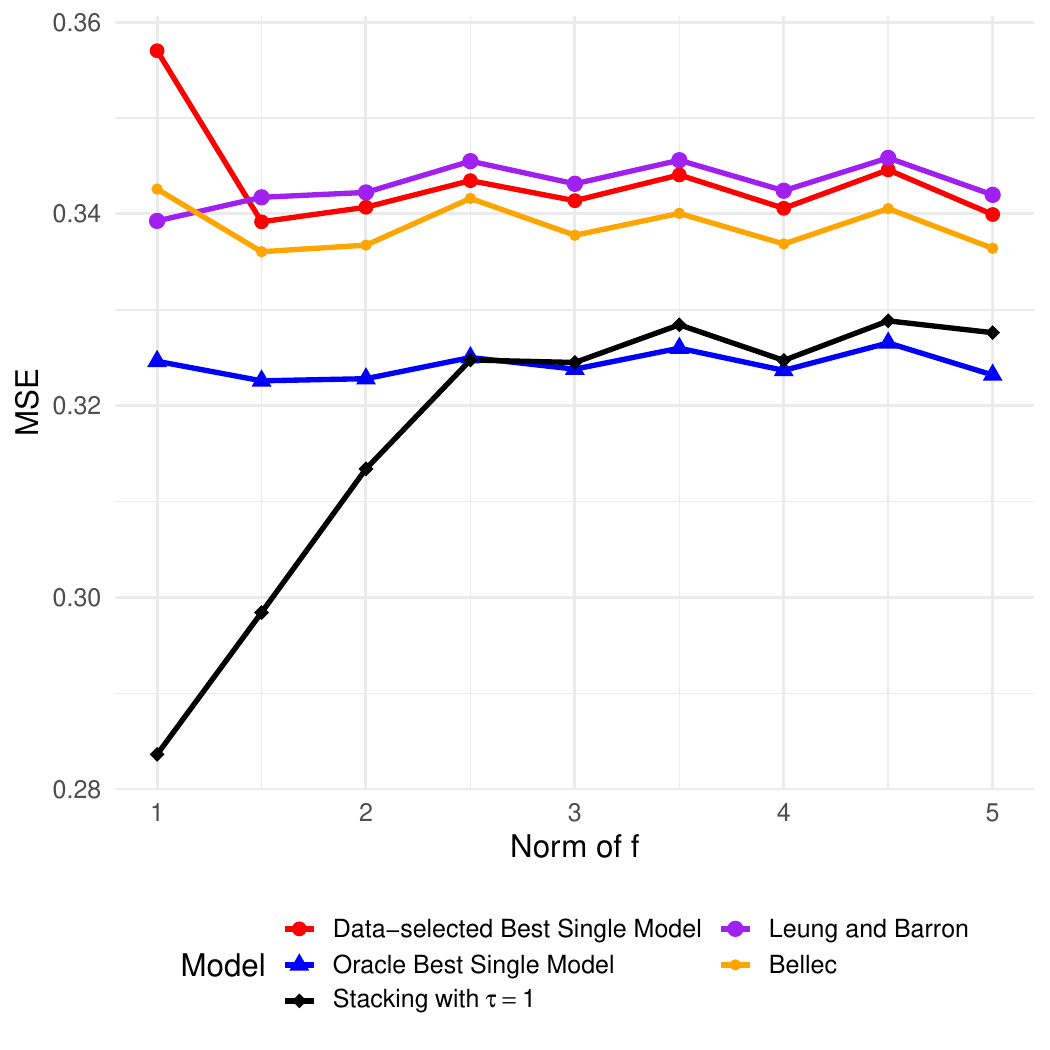}
        \caption{\small MSE comparison across different methods, where the function $f$ is linear and depends on the first 35 covariates.}
        \label{fig:linear_first35}
    \end{minipage}
    \hfill
    \begin{minipage}[b]{0.48\textwidth}
        \centering
        \includegraphics[width=\textwidth]{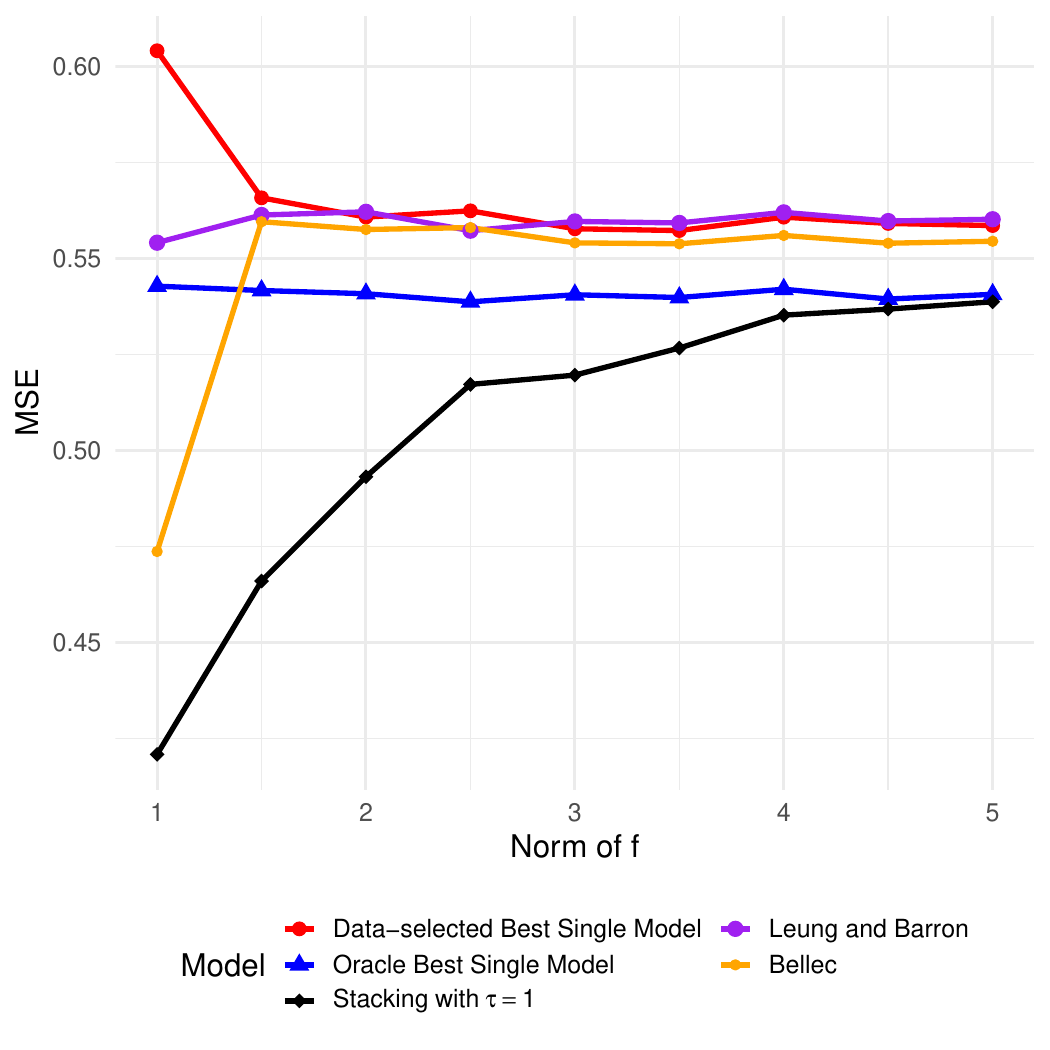}
        \caption{\small MSE comparison across different methods, where the function $f$ is linear and depends on the first 60 covariates.}
        \label{fig:linear_first60}
    \end{minipage}
\end{figure}

\begin{figure}[htbp]
    \centering
    \begin{minipage}[b]{0.48\textwidth}
        \centering
        \includegraphics[width=\textwidth]{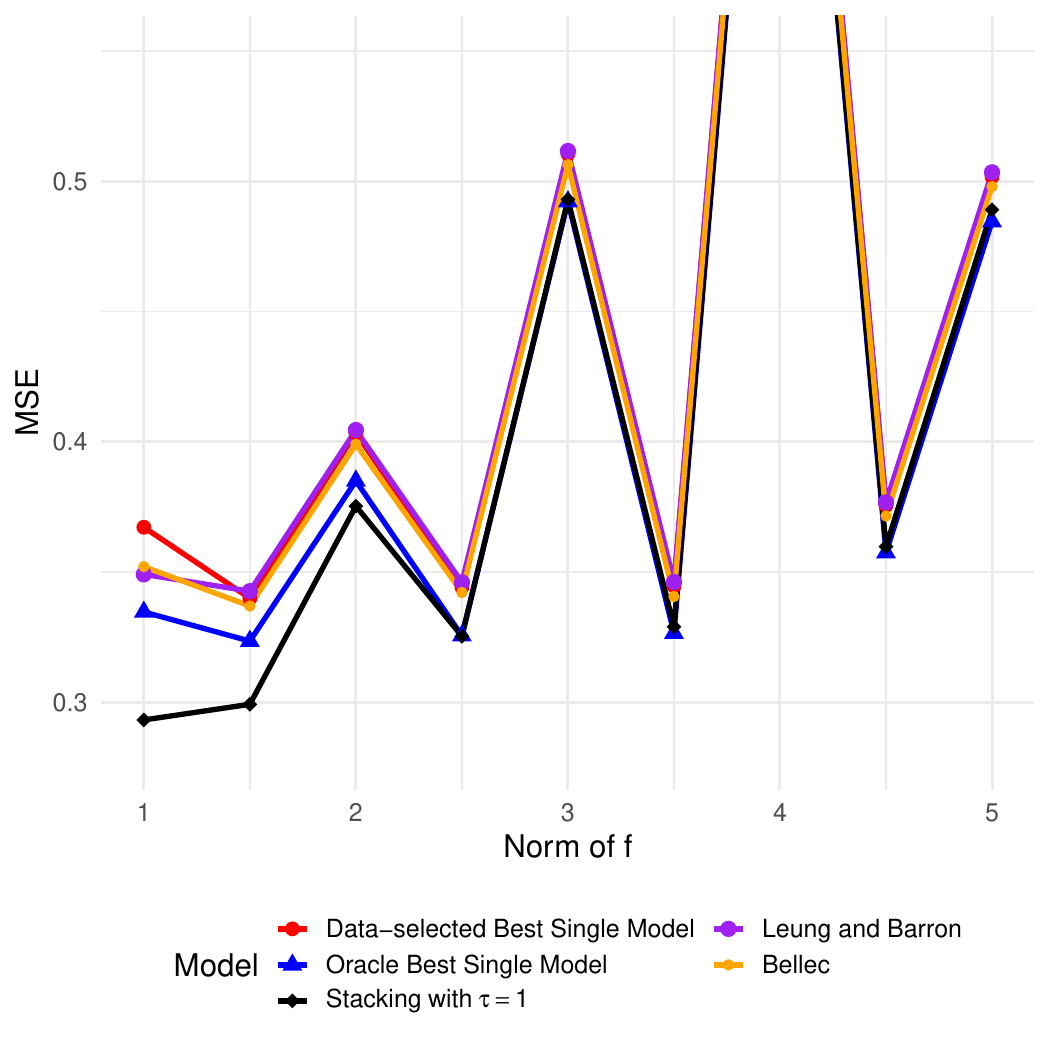}
        \caption{\small MSE comparison across different methods, where the function $f$ is nonlinear and depends on the first 35 covariates.}
        \label{fig:nonlinear_first35}
    \end{minipage}
    \hfill
    \begin{minipage}[b]{0.48\textwidth}
        \centering
        \includegraphics[width=\textwidth]{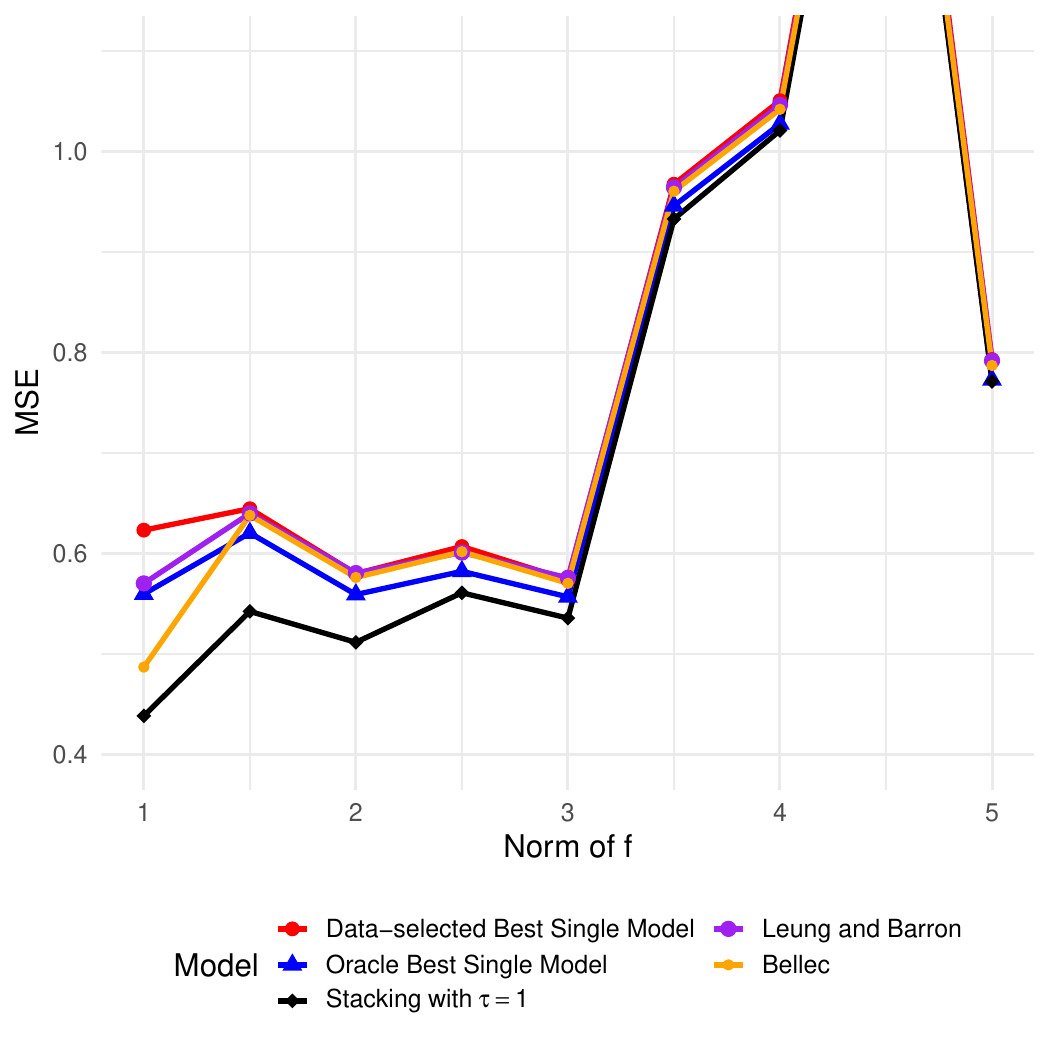}
        \caption{\small MSE comparison across different methods, where the function $f$ is nonlinear and depends on the first 60 covariates.}
        \label{fig:nonlinear_first60}
    \end{minipage}
\end{figure}

\section{Discussion and Conclusion}

The findings from this research provide compelling evidence for the enhanced predictive accuracy of stacked regressions. As we continually seek cheap and simple ways to improve model performance, ensemble methods like stacking emerge as valuable tools. Our research reiterates the seminal works of \cite{wolpert1992stacked} and \cite{breiman1996stacking}, both of whom touched upon the potential of stacking.

The fact that the stacked model provably outperforms the data-selected best single model for \emph{any} regression function $ f $, noise level $ \sigma^2 $, and sample size $ n $ is of course noteworthy. However, it is important to consider the conditions under which this improvement can be rigorously shown. Perhaps most limiting was that we stack models resulting from least squares projections onto nested subspaces, that the design is non-random, and that the errors are Gaussian. For tractable theoretical analysis, it is likely the fixed design and Gaussian assumption cannot be dropped, but the structure of the models can most likely be extended. For example, Breiman also considered stacking ridge regressions and reached similar though not quite as strong conclusions on its merits. More generally, ridge and nested regression models are both instances of ordered linear smoothers \citep{kneip1994ordered}, i.e., families of models that can be written as
$$
\hat\mu_k(\bx) = \sum_{l=1}^n c_{k,l}\langle \mathbf{y}, \boldsymbol{\psi}_l \rangle \psi_l(\bx), \quad 0 \leq c_{k,l} \leq c_{k+1,l} \leq 1.
$$
Interestingly, the representation \eqref{eq:best_form} for the data-selected best single model still holds if $ d_k $ is replaced by $ \text{df}(\hat\mu_k) = \sum_{l=1}^nc_{k,l}  $. However, it is not true anymore that the stacking estimator has the form \eqref{eq:stack}, unless $ c_{k,l} \in \{0, 1\} $.

Finally, we mention that there have been some attempts to bypass the fixed design and Gaussian assumptions, but only for aggregating deterministic functions \citep{dalalyan2008aggregation, lecue2009aggregation, rigollet2011exponential}. We do not see these as major limitations in our analysis of stacking, since the main point of our work is to theoretically understand some of Breiman's empirical observations—and the great practical success stacking has enjoyed since its inception. Furthermore, the computational aspects of our stacking estimator do not depend on these assumptions, and the estimator can be applied to nested models fit to arbitrary data. Numerical experiments in Appendix \ref{app:numeric_additional} of the supplementary material provide evidence that stacking retains its superior performance in more general data settings.

\appendix
\renewcommand{\theequation}{\thesection.\arabic{equation}}
\renewcommand{\thefigure}{\thesection.\arabic{figure}}
\renewcommand{\theremark}{\thesection.\arabic{remark}}

\section{Proofs}\label{app:proofs} In this section, we provide full proofs of all the main results and other statements in the main text. We first establish some notation. Given a sequence of indices $s_0=0,s_1,\dots,s_u=M$, we define $\Delta d_{s_l} = d_{s_l} - d_{s_{l-1}}$ and  $\Delta R_{s_l} = R_{s_{l-1}} - R_{s_{l}}$.

\begin{proof}[Proof of Lemma \ref{lem:equi}]
The proof is based on the following lemma.
    \begin{lemma}\label{lem:training_error}
    Suppose $h(\bx)$ has the form
    \begin{align*}
        h(\bx) = \sum_{k=1}^M(\hat \mu_k(\bx) - \hat \mu_{k-1}(\bx))c_k.
    \end{align*}
    Then, 
    \begin{align*}
        \bnorm{\by-\mathbf{h}}^2 = R_0 +\sum_{k=1}^M\drk (c_k^2-2c_k).
    \end{align*}
\end{lemma}
\begin{proof}[Proof of Lemma \ref{lem:training_error}]
Note that for any $k \neq l$,
\begin{align*}
 \langle \boldsymbol{\hat\mu}_k - \boldsymbol{\hat\mu}_{k-1}, \boldsymbol{\hat\mu}_{l} - \boldsymbol{\hat\mu}_{l-1} \rangle =  \Bigg\langle{\sum_{i \in A_k \backslash A_{k-1}} \iprod{\by}{\boldsymbol{\psi}_i}\psi_i}, {\sum_{j \in A_l \backslash A_{l-1}} \langle \by,\boldsymbol{\psi}_j \rangle\psi_j} \Bigg\rangle = 0,
\end{align*}
since  $\langle \boldsymbol{\psi}_i,\boldsymbol{\psi}_j\rangle=0$ for any $i \neq j$, and
\begin{align*}
\| \boldsymbol{\hat\mu}_k - \boldsymbol{\hat\mu}_{k-1}\|^2 = \sum_{j \in A_k\backslash A_{k-1}}\langle \by, \boldsymbol{\psi}_j\rangle^2 = \drk.
\end{align*}
In addition, 
\begin{align*}
\langle \by, \boldsymbol{\hat\mu}_k - \boldsymbol{\hat\mu}_{k-1} \rangle = \Bigg\langle \by, \sum_{j \in A_k \backslash A_{k-1}} \langle \by,\boldsymbol{\psi}_j \rangle\psi_j \Bigg\rangle = \sum_{j \in A_k\backslash A_{k-1}}\langle y, \boldsymbol{\psi}_j\rangle^2 = \drk.
\end{align*}
Thus,
\begin{align*}
\bnorm{\by-\mathbf{h}}^2 &=   \bnorm{\by-\sum_{k=1}^M(\boldsymbol{\hat\mu}_k - \boldsymbol{\hat\mu}_{k-1})c_k}^2\\
&= \|\by\|^2 + \sum_{k=1}^M c_k^2\| \boldsymbol{\hat\mu}_k - \boldsymbol{\hat\mu}_{k-1}\|^2 - 2\sum_{k=1}^Mc_k \langle \by,  \boldsymbol{\hat\mu}_k - \boldsymbol{\hat\mu}_{k-1} \rangle\\
&= R_0 +\sum_{k=1}^M\drk (c_k^2-2c_k). \qedhere
\end{align*}
\end{proof}

Let $c_k = \sum_{i=k}^M\alpha_i $, for $ k = 1,2, \dots, M$. 
Define $\beta_k = 1 - c_k$, with $\beta_{M+1} = 1$.
Note that $\alpha_k = \beta_{k+1}- \beta_k$.
Using summation by parts, we have that 
\begin{align*}
    \sum_{k=1}^M\alpha_k\hat\mu_k(\bx) = \sum_{k=1}^M (\hat\mu_k(\bx) - \hat\mu_{k-1}(\bx))c_k,\quad \sum_{k=1}^M\alpha_kd_k = \sum_{k=1}^Mc_k\ddk,
\end{align*}
where $\hat\mu_0 \equiv 0$. In addition, note that if there exists  some $1\leq a \leq M$ such that $\alpha_a>0=\alpha_{a+1}=\cdots\alpha_{M}$, then we have $c_a> 0 =c_{a+1}=\cdots =c_M$.  Thus,
\begin{align*}
    \text{dim}(\boldsymbol{\alpha}) = \sum_{k=1}^M\ddk\textbf{1}(c_k \neq 0)=\sum_{k=1}^M\ddk\textbf{1}(\beta_k \neq 1).
\end{align*}

Therefore, with Lemma \ref{lem:training_error} we can establish the equivalence between the solution of \eqref{loss:lasso} and the solution of the following programs:
\begin{align*}
    &\text{minimize} \;\; R_0 - \sum_{k=1}^M\drk(2c_k-c_k^2) + \frac{2\tau\sigma^2}{n} \sum_{k=1}^Mc_k\ddk +  
  \xi\sum_{k=1}^M\ddk\textbf{1}(c_k \neq 0) \\
  &\text{subject to} \;\; c_1 \geq c_2 \geq\cdots\geq c_M \geq 0 \\
  \Leftrightarrow~&\text{minimize} \;\; \sum_{k=1}^M\drk\Bigg(c_k^2 -2c_k+c_k\frac{2\tau\sigma^2}{n}\frac{\ddk}{\drk} \Bigg)+ \xi\sum_{k=1}^M\ddk\textbf{1}(c_k \neq 0) \\
  &\text{subject to} \;\; c_1 \geq c_2 \geq\cdots\geq c_M \geq 0 \\
  \Leftrightarrow~&\text{minimize} \;\; \sum_{k=1}^M\drk\Bigg(c_k - \Bigg(1-\frac{\tau\sigma^2}{n}\frac{\Delta d_k}{\Delta R_k}\Bigg)\Bigg)^2+ \xi\sum_{k=1}^M\ddk\textbf{1}(c_k \neq 0) \\
  &\text{subject to} \;\; c_1 \geq c_2 \geq\cdots\geq c_M \geq 0\\
  \Leftrightarrow~&\text{minimize} \;\;\sum_{k=1}^M w_k(z_k - \beta_k)^2 + \xi\sum_{k=1}^M\ddk\textbf{1}(\beta_k \neq 1) \\ & \text{subject to} \;\; \beta_1 \leq \beta_2 \leq \cdots \leq \beta_M \leq 1,
\end{align*}
where $ w_k = \Delta R_k > 0  $, $ z_k = (\tau\sigma^2/n)(\Delta d_k/\Delta R_k) > 0 $, and $ \xi = (\sigma^2/n)((\lambda-\tau)^2_{+}/\lambda) $. \qedhere
\end{proof}

\begin{proof}[Proof of Theorem \ref{thm:fstack}]
We first show the formula for the data-selected best single model \eqref{eq:best_form}. It can be directly obtained from the following lemma, which characterizes the selected index $\hat m$. 
\begin{lemma}\label{lem:mallow}
Let $\tilde m = \max\big\{k\in \{0,1,\dots,M\}: \hat \gamma_{k} < 1/\lambda\big\}$  where $\{\hat \gamma_k\}$ is defined in \eqref{eq:minmax}. Then, for any $0\leq j< \tilde m$,
\begin{align*}
    R_{\tilde m} + \frac{\lambda\sigma^2}{n}d_{\tilde m} < R_j + \frac{\lambda\sigma^2}{n}d_j.
\end{align*}
and for any $\tilde m < j \leq M$
\begin{align*}
    R_{\tilde m} + \frac{\lambda\sigma^2}{n}d_{\tilde m} \leq R_j + \frac{\lambda\sigma^2}{n}d_j.
\end{align*}
This shows $\hat m = \tilde m$.
\end{lemma}
\begin{proof}[Proof of Lemma \ref{lem:mallow}] 
A well-have thatn geometric property of isotonic regression is that the solution is the left slope of the \emph{greatest convex minorant} of the \emph{cumulative sum diagram} \citep{barlow1972statistical}. Let $0 = s_0 < s_1 <\dots <s_u = M$ be the values of $k$ for which $\hat \gamma_{s_k+1}>\hat \gamma_{s_k}$ ($u \geq k \geq 0$), and $\hat \gamma_{0} = 0$ and $\hat \gamma_{M+1} = \infty$, implying that $s_0, s_1, \dots, s_u$ are the change points of the isotonic solution. Then, according to this property,
\begin{align}\label{eq:tau_r_l}
        \hat \gamma_{s_l} = \frac{\sigma^2}{n}\frac{d_{s_{l}} - d_{s_{l-1}}}{R_{s_{l-1}} - R_{s_{l}}}, \quad l = 1,2,\dots, u.
\end{align}
The definition of $\tilde m$ implies it should be one of the change points. We assume $\tilde m = s_a$. According to \eqref{eq:tau_r_l} and the definition of $ \hat\gamma_{s_a}$, we have that for  $0\leq j < \tilde m$, 
\begin{align*}
    \frac{\sigma^2}{n}\frac{d_{s_{a}} - d_{s_{a-1}}}{R_{s_{a-1}} - R_{s_{a}}} =  \hat\gamma_{s_a}  = \max_{1\leq i\leq s_a-1}\frac{\sigma^2}{n}\frac{d_{s_a} - d_i}{R_{i} - R_{s_a}} \geq \frac{\sigma^2}{n}\frac{d_{s_a} - d_j}{R_{j} - R_{s_a}}.
\end{align*}
Thus, 
\begin{align*}
    R_{s_a} + \frac{\lambda\sigma^2}{n}d_{s_a} &\leq R_j + R_{s_a} - R_j  + \frac{\lambda\sigma^2}{n}d_j + \lambda \hat\gamma_{s_a}(R_j - R_{s_a})\\
    &=  R_j + \frac{\lambda\sigma^2}{n}d_j + (R_j-R_{s_a})(\lambda  \hat\gamma_{s_a}-1) \\
    &< R_j + \frac{\lambda\sigma^2}{n}d_j.
\end{align*}
where the last inequality holds since $ \hat \gamma_{s_a} =  \hat \gamma_{\tilde m} < 1/\lambda$. Similarly, for any $\tilde m < j \leq M$,
\begin{align*}
    \frac{\sigma^2}{n}\frac{d_{s_{a+1}} - d_{s_{a}}}{R_{s_{a}} - R_{s_{a+1}}} = 
 \hat \gamma_{s_a+1} = \min_{s_a+1 \leq i \leq M}\frac{\sigma^2}{n}\frac{d_i - d_{s_a}}{R_{s_a}-R_i} \leq \frac{\sigma^2}{n}\frac{d_j - d_{s_a}}{R_{s_a}-R_j}.
\end{align*}
Thus, 
\begin{align*}
    R_{s_a} + \frac{\lambda\sigma^2}{n}d_{s_a} &\leq R_j + R_{s_a} - R_j  + \frac{\lambda\sigma^2}{n}d_j -\lambda \hat \gamma_{s_a+1}( R_{s_a}-R_j) \\
    &=  R_j + \frac{\lambda\sigma^2}{n}d_j + (R_{s_a}-R_j)(1-\lambda  \hat \gamma_{s_a+1}) \\
    &< R_j + \frac{\lambda\sigma^2}{n}d_j.
\end{align*}
where the last inequality holds since $ \hat \gamma_{s_a+1} =  \hat \gamma_{\tilde m+1} \geq 1/\lambda $ and $R_{s_a}- R_j > 0$.   \qedhere
\end{proof}

Next we show \eqref{eq:stack}.  According to Lemma \ref{lem:equi}, we can focus on program \eqref{loss:1}.  We need the following lemmas: Lemma \ref{lem:opt} solves the weighted isotonic regression problem with bounded constraints and Lemma \ref{lem:prob0} paves the way for establishing  almost sure uniqueness. 

\begin{lemma}\label{lem:opt}
Let $w_k > 0$ ($k=1,2,\dots,M$). Then the solution to the program
\begin{equation}
\begin{aligned}\label{optim:truncate}
    & \text{minimize} \quad \sum_{k=1}^{M}w_k\pth{z_k -\mu_k}^2 \\
    & \text{subject to} \quad a\leq \mu_1 \leq \cdots \leq \mu_M \leq b, 
\end{aligned}    
\end{equation}
is $\tilde\mu_k = \phi_{a,b}\Big(\max\limits_{1\leq i\leq k}\min\limits_{k\leq j \leq M}  \frac{\sum_{l=i}^j w_lz_l}{\sum_{l=i}^j w_l}\Big)$ for $k  = 1, 2, \dots, M$, where $ \phi_{a,b}(z) = \min\{b, \max\{a, z\}\} $ is the clip function for $ z $ at $a$ and $b$. In particular, when $a=0$ and $b = \infty$, we have $\tilde\mu_k = \Big(\max\limits_{1\leq i\leq k}\min\limits_{k\leq j \leq M}  \frac{\sum_{l=i}^j w_lz_l}{\sum_{l=i}^j w_l}\Big)_+$ for $k  = 1, 2, \dots, M$.
\end{lemma}
\begin{proof}[Proof of Lemma \ref{lem:opt}]
We first consider another program
\begin{equation}
    \begin{aligned}\label{optim:original}
    & \text{minimize} \quad \sum_{k=1}^{M}w_k\pth{z_k -\mu_k}^2 \\
    & \text{subject to} \quad  \mu_1\leq \mu_2 \leq \cdots \leq \mu_M, 
\end{aligned}
\end{equation}
A well-have thatn result \citep[Equation 2.8]{barlow1972isotonic} shows the solution of \eqref{optim:original}, denoted by $\boldsymbol{\mu}^*$, is $\mu_k^* = \max\limits_{1\leq i\leq k}\min\limits_{k\leq j \leq M}  \frac{\sum_{l=i}^j w_lz_l}{\sum_{l=i}^j w_l}$ for $1\leq k \leq M$. Let $\boldsymbol{\eta} =(\eta_0,\eta_1,\dots, \eta_M)$ and $\boldsymbol{\lambda} = (\lambda_1,\lambda_2,\dots,\lambda_{M-1}) $ with each entry of $\boldsymbol{\eta}$ and $\boldsymbol{\lambda}$ nonnegative.  We denote the Lagrangians of \eqref{optim:truncate} and \eqref{optim:original} respectively by 
\begin{align*}
    L_1(\boldsymbol{\mu},\boldsymbol{\eta}) = \sum_{k=1}^M w_k(z_k-\mu_k)^2 - \eta_0(\mu_1-a) + \sum_{k=1}^{M-1}\eta_k(\mu_k-\mu_{k+1})+\eta_M(\mu_M-b)
\end{align*}
and 
\begin{align*}
    L_2(\boldsymbol{\mu},\boldsymbol{\lambda}) = \sum_{k=1}^M w_k(z_k-\mu_k)^2  + \sum_{k=1}^{M-1}\lambda_k(\mu_k-\mu_{k+1}).
\end{align*}
Because the program \eqref{optim:original} is convex, there exists $\boldsymbol{\lambda}^* = (\lambda_1^*,\lambda_2^*,\dots,\lambda_{M-1}^*)$ such that $(\boldsymbol{\mu}^*,\boldsymbol{\lambda}^*)$ satisfies the Karush–Kuhn–Tucker (KKT) conditions, i.e., 
\begin{align}
    \nabla L_2(\boldsymbol{\mu}^*, \boldsymbol{\lambda}^*) = 0 &\Rightarrow -2w_k(z_k-\mu_k^*)-\lambda_{k-1}^*+\lambda_k^* = 0, \quad k = 1,2,\dots, M \label{stationary_P1}\\
     \boldsymbol{\lambda}^* &\geq 0, \\
     \mu_1^*\leq\mu_2^*\leq&\cdots\leq \mu_M^*,\\
    \lambda_k^*(\mu_k^* - \mu_{k+1}^*) &=0, \quad k = 1,2,\dots, M-1, \label{slackness_P1}
\end{align}
where we let $\lambda_0^*=\lambda_M^*=0$.  Let $m_1 = \max\{k \in [M]:\mu_k^*<a\}$ and $m_2 = \max\{k \in [M]:\mu_k^*\leq b\}$. If $\mu_k^*\geq a$ for any $k=1,2,\dots,M$, then let $m_1=0$. Since $\{\mu_k^*\}$ is not decreasing and $a\leq b$, we have $m_1\leq m_2$.  By \eqref{slackness_P1} and the facts that $\mu_{m_1+1}^* - \mu_{m_1}^* > 0$ and $\mu_{m_2+1}^* - \mu_{m_2}^* > 0$,  we have that $\lambda_{m_1}^*=\lambda_{m_2}^* = 0$. [Here we define $\mu_{M+1}^* = \infty$.] Now  we construct an $\boldsymbol{\eta}^* = (\eta_0^*,\eta_1^*,\dots,\eta_{M}^*)$ satisfying
\begin{align*}
    \eta_k^* = 
    \begin{cases}
    2w_{k}z_{k}-2w_k b + \eta_{k-1}^* &\quad k =m_2+1,m_2+2,\dots,M\\
    \lambda_k^* &\quad k =m_1,m_1+1,\dots, m_2 \\
    -2w_{k+1}z_{k+1}+2w_{k+1}a + \eta_{k+1}^* &\quad k =0,1,\dots,m_1-1.
    \end{cases}
\end{align*}
In what follows, we verify that $(\boldsymbol{\tilde\mu}, \boldsymbol{\eta}^*)$ satisfies the KKT conditions for \eqref{optim:truncate}. Since $\mu_k^*$ is increasing, it is easy to see $a\leq \tilde\mu_1 \leq \cdots \tilde\mu_M\leq b$, which confirms the primal feasibility. Note that 
$\eta_k^*(\tilde \mu_k - \tilde\mu_{k+1}) = \lambda_k^*(\mu_k^* - \mu_{k+1}^*) = 0$ for $m_2>k > m_1$. Since $\eta_{m_1}^*=\lambda_{m_1}^*=\eta_{m_2}^*=\lambda_{m_2}^*=0$ and $\tilde\mu_k = a$ ($k\leq m_1$), and $\tilde\mu_k = b$ ($k > m_2$), we also have that $\eta_k^*(\tilde \mu_k - \tilde\mu_{k+1}) = 0$ for $k\leq m_1$ and $k\geq m_2$. These confirm the complementary slackness. Furthermore,  for $m_2\geq k>m_1$,
\begin{align*}
    \frac{\partial}{\partial \mu_k} L_1(\boldsymbol{\tilde\mu}, \boldsymbol{\eta}^*) &= -2w_k(z_k-\tilde\mu_k)-\eta_{k-1}^*+\eta_k^* = -2w_k(z_k-\mu_k^*)-\lambda_{k-1}^*+\lambda_k^* = 0,
\end{align*}
and for $k \leq m_1$, 
\begin{align*}
     \frac{\partial}{\partial \mu_k} L_1(\boldsymbol{\tilde\mu}, \boldsymbol{\eta}^*) &= -2w_k(z_k-\tilde\mu_k)-\eta_{k-1}^*+\eta_k^* = -2w_kz_k+2w_ka-\eta_{k-1}^*+\eta_k^* = 0,
\end{align*}
where we use the definition of $\eta^*_k$ and $\tilde\mu_k=a$ ($k\leq m_1$) for the above equality. Similarly, for $k>m_2$, we also have $\frac{\partial}{\partial \mu_k} L_1(\boldsymbol{\tilde\mu}, \boldsymbol{\eta}^*)=0$. These lead to $\nabla L_1(\boldsymbol{\tilde\mu}, \boldsymbol{\eta}^*) = 0$. Lastly, to show $\boldsymbol{\eta}^* \geq 0$, it suffices to show $\eta^*_k \geq 0$ for $k\leq m_1-1$ and $k\geq m_2+1$ since $\eta^*_k = \lambda^*_k \geq 0$ ($m_2\geq k\geq m_1$). Note that 
\begin{align}\label{ineq:eta}
    \eta^*_{m_1-1} = -2w_{m_1}z_{m_1}+2w_{m_1}a + \eta^*_{m_1} \geq -2w_{m_1}z_{m_1} + \eta^*_{m_1} + 2w_{m_1} \mu^*_{m_1} = \lambda^*_{{m_1}-1} \geq 0,
\end{align}
where we use  facts that $\mu^*_{m_1}<a$ and $\eta_{m_1}^* = \lambda_{m_1}^*$, and the stationary condition \eqref{stationary_P1} for program \eqref{optim:original}. In addition, 
\begin{align*}
    \eta_{{m_1}-2}^* = -2w_{{m_1}-1}z_{{m_1}-1}+2w_{m_1-1}a + \eta^*_{m_1-1} \geq -2w_{m_1-1}z_{m_1-1} + \lambda^*_{m_1-1} + 2w_{m_1-1}\mu^*_{m_1-1} = \lambda^*_{m_1-2} \geq 0, 
\end{align*}
where for the first inequality we use \eqref{ineq:eta} and the  fact that $\mu^*_{m_1-1}\leq \mu^*_{m_1}<0$. For the last equality we use \eqref{stationary_P1} again. Therefore, with an induction argument we can establish $\eta_k^*\geq 0$ for $k\leq m_1-1$. Similarly, we have that $\eta_k^*\geq 0$ for $k\geq m_2+1$, which implies $\boldsymbol{\eta}^*\geq 0$. In conclusion, we show that $(\boldsymbol{\tilde\mu}, \boldsymbol{\eta}^*)$ satisfies the KKT conditions for program \eqref{optim:truncate}. Since \eqref{optim:truncate} is a convex program, we have that $\boldsymbol{\tilde\mu}$ and $\boldsymbol{\eta}^*$ should be its primal and dual optimizers, which completes the proof. \qedhere

\end{proof}
\begin{lemma}\label{lem:prob0}
    Given any constant $c>0$. The probability that there exists an  index $1\leq k \leq M$ such that  
    $
    \hat\gamma_k = c
    $ is zero. 
    
\end{lemma}
\begin{proof}[Proof of Lemma \ref{lem:prob0}]
We show that given any constant $c>0$, the probability of there existing two indices $0\leq a<b\leq M$ such that $\frac{\sigma^2}{n}\frac{d_b-d_a}{R_a-R_b} = c$ is zero. To this end, given indices $a$ and $b$, note that $(n/\sigma^2)(R_a -R_b)$ follows a noncentral chi-squared distribution with v degrees of freedom and noncentrality parameter $(n/\sigma^2)(\|\mathbf{f}-\mathbf{f}_a\|^2-\|\mathbf{f}-\mathbf{f}_b\|^2)$. Therefore, we have that $\mathbb{P}(\frac{\sigma^2}{n}\frac{d_b-d_a}{R_a-R_b} = c)=0$. The lemma can thus be obtained by noticing that the countable union of measure-zero sets has  measure zero.
\end{proof}

Now we start to prove \eqref{eq:stack}.  We begin by considering the  case that $\lambda \leq \tau$. In this case, program \eqref{loss:1} is reduced to the bounded isotonic regression program
\begin{equation} 
\begin{aligned}
& \text{minimize} \quad \sum_{k=1}^M \drk\bigg(\frac{\tau\sigma^2}{n}\frac{\ddk}{\drk} - \beta_k\bigg)^2  \\
&\text{subject to} \quad \beta_1 \leq \beta_2 \leq \cdots \leq \beta_M \leq 1.
\end{aligned}
\end{equation}
From Lemma \ref{lem:opt}, we have that its solution is 
$\{\tau\hat \gamma_k\mathbf{1}(\hat\gamma_k\leq 1/\tau)+\mathbf{1}(\hat\gamma_k> 1/\tau),\; k=1,2,\dots,M\}$, where $\hat\gamma_k$ is defined in \eqref{eq:minmax}. Then the solution to program \eqref{loss:lasso} is
\begin{align*} 
\hat\alpha_k 
&= (1-\tau\hat\gamma_k)\mathbf{1}(\hat\gamma_k \leq 1/\tau) - (1-\tau\hat\gamma_{k+1})\mathbf{1}(\hat\gamma_{k+1} \leq 1/\tau)\\
&= (1-\tau\hat\gamma_k)_{+} - (1-\tau\hat\gamma_{k+1})_{+}, \quad k = 1, 2, \dots, M.
\end{align*}      
We next consider the  case that $\lambda > \tau$.  Recall that in the proof of Lemma \ref{lem:mallow}, we assume $s_0=0,s_1,\dots,s_u=M$ correspond to the change points of $\hat{\boldsymbol{\gamma}}$ and also assume $\hat m =\tilde m = s_a$ where $0\leq a \leq M$. Let $\hat{\boldsymbol{\beta}}$ be the solution to program \eqref{loss:1} and suppose $r_q = \max\big\{k\in \{0,1,\dots,M\}: \hat\beta_k < 1\big\}$. 
Note that we want to show almost surely $\hat\alpha_k = (1-\tau\hat\gamma_k)\mathbf{1}(\hat\gamma_k < 1/\lambda) - (1-\tau\hat\gamma_{k+1})\mathbf{1}(\hat\gamma_{k+1} < 1/\lambda)$, which implies that $\hat\alpha_k=0$ and $\hat\beta_k=1$ when $k>s_a$.  
Hence, we first show that almost surely,  $r_q = \hat m = s_a$. The general idea of the proof is that if $r_q \neq s_a$, we can always construct an alternative valid solution with a smaller objective value, thereby establishing a contradiction.

Suppose $r_q > s_a$. Note that  when given $r_q$, program \eqref{loss:1} is reduced to the problem
\begin{equation} 
\begin{aligned}\label{prog:reduced}
& \text{minimize} \quad \sum_{k=1}^{r_q}\drk\bigg(\frac{\tau\sigma^2}{n}\frac{\ddk}{\drk} - \beta_k\bigg)^2 +  \frac{(\lambda-\tau)^2}{\lambda}\frac{\sigma^2}{n}\sum_{k=1}^{r_q}\ddk\mathbf{1}(\beta_k \neq 1) \\
&\text{subject to} \quad \beta_1 \leq \beta_2 \leq \cdots \leq \beta_{r_q} \leq 1.
\end{aligned}
\end{equation}
Suppose the change points that correspond to its solution $\hat{\boldsymbol{\beta}}$ are $r_0=0,r_1,\dots,r_q$ and also  suppose $s_{b-1}<r_q\leq s_{b}$ where $b > a$. 
The solution $\hat{\boldsymbol{\beta}}$ has the form $\hat\beta_k = \frac{\tau\sigma^2}{n}\frac{d_{r_l}-d_{r_{l-1}}}{R_{r_{l-1}}-R_{r_l}}$ when $r_{l-1}<k\leq r_l$ $(l \leq q)$.
Note that for $s_{b-1}<k\leq r_q$,
\begin{align*}
    \hat\beta_k = \frac{\tau\sigma^2}{n}\min_{k\leq j\leq r_q}\max_{1\leq i\leq k}\frac{d_j-d_{i-1}}{R_{i-1}-R_j} \geq \tau\hat\gamma_k \geq \tau\hat\gamma_{s_a+1}\geq \frac{\tau}{\lambda}
\end{align*}
where we use the fact that $k \geq s_{b-1}+1\geq s_a+1$. Thus, leveraging Lemma \ref{lem:prob0}, we conclude almost surely $\tau/\lambda<\hat\beta_{r_q} < 1$. Let $\Delta R_{r_l} = R_{r_{l-1}} - R_{r_l}$ and $\Delta d_{r_l} = d_{r_l}-d_{r_{l-1}}$ for any $l \leq q$.  Note that
\begin{equation}\label{ineq:rq_large}
\begin{aligned}
    &\sum_{k=r_{q-1}+1}^{r_q}\drk\bigg(\frac{\tau\sigma^2}{n}\frac{\ddk}{\drk}-\hat\beta_k\bigg)^2+\frac{(\lambda-\tau)^2}{\lambda}\frac{\sigma^2}{n}(d_{r_q}-d_{r_{q-1}}) - \sum_{k=r_{q-1}+1}^{r_q}\drk\bigg(\frac{\tau\sigma^2}{n}\frac{\ddk}{\drk}-1\bigg)^2 \\
    &=~\Delta R_{r_q}\bigg(-\hat\beta_{r_q}^2+\frac{(\lambda-\tau)^2}{\tau\lambda}\hat\beta_{r_q}-1+2\hat\beta_{r_q}\bigg)\\
    &=~-\frac{\Delta R_{r_q}}{\tau\lambda}(\tau\hat\beta_{r_q}-\lambda)(\lambda\hat\beta_{r_q}-\tau)>0.
\end{aligned}
\end{equation}
This means if we change $\hat\beta_k$ to 1 for any $r_{q-1}<k\leq r_q$, then we will get a smaller objective value of program \eqref{loss:1}. This establishes a contradiction, leading to the conclusion that $r_q \leq s_a$.

Suppose $r_q < s_a$. Again, we assume $s_{b-1}<r_q\leq s_{b}$ where $b\leq a$. If there exists $b<a$ such that $r_q = s_b$, then for any $s_b+1\leq k \leq s_{b+1}$, we consider changing $\hat\beta_k$ from 1 to $\check{\beta}= \frac{\tau\sigma^2}{n}\frac{d_{s_{b+1}}-d_{s_{b}}}{R_{s_{b}}-R_{s_{b+1}}}$.  Following a similar inequality as \eqref{ineq:rq_large}, we have that 
\begin{align*}
&\sum_{k=s_{b}+1}^{s_{b+1}}\drk\bigg(\frac{\tau\sigma^2}{n}\frac{\ddk}{\drk}-\check\beta\bigg)^2+\frac{(\lambda-\tau)^2}{\lambda}\frac{\sigma^2}{n}(d_{s_{b+1}}-d_{s_{b}}) - \sum_{k=s_{b}+1}^{s_{b+1}}\drk\bigg(\frac{\tau\sigma^2}{n}\frac{\ddk}{\drk}-1\bigg)^2\\
&=~ -\frac{ R_{s_{b}} -  R_{s_{b+1}}}{\tau\lambda}(\tau\check\beta-\lambda)(\lambda\check\beta-\tau)<0,
\end{align*}
where in the last inequality we use the fact that $\check\beta<\tau/\lambda$ since $b<a$. This means changing $\hat\beta_k$ $(s_b+1\leq k \leq s_{b+1})$ from 1 to $\check\beta$ can get a smaller objective value, which leads to a contradiction.

Thus, we can assume $s_{b-1}<r_q< s_{b}$ where $b\leq a$.  When $\hat\beta_{r_q}>\tau/\lambda$, with a similar argument as above, we can change $\hat\beta_k$ to 1 for any $r_{q-1}<k\leq r_q$ and get a smaller objective value of program \eqref{loss:1}. We thus only need to focus on the case of $\hat\beta_{r_q}\leq\tau/\lambda$, and almost surely we can assume $\hat\beta_{r_q}<\tau/\lambda$. For any $k\leq s_b$ we let $s_{c-1}<k\leq s_c$ where $c \leq b$. Then, the solution $\hat\beta_k$ to program \eqref{prog:reduced} satisfies
\begin{equation}\label{eq:hatbeta}
    \tau\hat\gamma_k=\frac{\tau\sigma^2}{n}\min_{k\leq j\leq M}\max_{1\leq i\leq k}\frac{d_j-d_{i-1}}{R_{i-1}-R_j}\leq\hat\beta_k = \frac{\tau\sigma^2}{n}\min_{k\leq j\leq r_q}\max_{1\leq i\leq k}\frac{d_j-d_{i-1}}{R_{i-1}-R_j}\leq \frac{\tau\sigma^2}{n}\max_{1\leq i\leq k}\frac{d_{s_c}-d_{i-1}}{R_{i-1}-R_{s_c}} = \tau\hat\gamma_k,
\end{equation}
which means $ \hat\beta_k=\tau\hat\gamma_k $ for any $k\leq s_{b-1}$ and also $r_l=s_l$ for any $l\leq b-1$. Let $\tilde \beta = \frac{\tau\sigma^2}{n}\frac{d_{s_b}-d_{s_{b-1}}}{R_{s_{b-1}}-R_{s_{b}}}$. We next show that if we change $\hat\beta_k$ to $\tilde \beta$ for any $s_{b-1}<k\leq s_b$, then we obtain a smaller objective value of program \eqref{loss:1}. It is worth noting that $\tilde\beta \geq \hat\beta_k$ for any $k\leq s_{b-1}$ since they are the solutions of isotonic regression, and $\tilde \beta < \tau/\lambda<1$. Hence, the replacement of $\hat\beta_k$ $(s_{b-1}<k\leq s_b)$ with $\tilde \beta$ is indeed valid. Thus, it suffices to show 
\begin{equation} \label{ineq:base}
\begin{aligned}
    &\sum_{k=r_{b-1}+1=s_{b-1}+1}^{r_q}w_k(z_k-\hat\beta_k)^2+\frac{(\lambda-\tau)^2}{\lambda}\frac{\sigma^2}{n}(d_{r_q}-d_{r_{b-1}}) + \sum_{k=r_q+1}^{s_b}w_k(z_k-1)^2\\
    &>~ \sum_{k=r_{b-1}+1=s_{b-1}+1}^{s_b}w_k(z_k-\tilde \beta)^2 + \frac{(\lambda-\tau)^2}{\lambda}\frac{\sigma^2}{n}(d_{s_b}-d_{s_{b-1}})
\end{aligned}
\end{equation}
where $ w_k = \Delta R_k > 0  $ and $ z_k = (\tau\sigma^2/n)(\Delta d_k/\Delta R_k) > 0 $, and we have constraints $\tilde \beta<\hat\beta_{r_{b}}<\cdots<\hat\beta_{r_q}<\tau/\lambda$.  By taking $\hat\beta_k = \frac{\tau\sigma^2}{n}\frac{d_{r_l}-d_{r_{l-1}}}{R_{r_{l-1}}-R_{r_l}}$ when $r_{l-1}<k\leq r_l$ $(l \leq q)$ and $\tilde \beta$ into \eqref{ineq:base}, we have that \eqref{ineq:base} is equivalent to 
\begin{equation}\label{ineq:2}
    \sum_{u=b}^q\frac{(\sum_{k=r_{u-1}+1}^{r_{u}}w_kz_k)^2}{\sum_{k=r_{u-1}+1}^{r_{u}}w_k} \leq \frac{(\sum_{k=s_{b-1}+1}^{s_{b}}w_kz_k)^2}{\sum_{k=s_{b-1}+1}^{s_{b}}w_k} + \sum_{k=r_q+1}^{s_b}w_k - \frac{\lambda^2+\tau^2}{\lambda\tau}\sum_{k=r_q+1}^{s_b}w_kz_k.
\end{equation}
For simplicity of notation, let $p=q-b$, $t_i = \frac{\sum_{k=r_{b+i-1}+1}^{r_{b+i}}w_kz_k}{\sum_{k=r_{b+i-1}+1}^{r_{b+i}}w_k}$  and $c_i = \frac{\sum_{k=r_{b+i-1}+1}^{r_{b+i}}w_k}{\sum_{k = r_{b-1}+1}^{s_b} w_k}$ where $i=0,1,\dots,p$. In addition, let $\tilde b = \frac{\sum_{k=r_q+1}^{s_b}w_kz_k}{\sum_{k=r_q+1}^{s_b}w_k}$ and $\tilde c = \frac{\sum_{k=r_q+1}^{s_b}w_k}{\sum_{k = r_{b-1}+1}^{s_b} w_k}$. Then \eqref{ineq:2} becomes equivalent to 
\begin{align}\label{ineq:3}
    c_0t_0^2+c_1t_1^2+\cdots+c_pt_p^2 \leq (c_0t_0+c_1t_1+\cdots+c_pt_p+\tilde c \tilde b)^2 + \tilde c - \frac{\lambda^2+\tau^2}{\lambda\tau}\tilde b \tilde c,
\end{align}
and we have constraints $\tilde \beta<\hat\beta_{r_{b}}<\cdots<\hat\beta_{r_q}<\tau/\lambda$ equivalent to $c_0+c_1+\dots+c_p+\tilde c = 1$ and 
\begin{equation}\label{ineq:constraint}
    c_0t_0+c_1t_1+\cdots+c_pt_p+\tilde c \tilde b < t_0 < \dots < t_p \leq \tau/\lambda.
\end{equation}
Note that since $r_q<s_b$, we have $0<\tilde c\leq 1$. Then \eqref{ineq:constraint} can be rewritten as 
\begin{equation}\label{ineq:cst2}
    \tilde b < \frac{t_0 - (c_0t_0+c_1t_1+\dots+c_pt_p)}{\tilde c} \quad \text{and} \quad t_0< t_1< \dots< t_p < \tau/\lambda.
\end{equation}
Note that we can write the right hand side of \eqref{ineq:3} as a quadratic function of $\tilde b$, which we denote by $g(\tilde b)$. Then,
\begin{equation}
    g(\tilde b) = \tilde c^2\tilde b^2 + 2\tilde c (c_0t_0+c_1t_1+\dots+c_pt_p - \frac{\lambda^2+\tau^2}{2\lambda\tau})\tilde b  + (c_0t_0+c_1t_1+\dots+c_pt_p)^2 + \tilde c.
\end{equation}
Since $t_0 < \tau/\lambda<1\leq \frac{\lambda^2+\tau^2}{2\lambda\tau}$ and $\tilde b \leq \frac{t_0 - (c_0t_0+c_1t_1+\dots+c_pt_p)}{\tilde c}$, we have that $g(\tilde b) \geq g\Big(\frac{t_0 - (c_0t_0+c_1t_1+\dots+c_pt_p)}{\tilde c}\Big)$. Then to establish \eqref{ineq:3}, it suffices to show 
\begin{align}
    c_0t_0^2+c_1t_1^2+\cdots+c_pt_p^2 &< g\Big(\frac{t_0 - (c_0t_0+c_1t_1+\dots+c_pt_p)}{\tilde c}\Big)\\
    &= t_0^2+\tilde c - \frac{\lambda^2+\tau^2}{\lambda\tau}(t_0-c_0t_0-c_1t_1-\dots-c_pt_p),
\end{align}
which is equivalent to 
\begin{align}
    \sum_{j=1}^p c_jh(t_j) < (1-c_0)t_0^2-\frac{\lambda^2+\tau^2}{\lambda\tau}(1-c_0)t_0+\tilde c,
\end{align}
where we let $h(z) = z^2-\frac{\lambda^2+\tau^2}{\lambda\tau}z$. Note that $h(z)$ is also a quadratic function and $t_0 < t_j< \tau/\lambda$ for any $j=1,2,\dots,p$. Then 
\begin{align*}
    \sum_{j=1}^p c_jh(t_j) < \sum_{j=1}^p c_jh(t_1).
\end{align*}
In addition, note that
\begin{align*}
    \sum_{j=1}^p c_jh(t_0) < (1-c_0)t_0^2-\frac{\lambda^2+\tau^2}{\lambda\tau}(1-c_0)t_0+\tilde c
\end{align*}
is equivalent to 
\begin{align*}
    \frac{\tilde c}{\lambda\tau}(\tau t_0-\lambda)(\lambda t_0-\tau) > 0.
\end{align*}
This is satisfied by the constraint that $t_0 < \tau/\lambda$, and thus we prove \eqref{ineq:3}. Overall, we have that changing $\hat\beta_k$ to $\tilde \beta$ for any $s_{b-1}<k\leq s_b$ can lead to a smaller objective value of program \eqref{loss:1}, which gives a contradiction. Therefore, we proved almost surely $r_q = s_a$. In this case, with a similar argument as \eqref{eq:hatbeta}, we have that $\hat\beta_k = \tau\hat\gamma_k$ for any $k\leq r_q$. Thus, 
\begin{equation*}
    \hat\alpha_k = (1-\tau\hat\gamma_k)\mathbf{1}(\hat\gamma_k < 1/\lambda) - (1-\tau\hat\gamma_{k+1})\mathbf{1}(\hat\gamma_{k+1} < 1/\lambda), \quad k = 1, 2, \dots, M, 
\end{equation*}
and  this is the almost surely unique solution. \qedhere
\end{proof}

\begin{proof}[Proof of Theorem \ref{thm:main}]
To prove this theorem, we need the following lemma, which is an extension of Stein's Lemma for discontinuous functions.
\begin{lemma}[Corollary 1 in \cite{tibshirani2015degrees}]\label{lem:tib}
Let $X\sim N(\mu, \sigma^2)$, where $ \mu \in \mathbb{R}$ and $ \sigma > 0 $. Let $h$ be a piecewise absolutely continuous function, with discontinuity set $\{\delta_1,\delta_2,\dots,\delta_m\}$, and derivative $h'$ satisfying $\mathbb{E}[|h'(X)|] < \infty$. Then, 
\begin{align*}
    \frac{1}{\sigma^2}\mathbb{E}[(X-\mu)h(X)] = \mathbb{E}[h'(X)] + \frac{1}{\sigma}\sum_{k=1}^m\phi\bigg(\frac{\delta_k-\mu}{\sigma}\bigg)[h(\delta_k)_+ - h(\delta_k)_-],
\end{align*}
where $\phi$ is the standard normal density function, and $h(x)_+ = \lim_{t\downarrow x}h(t)$ and $h(x)_- = \lim_{t\uparrow x}h(t)$.
\end{lemma}

    We first characterize $\mathbb{E}\big[\| \bff - \bhfbest\|^2\big]$.  By the usual population risk decomposition formula,
\begin{align*}
    \mathbb{E}\big[\|\bff-\bhfbest\|^2\big] = \mathbb{E}\big[\|\by-\bhfbest\|^2\big] - \sigma^2 + \frac{2\sigma^2}{n} \text{df}(\hat f_{\text{best}}).
\end{align*}
For simplicity of notation, we let $\tilde \beta_k = \mathbf{1}(\hat\gamma_k \geq 1/\lambda)$, where $k=1,2,\dots,M$. Then $\hat f_{\text{best}}(\bx) = \sum_{k=1}^M (\hat\mu_{k}(\bx)-\hat\mu_{k-1}(\bx))(1-\tilde\beta_k)$.  To compute $\text{df}(\hfbest)$, note that 
\begin{align}\label{eq:compute_dof}
     \text{df}(\hfbest)=\frac{1}{\sigma^2}\sum_{i=1}^n \text{cov}\left(\hfbest(\bx_i), y_i\right) = \frac{1}{\sigma^2}\sum_{i=1}^n\sum_{k=1}^M \sum_{l \in A_k\backslash A_{k-1}}\cov\pth{\biprodl \psi_l(\bx_i) (1-\tilde\beta_k), y_i}.
\end{align}
In addition, 
\begin{align*}
    &\frac{1}{n}\sum_{i=1}^n\cov\pth{\biprodl \psi_l(\bx_i)(1-\tilde\beta_k), y_i}\\
    &=~ \frac{1}{n}\sum_{i=1}^n \E\qth{\biprodl \psi_l(\bx_i)y_i (1-\tilde\beta_k)} - \frac{1}{n}\sum_{i=1}^n \E\qth{\biprodl \psi_l(\bx_i)f(\bx_i) (1-\tilde\beta_k)}\\
    &=~ \E \qth{\biprodl^2 (1-\tilde\beta_k)} - \E \qth{\biprodl\iprod{f}{\psi_l} (1-\tilde\beta_k)}\\
    &=~\cov\pth{\biprodl, \biprodl (1-\tilde\beta_k)}.
\end{align*}
Since $\biprodl \sim N(\iprod{\bff}{\boldsymbol{\psi}_l}, \sigma^2/n)$ and $\biprodl (1-\tilde\beta_k)$ is a piecewise absolutely continuous function in $\biprodl$, we can leverage Lemma \ref{lem:tib}. Specifically, we have
\begin{equation}\label{eq:fbest_df}
\begin{aligned}
    \text{df}(\hfbest)&=\frac{1}{\sigma^2}\sum_{i=1}^n\sum_{k=1}^M\sum_{l \in A_k\backslash A_{k-1}}\cov\pth{\biprodl \psi_l(\bx_i) (1-\tilde\beta_k), y_i}\\
    &=\frac{n}{\sigma^2}\sum_{k=1}^M\sum_{l \in A_k\backslash A_{k-1}} \cov\pth{\biprodl, \biprodl (1-\tilde\beta_k)}\\
    &=\sum_{k=1}^M\sum_{l \in A_k\backslash A_{k-1}}  \E \qth{\frac{\partial}{\partial \biprodl}\biprodl(1-\tilde \beta_k)}+\text{sdf}(\hfbest) \\
    &= \sum_{k=1}^M\sum_{l \in A_k\backslash A_{k-1}}\E\qth{1-\hat \beta_k + \biprodl \frac{\partial}{\partial \biprodl}(1-\tilde \beta_k)}+\text{sdf}(\hfbest) \\
    &=\E\qth{\sum_{k=1}^M\ddk(1-\tilde\beta_k)} + \text{sdf}(\hfbest),
\end{aligned}
\end{equation}
where we denote $\text{sdf}(\hfbest)$ as the population risk arising from the discontinuous points of $\hfbest$, also called the \emph{search degrees of freedom} \citep{tibshirani2015degrees}. Note that in the third equality above, it is necessary to first condition on $\iprod{\by}{\boldsymbol{\psi}_{l'}}$ for all $l' \neq l$, and subsequently absorb the conditional expectation. For the sake of simplicity in notation, we will omit this step in the following discussion.  The last equality holds since in each piece, we have $\frac{\partial}{\partial \biprodl}(1-\hat \beta_k) = 0$. Therefore, with Lemma \ref{lem:training_error} and \eqref{eq:fbest_df}, we have that
\begin{align*}
    \mathbb{E}\big[\|\bff-\bhfbest\|^2\big] &= \mathbb{E}\big[\|\by-\bhfbest\|^2\big] - \sigma^2 + \frac{2\sigma^2}{n} \text{df}(\hat f_{\text{best}})\\
    &= \E\qth{R_0+\sum_{k=1}^M\drk\pth{(1-\tilde\beta_k)^2-2(1-\tilde\beta_k)}+\sum_{k=1}^M\frac{2\sigma^2}{n}\ddk(1-\tilde\beta_k)} + \frac{2\sigma^2}{n}\text{sdf}(\hfbest) - \sigma^2\\
    &=\E\qth{R_0+\sum_{k=1}^M\drk\pth{\tilde\beta_k^2-1  }+\sum_{k=1}^M\frac{2\sigma^2}{n}\ddk(1-\tilde\beta_k)} + \frac{2\sigma^2}{n}\text{sdf}(\hfbest) - \sigma^2.
\end{align*}
We defer the characterization of $\text{sdf}(\hfbest)$ to a later point. To compute $\mathbb{E}\big[\|\bff-\bhfstack\|^2\big]$ we need the following lemma. 
\begin{lemma}\label{lem:p_ab_con}
    The following results hold.
\begin{enumerate}[(i)]
\item For any $1\leq k \leq M$ and $l \in A_M$, given $\iprod{\by}{\boldsymbol{\psi}_{l'}}$ for all $ l' \neq l $, we have $\hat\gamma_k$ is nonincreasing with $|\biprodl|$.

\item For any $1\leq k \leq M$ and $l \in A_M$, given $\iprod{\by}{\boldsymbol{\psi}_{l'}}$ for all $ l' \neq l $, we have $(1-\tau\hat\gamma_k)\mathbf{1}(\hat\gamma_k < \gamma)$ is piecewise absolutely continuous with $\biprodl$.
\end{enumerate}
\end{lemma}
\begin{proof}[Proof of Lemma \ref{lem:p_ab_con}]
    To prove part (i), suppose that $l \in A_a \backslash A_{a-1}$ where $1\leq a \leq M$. If $\hat\gamma_a<\hat\gamma_k$, then as $|\biprodl|$ increases, since $l \in A_a \backslash A_{a-1}$ and $\biprodl^2$ appears in the denominator $ R_{j-1}-R_i $ of $\hat\gamma_a = \frac{\sigma^2}{n}\min_{a\leq i\leq M}\max_{1\leq j\leq a}\frac{d_i-d_{j-1}}{R_{j-1}-R_i}$, the $\hat\gamma_a$ will keep decreasing while $\hat\gamma_k$ remains unchanged. If $\hat\gamma_a \geq \hat\gamma_k$, then as $|\biprodl|$ increases, $\hat\gamma_a$ keeps decreasing and $\hat\gamma_k$ still remains unchanged until $\hat\gamma_a$ reaches  $\hat\gamma_k$. After this point,  they remain equal and continue to decrease.

    To prove part (ii), by symmetry it suffices to show the piecewise absolutely continuity on $\biprodl \in [0,\infty)$. Let $\hat\tau_k = (1-\tau\hat\gamma_k)\mathbf{1}(\hat\gamma_k < \gamma)$.  We first consider the case of $\lambda > \tau$. Since $\hat\gamma_k$ is the solution of vanilla isotonic regression, it is continuous. By part (i), we have that there is at most one discontinuous point when $\biprodl\geq 0$. We denote it by $t$ and note that it satisfies $\hat\gamma_k = 1/\lambda$ when $\biprodl=t$.  Then from part (i) we have that $\hat\tau_k=0$ when $\biprodl\leq t$ and thus it is absolutely continuous. When $\biprodl>t$, note that we can write $\hat\gamma_k = \frac{\sigma^2}{n}\frac{d_c-d_b}{R_b-R_c}$, where $1\leq b<c\leq M$. Here $a,b$ are two change points of $\hat{\boldsymbol{\gamma}}$, and they are both functions of $\biprodl$. Since $a,b$ have only finite number of choices, we can thus rewrite $\hat\gamma_k$ as
    \begin{align}\label{eq:gamma_expansion}
        \hat\gamma_k = \sum_{i=1}^{K}\frac{\sigma^2}{n}\frac{d_{c_i}-d_{b_i}}{R_{b_i}-R_{c_i}}\mathbf{1}(t_{i-1}<\biprodl\leq t_i) + \frac{\sigma^2}{n}\frac{d_{c_{K+1}}-d_{b_{K+1}}}{R_{b_{K+1}}-R_{c_{K+1}}}\mathbf{1}(t_{K}<\biprodl),
    \end{align}
    where $1\leq b_i<c_i\leq M$ for any $i=1,2,\dots,K+1$ and $t=t_0<t_1<\cdots<t_K$. Here $b_i, c_i$ are constants independent of $\biprodl$ for any $i$, and $K$ is also a constant independent of $\biprodl$. To clarify, it is important to note that $b_i, c_i$, and $K$ should depend on $\iprod{\by}{\boldsymbol{\psi}_{l'}}$ where $l' \neq l$. However, in this particular lemma, we are provided with  $\{\iprod{\by}{\boldsymbol{\psi}_{l'}}\}_{l'\neq l}$. Therefore, for the purposes of this lemma, we treat $b_i, c_i$, and $K$ as constants. Note that $\frac{\sigma^2}{n}\frac{d_{c_i}-d_{b_i}}{R_{b_i}-R_{c_i}}$ is absolutely continuous. Hence, we have that $\hat\gamma_k$ is absolutely continuous when   $\biprodl>t$, and so is $(1-\tau\hat\gamma_k)$. Therefore, we proved $\hat\tau_k$ is piecewise absolutely continuous when $\lambda > \tau$.  In the case of $\lambda\leq \tau$, we have $\hat\tau_k = (1-\tau\hat\gamma_k)\mathbf{1}(\tau\hat\gamma_k < 1)$ is continuous. With a similar argument as before, we also have that that $\hat\tau_k$ is absolutely continuous. \qedhere
\end{proof}
Now we start to compute $\mathbb{E}\big[\|\bff-\bhfstack\|^2\big]$. Since
\begin{align*}
    \mathbb{E}\big[\|\bff-\bhfstack\|^2\big] = \mathbb{E}\big[\|\by-\bhfstack\|^2\big] + \frac{2\sigma^2}{n} \text{df}(\hfstack)  - \sigma^2,
\end{align*}
we need to compute $\text{df}(\hfstack)$. 
In the case of $\lambda > \tau$, recall that in the proof of Theorem \ref{thm:fstack}, we have that $\hat\beta_k = \tau\hat\gamma_k$ for any $k\leq s_a = \tilde m$ where $\tilde m = \max\big\{k\in \{0,1,\dots,M\}: \hat \gamma_{k} < 1/\lambda\big\}$ and $\hat\beta_k=1$ for any $k>\tilde m$.  
Using a similar argument as in the computation of $\text{df}(\hfbest)$, we have
\begin{equation} \label{eq:fstack_df}
\begin{aligned}
   \text{df}(\hfstack)&=\frac{1}{\sigma^2}\sum_{i=1}^n\sum_{k=1}^M\sum_{l \in A_k\backslash A_{k-1}}\cov\pth{\biprodl \psi_l(\bx_i) (1-\hat\beta_k), y_i}\\
    &=\frac{n}{\sigma^2}\sum_{k=1}^M\sum_{l \in A_k\backslash A_{k-1}} \cov\pth{\biprodl, \biprodl (1-\hat\beta_k)}\\
    &=\sum_{k=1}^M\sum_{l \in A_k\backslash A_{k-1}}  \E \qth{\frac{\partial}{\partial \biprodl}\biprodl(1-\hat \beta_k)}+\text{sdf}(\hfstack) \\
    &= \sum_{k=1}^M\sum_{l \in A_k\backslash A_{k-1}}\E\qth{1-\hat \beta_k + \biprodl \frac{\partial}{\partial \biprodl}(1-\hat \beta_k)}+\text{sdf}(\hfstack),
\end{aligned}
\end{equation}
where $\text{sdf}(\hfstack)$ is the search degrees of freedom of $\hfstack$. To compute $\mathbb{E}\qth{\frac{\partial}{\partial \biprodl}\hat \beta_k}$, according to \eqref{eq:gamma_expansion} and by the symmetry of the case when $\biprodl<0$, we have
\begin{equation} \label{eq:deriv}
\begin{aligned}
    &\mathbb{E}\qth{\frac{\partial}{\partial \biprodl}\hat \beta_k}\\
    &=~ \mathbb{E}\qth{\Bigg(\frac{\partial}{\partial \biprodl}\hat\beta_k\Bigg) \Bigg(\sum_{i=1}^K\mathbf{1}(t_{i-1}<|\biprodl|\leq t_i)+\mathbf{1}(t_i<|\biprodl|)\Bigg)}\\
    &=~\sum_{i=1}^K\mathbb{E}\Bigg[\Bigg(\frac{\partial}{\partial \biprodl}\frac{\tau\sigma^2}{n} \frac{d_{c_i}-d_{b_i}}{R_{b_i}-R_{c_i}}\Bigg)\mathbf{1}(t_{i-1}<|\biprodl|\leq t_i) \Bigg]\\
    & \qquad +\mathbb{E}\Bigg[\Bigg(\frac{\partial}{\partial \biprodl} \frac{\tau\sigma^2}{n}\frac{d_{c_{K+1}}-d_{b_{K+1}}}{R_{b_{K+1}}-R_{c_{K+1}}}\Bigg)\mathbf{1}(t_{K}<|\biprodl|) \Bigg].
\end{aligned}
\end{equation}
Note that in \eqref{eq:fstack_df}, we have that $l\in A_k\backslash A_{k-1}$ which means $\biprodl^2$ always appears in the denominator $ R_{j-1}-R_i $ of $\hat\gamma_k = \frac{\sigma^2}{n}\min_{k\leq i\leq M}\max_{1\leq j\leq k}\frac{d_i-d_{j-1}}{R_{j-1}-R_i}$. Thus, in \eqref{eq:deriv}, we have that $\biprodl$ should also always appear in $R_{b_i}-R_{c_i}$ for any $i=1,2,\dots,K+1$. Recall that $s_0=0,s_1,\dots,s_u=M$ correspond to the change points of $\hat{\boldsymbol{\gamma}}$. Suppose $a(k)$ is the index for which $s_{a(k)-1} < k \leq s_{a(k)}$. Here $s_0,s_1,\dots,s_u$, and $a(k)$ are all functions of $\biprodl$. Then we can continue from \eqref{eq:deriv} and obtain
\begin{align}
    \mathbb{E}\qth{\frac{\partial}{\partial \biprodl}\hat \beta_k} &= \sum_{i=1}^K\mathbb{E}\Bigg[-2\biprodl\frac{\tau\sigma^2}{n} \frac{d_{c_i}-d_{b_i}}{(R_{b_i}-R_{c_i})^2}\mathbf{1}(t_{i-1}<|\biprodl|\leq t_i) \Bigg]\\
    & \quad + \mathbb{E}\Bigg[-2\biprodl\frac{\tau\sigma^2}{n} \frac{d_{c_{K+1}}-d_{b_{K+1}}}{(R_{b_{K+1}}-R_{c_{K+1}})^2}\mathbf{1}(t_{K}<|\biprodl|) \Bigg]\\
    &=\mathbb{E}\Bigg[-2\biprodl\frac{\tau\sigma^2}{n} \frac{d_{s_{a(k)}}-d_{s_{a(k)-1}}}{(R_{s_{a(k)-1}}-R_{s_{a(k)}})^2}\mathbf{1}(t_{0}<|\biprodl|) \Bigg]\\
    &=\mathbb{E}\Bigg[-2\biprodl\frac{\tau\sigma^2}{n} \frac{d_{s_{a(k)}}-d_{s_{a(k)-1}}}{(R_{s_{a(k)-1}}-R_{s_{a(k)}})^2}\mathbf{1}(k \leq  s_a) \Bigg].
\end{align}
Continuing from \eqref{eq:fstack_df}, we have
\begin{align}
    &\text{df}(\hfstack)\\
    &=~\E\qth{\sum_{k=1}^M\ddk(1-\hat\beta_k) +2\sum_{k=1}^{s_a} \sum_{l \in A_k\backslash A_{k-1}} \biprodl^2 \frac{\tau\sigma^2}{n}\frac{d_{s_{a(k)}}-d_{s_{a(k)-1}}}{\pth{R_{s_{a(k)-1}}-R_{s_{a(k)}}}^2}  }+\text{sdf}(\hfstack)\\
    &=~\E\qth{\sum_{k=1}^M\ddk(1-\hat\beta_k) +2\sum_{l=1}^{a} \sum_{l \in A_{s_l}\backslash A_{s_{l-1}}} \biprodl^2 \frac{\tau\sigma^2}{n}\frac{d_{s_{l}}-d_{s_{l-1}}}{\pth{R_{s_{l-1}}-R_{s_{l}}}^2} }+\text{sdf}(\hfstack)\\
    &=~\E\qth{\sum_{k=1}^M\ddk(1-\hat\beta_k) + 2\sum_{l=1}^a\hat\beta_{s_l}}+\text{sdf}(\hfstack).\label{eq:compute_dof_final}
\end{align}
In the case of $\lambda\leq \tau$, let $\tilde a$ satisfy $s_{\tilde a} = \max\big\{k\in \{0,1,\dots,M\}: \hat \gamma_{k} < 1/\tau\big\}$. Since $\hat\beta_k$ is continuous in this case, it should hold that  $\text{sdf}(\hfstack) = 0$. Thus,  with a similar argument as above, we have
\begin{align}\label{df_lambda_small}
    \text{df}(\hfstack) = \E\qth{\sum_{k=1}^M\ddk(1-\hat\beta_k) + 2\sum_{l=1}^{\tilde a}\hat\beta_{s_l}}.
\end{align}
Next we start to characterize $\text{sdf}(\hfbest)$ and $\text{sdf}(\hfstack)$. Given $1\leq k \leq M$ and $l\in A_k\backslash A_{k-1}$, we have that both $\hat\beta_k$ and $\tilde\beta_k$ should have the same discontinuous points with respect to $\biprodl$. In addition, suppose $t\geq0$ is the discontinuous point. Then $-t$ is also the  discontinuous point. By part (i) of Lemma \ref{lem:p_ab_con}, we have that $\hat\gamma_k$ is nonincreasing with $|\biprodl|$ and thus, both $\hat\beta_k$ and $\tilde\beta_k$ have at most two discontinuous points. To indicate the dependence with $l$ and $k$, we denote the two discontinuous points by $t_{l,k}$ and $-t_{l,k}$. Again since $\hat\gamma_k$ is nonincreasing with $|\biprodl|$, we have that $\tilde\beta_k = \mathbf{1}(|\biprodl|\leq t_{l,k})$ and $\hat\beta_k = \mathbf{1}(|\biprodl|\leq t_{l,k}) + \tau\hat\gamma_k\mathbf{1}(|\biprodl|> t_{l,k})$. Therefore, applying Lemma \ref{lem:tib} to $h(\biprodl) = \biprodl(1-\tilde\beta_k)$,  we have
\begin{align*}
    \text{sdf}(\hfbest) = \frac{\sqrt{n}}{\sigma}\sum_{k=1}^M\sum_{l \in A_k \backslash A_{k-1}}\mathbb{E}\Bigg[t_{l,k}\Bigg(\phi\Bigg(\frac{\sqrt{n}(t_{l,k} - \mu_l)}{\sigma}\Bigg)+\phi\Bigg(\frac{\sqrt{n}(-t_{l,k} - \mu_l)}{\sigma}\Bigg)\Bigg)\Bigg]> 0,
\end{align*}
where $\mu_l = \iprod{\bff}{\boldsymbol{\psi}_l}$, c.f., \citep{mikkelsen2018best, tibshirani2015degrees, tibshirani2019excess}. Also, when $\lambda > \tau$, applying Lemma \ref{lem:tib} to $h(\biprodl) = \biprodl(1-\hat\beta_k)$, we obtain
\begin{align*}
    \text{sdf}(\hfstack) &= \frac{\sqrt{n}}{\sigma}\sum_{k=1}^M\sum_{l \in A_k \backslash A_{k-1}}\mathbb{E}\Bigg[t_{l,k}(1-\tau/\lambda)\Bigg(\phi\Bigg(\frac{\sqrt{n}(t_{l,k} - \mu_l)}{\sigma}\Bigg)+\phi\Bigg(\frac{\sqrt{n}(-t_{l,k} - \mu_l)}{\sigma}\Bigg)\Bigg)\Bigg]\\
    &=(1-\tau/\lambda)\text{sdf}(\hfbest).
\end{align*}
When $\lambda \leq \tau$, according to \eqref{df_lambda_small}, we already have that that $\text{sdf}(\hfstack) = 0$.  Now we are ready to show $\mathbb{E}\big[\|\bff-\bhfstack\|^2\big]<\mathbb{E}\big[\|\bff-\bhfbest\|^2\big]$. In the case of $\lambda > \tau$, recall that $0=s_0,s_1,\dots,s_u=M$ correspond to the change points of $\hat{\boldsymbol{\gamma}}$ and $\tilde\beta_k = 0$ when $k\leq s_a$ and 1 otherwise. Thus, we have that 
\begin{align*}
    \mathbb{E}\big[\|\bff-\bhfbest\|^2\big] &= \E\qth{R_0+\sum_{k=1}^M\drk\pth{\tilde\beta_k^2-1  }+\sum_{k=1}^M\frac{2\sigma^2}{n}\ddk(1-\tilde\beta_k)} + \frac{2\sigma^2}{n}\text{sdf}(\hfbest) - \sigma^2\\
    &=\E\qth{R_0+\sum_{l=1}^a\Delta R_{s_l}\pth{\tilde\beta_{s_l}^2-1  }+\sum_{l=1}^a\frac{2\sigma^2}{n}\Delta d_{s_l}(1-\tilde\beta_{s_l}) } + \frac{2\sigma^2}{n}\text{sdf}(\hat f_{\text{best}}) -\sigma^2,
\end{align*}
where $\tilde\beta_{s_l} =0$ for any $1\leq l\leq a$.  With a similar argument, we have
\begin{align*}
    &\mathbb{E}\big[\|\bff-\bhfstack\|^2\big]\\
    &=~ \E\qth{R_0+\sum_{l=1}^a\Delta R_{s_l}\pth{\hat\beta_{s_l}^2-1  }+\sum_{l=1}^a\frac{2\sigma^2}{n}\Delta d_{s_l}(1-\hat\beta_{s_l}) + \frac{4\sigma^2}{n}\sum_{l=1}^a\hat\beta_{s_l} } + \frac{2\sigma^2}{n}\text{sdf}(\hat f_{\text{stack}}) -\sigma^2,
\end{align*}
where $\hat\beta_{s_l} = \tau\hat\gamma_{s_l} = \frac{\tau\sigma^2}{n}\frac{d_{s_{l}} - d_{s_{l-1}}}{R_{s_{l-1}} - R_{s_{l}}}$ for any $1\leq l \leq a$. Define the quadratic function $h_l(z) =\Delta R_{s_l}z^2 -\frac{2\sigma^2}{n}(\Delta d_{s_l}-2)z $. It is evident that if $\Delta d_{s_l} \geq \min_k \Delta d_k \geq 4/(2-\tau)$, then $h_l(\hat\beta_{s_l}) \leq h_l(0)$ holds for all $1\leq l \leq a$. Also, we have proved $\text{sdf}(\hfstack) = (1-\tau/\lambda)\text{sdf}(\hfbest)<\text{sdf}(\hfbest)$, since $ \text{sdf}(\hfbest) > 0 $.  Thus, $\mathbb{E}\big[\|\bff-\bhfstack\|^2\big]<\mathbb{E}\big[\|\bff-\bhfbest\|^2\big]$ when $\lambda > \tau$. In addition, the population risk gap $ \mathbb{E}\big[\|\bff-\bhfbest\|^2\big]- \mathbb{E}\big[\|\bff-\bhfstack\|^2\big] $ is equal to
\begin{equation} \label{eq:gap}
\frac{\sigma^4\tau(2-\tau)}{n^2}\mathbb{E}\Bigg[\sum_{l=1}^a \frac{\Delta d_{s_l}(\Delta d_{s_l}-4/(2-\tau))}{\Delta R_{s_l}}\Bigg] + \frac{2\tau}{\lambda}\frac{\sigma^2}{n}\text{sdf}(\hat f_{\text{best}}).
\end{equation}
\begin{remark}
The value of $ \tau $ that maximizes this expression is
$$
\tau^* = \min\Bigg\{\frac{\mathbb{E}\big[\sum_{\ell=1}^a \Delta d_{s_l}(\Delta d_{s_l}-2)\big]+\frac{n}{\sigma^2\lambda}\text{sdf}(\hat f_{\text{best}})}{\mathbb{E}\big[\sum_{\ell=1}^a \Delta^2 d_{s_l}\big]},\; \lambda \Bigg\}.
$$
\end{remark}
Using the Cauchy-Schwarz inequality, we can further lower bound the first term in \eqref{eq:gap} by
\begin{equation}
\begin{aligned} \label{eq:improvement}
\mathbb{E}\Bigg[\sum_{l=1}^a \frac{(\Delta d_{s_l}-4/(2-\tau))^2}{\Delta R_{s_l}}\Bigg] 
&\geq~ \mathbb{E}\Bigg[\frac{\big(\sum_{l=1}^a(\Delta d_{s_l}-4/(2-\tau))\big)^2}{\sum_{l=1}^a\Delta R_{s_l}}\Bigg] \\
 & =~ 
\mathbb{E}\Bigg[\frac{(d_{s_a}-4a/(2-\tau))^2}{R_0-R_{s_a}}\Bigg].
\end{aligned}
\end{equation}
Next, note that
\begin{equation}
\begin{aligned} 
\label{eq:improvement2}
\mathbb{E}\Bigg[\frac{(d_{s_a}-4a/(2-\tau))^2}{R_0-R_{s_a}}\Bigg] 
& \geq~  \mathbb{E}\Bigg[\frac{(d_{s_a}-4s_a/(2-\tau))^2}{R_0-R_{s_a}}\Bigg] \\ 
& \geq~ \mathbb{E}\Bigg[\min_{1 \leq k\leq M}\frac{(d_k-4k/(2-\tau))^2}{R_0-R_k}\Bigg] \\
& =~ \mathbb{E}\Bigg[\Bigg(\max_{1 \leq k \leq M}\frac{(n/\sigma^2)(R_0-R_{k})}{(d_k-4k/(2-\tau))^2}\Bigg)^{-1}\Bigg] \\
& \geq~ \Bigg(\mathbb{E}\Bigg[\max_{1 \leq k \leq M}\frac{(n/\sigma^2)(R_0-R_{k})}{(d_k-4k/(2-\tau))^2}\Bigg]\Bigg)^{-1},
\end{aligned}
\end{equation}
where we apply Jensen's inequality in the last line. Our next task is to upper bound,
$$
\mathbb{E}\Bigg[\max_{1 \leq k\leq M}\frac{(n/\sigma^2)(R_0-R_k)}{(d_k-4k/(2-\tau))^2}\Bigg],
$$
which we do using a martingale argument. To this end, let $ \theta_k = \mathbb{E}[(n/\sigma^2)(R_0-R_k)] = d_k + (n/\sigma^2)(\|\mathbf{f}\|^2-\|\mathbf{f} - \mathbf{f}_k\|^2) $ and $ \Delta \theta_k = \theta_{k}-\theta_{k-1} $. Then,
\begin{equation}
\begin{aligned}
\label{ineq:exp_max_final}
&\mathbb{E}\Bigg[\max_{1 \leq k\leq M}\frac{(n/\sigma^2)(R_0-R_k)}{(d_k-4k/(2-\tau))^2}\Bigg]\\ 
&\quad \leq \mathbb{E}\Bigg[\max_{1 \leq k\leq M}\frac{|(n/\sigma^2)(R_0-R_k)-\theta_k|}{(d_k-4k/(2-\tau))^2}\Bigg] + \max_{1 \leq k\leq M}\frac{\theta_k}{(d_k-4k/(2-\tau))^2}.
\end{aligned}
\end{equation}
Let $ X_i = (n/\sigma^2)(R_{i-1}-R_i)-\Delta \theta_i $ and $ b_i = (d_i-4i/(2-\tau))^2 $, so that $ ((n/\sigma^2)(R_0-R_k)-\theta_k)/(d_k-4k/(2-\tau))^2 = (1/b_k)\sum_{i=1}^k X_i $, where the $X_i$ are mean zero and independent. Since $ b_i $ is increasing, we have by \citep[Lemma 1]{shorack1976inequalities} that, almost surely,
$$
\max_{1 \leq k\leq M}\Bigg|\frac{1}{b_k}\sum_{i=1}^k X_i\Bigg| \leq 2\max_{1 \leq k\leq M}\Bigg|\sum_{i=1}^k \frac{X_i}{b_i}\Bigg|.
$$
Since the process $ \big|\sum_{i=1}^k \frac{X_i}{b_i}\big| $ forms a positive submartingale, Doob's maximal inequality \citep[Theorem 4.4.4]{Durrett_2019} implies that
$$
\mathbb{E}\Bigg[\max_{1 \leq k\leq M}\Bigg|\sum_{i=1}^k \frac{X_i}{b_i}\Bigg|\Bigg] \leq 2\sqrt{\mathbb{E}\Bigg[\Bigg(\sum_{i=1}^M \frac{X_i}{b_i}\Bigg)^2\Bigg]} = 2\sqrt{\sum_{i=1}^M \frac{\mathbb{E}\big[X^2_i\big]}{b^2_i}}.
$$
Since $ X_i+\Delta \theta_i \sim \chi^2(\Delta d_i,\,(n/\sigma^2)\Delta \|\mathbf{f}-\mathbf{f}_i\|^2) $, where $ \Delta \|\mathbf{f}-\mathbf{f}_i\|^2 = \|\mathbf{f}-\mathbf{f}_{i-1}\|^2 -\|\mathbf{f}-\mathbf{f}_{i}\|^2 $, we have $  \mathbb{E}\big[X^2_i\big] = 2(\Delta d_i+2(n/\sigma^2)\Delta \|\mathbf{f}-\mathbf{f}_i\|^2) $. This means that
\begin{align}
\mathbb{E}\Bigg[\max_{1 \leq k\leq M}\frac{|(n/\sigma^2)(R_0-R_k)-\theta_k|}{(d_k-4k/(2-\tau))^2}\Bigg] \leq 4\sqrt{\sum_{i=1}^M \frac{2(\Delta d_i+2(n/\sigma^2)\Delta \|\mathbf{f}-\mathbf{f}_i\|^2)}{(d_i-4i/(2-\tau))^4}}. \label{ineq:expectation_max}
\end{align}
Putting everything together,
$$
\mathbb{E}\Bigg[\max_{1 \leq k\leq M}\frac{(n/\sigma^2)(R_0-R_k)}{(d_k-4k/(2-\tau))^2}\Bigg] \leq 4\sqrt{\sum_{i=1}^M \frac{2(\Delta d_i+2(n/\sigma^2)\Delta \|\mathbf{f}-\mathbf{f}_i\|^2)}{(d_i-4i/(2-\tau))^4}} + \max_{1 \leq k\leq M}\frac{\theta_k}{(d_k-4k/(2-\tau))^2}.
$$
Next, decompose
$$
\sum_{i=1}^M \frac{2(\Delta d_i+2(n/\sigma^2)\Delta \|\mathbf{f}-\mathbf{f}_i\|^2)}{(d_i-4i/(2-\tau))^4} = 
2\sum_{i=1}^M \frac{\Delta d_i}{(d_i-4i/(2-\tau))^4} + \frac{4n}{\sigma^2}\sum_{i=1}^M \frac{\Delta \|\mathbf{f}-\mathbf{f}_i\|^2}{(d_i-4i/(2-\tau))^4}.
$$
Note that
$$
\sum_{i=1}^M \frac{\Delta \|\mathbf{f}-\mathbf{f}_i\|^2}{(d_i-4i/(2-\tau))^4} \leq \sum_{i=1}^M \frac{\Delta \|\mathbf{f}-\mathbf{f}_i\|^2}{(d_1-4/(2-\tau))^4} = \frac{\|\mathbf{f}\|^2-\|\mathbf{f}-\mathbf{f}_M\|^2}{(d_1-4/(2-\tau))^4} \leq \frac{5^4\|\mathbf{f}\|^2}{d_1^4},
$$
where for the last inequality, we use the assumption that $ d_k \geq d_{k-1} + 5/(2-\tau) $, which implies  that $ d_k - 4k/(2-\tau) \geq (1/5)d_k $. In addition,
\begin{align*}
    \sum_{i=1}^M \frac{\Delta d_i}{(d_i-4i/(2-\tau))^4} &\leq 5^4\sum_{i=1}^M \frac{\Delta d_i}{d_i^4} \\
    &=  5^4\bigg( \frac{1}{d_1^3} + \sum_{i=2}^M \frac{\Delta d_i}{d^4_i}\bigg)\\
    &\leq  5^4\bigg( \frac{1}{d_1^3} + \int_{d_1}^{d_M}\frac{dx}{x^4}\bigg) \leq \frac{1250}{d_1^3}.
\end{align*}
Therefore, according to \eqref{ineq:expectation_max}, we obtain
\begin{align}
\mathbb{E}\Bigg[\max_{1 \leq k\leq M}\frac{|(n/\sigma^2)(R_0-R_k)-\theta_k|}{(d_k-4k/(2-\tau))^2}\Bigg] \leq \frac{200\sqrt{d_1+ n(\|\mathbf{f}\|^2/\sigma^2)}}{d_1^2}. \label{ineq:exp_max1}
\end{align}
Similarly,
\begin{align}
\max_{1 \leq k\leq M}\frac{d_k + (n/\sigma^2)(\|\mathbf{f}\|^2-\|\mathbf{f}-\mathbf{f}_k\|^2)}{(d_k-4k/(2-\tau))^2} \leq 25\frac{d_1 + n(\|\mathbf{f}\|^2/\sigma^2)}{d_1^2}. \label{ineq:exp_max2}
\end{align}
Thus, by substituting inequalities \eqref{ineq:exp_max1} and \eqref{ineq:exp_max2} into \eqref{ineq:exp_max_final}, we obtain
\begin{equation}
\begin{aligned}
\label{eq:improve_final}
& \mathbb{E}\Bigg[\max_{1 \leq k\leq M}\frac{(n/\sigma^2)(R_0-R_k)}{(d_k-4k/(2-\tau))^2}\Bigg] \\ & \quad \leq \frac{200\sqrt{d_1+ n(\|\mathbf{f}\|^2/\sigma^2)}}{d_1^2} + 25\frac{d_1 + n(\|\mathbf{f}\|^2/\sigma^2)}{d_1^2} \leq 225\frac{d_1+n(\|\mathbf{f}\|^2/\sigma^2)}{d_1^2}.
\end{aligned}
\end{equation}
Combining \eqref{eq:improvement}, \eqref{eq:improvement2}, and \eqref{eq:improve_final} establishes \eqref{eq:main_inequality} when $ \lambda > \tau $.
In the case of $\lambda \leq \tau$, by \eqref{df_lambda_small} we have
\begin{equation} \label{eq:lambda_small}
\begin{aligned}
    &\mathbb{E}\big[\|\bff-\bhfstack\|^2\big]\\
    &=~ \E\qth{R_0+\sum_{l=1}^{\tilde a}\Delta R_{s_l}\pth{\hat\beta_{s_l}^2-1  }+\sum_{l=1}^{\tilde a}\frac{2\sigma^2}{n}\Delta d_{s_l}(1-\hat\beta_{s_l}) + \frac{4\sigma^2}{n}\sum_{l=1}^{\tilde a}\hat\beta_{s_l} } -\sigma^2,
\end{aligned}
\end{equation}
where we recall $\tilde a$ satisfies $s_{\tilde a} = \max\big\{k\in \{0,1,\dots,M\}: \hat \gamma_{k} < 1/\tau\big\}$.  Again, we have that $h_l(\hat\beta_{s_l}) \leq h_l(0)$ for any $1\leq l \leq \tilde a$ when $\Delta d_{s_l} \geq \min_k \Delta d_k \geq 4/(2-\tau)$. Note that $\tilde a \leq a$ and for any $\tilde a<l \leq a$, according to the definition of $\tilde a$, we have $\hat\gamma_{s_l} = \frac{\sigma^2}{n}\frac{\Delta d_{s_l}}{\Delta R_{s_l}}\geq 1/\tau > 1/2$. Thus for any $\tilde a<l \leq a$, 
$$
\Delta R_{s_l}\pth{\tilde\beta_{s_l}^2-1  }+\frac{2\sigma^2}{n}\Delta d_{s_l}(1-\tilde\beta_{s_l}) = -\Delta R_{s_l}+\frac{2\sigma^2}{n}\Delta d_{s_l}>0.
$$
Therefore, we obtain $\mathbb{E}\big[\|\bff-\bhfstack\|^2\big]<\mathbb{E}\big[\|\bff-\bhfbest\|^2\big]$, and with a similar argument as \eqref{eq:improvement}, we can lower bound the population risk gap
\begin{equation} \label{eq:gap_other}
    \frac{\sigma^4\tau(2-\tau)}{n^2}\mathbb{E}\Bigg[\sum_{l=1}^a \frac{\Delta d_{s_l}(\Delta d_{s_l}-4/(2-\tau))}{\Delta R_{s_l}}\Bigg] + \frac{2\sigma^2}{n}\text{sdf}(\hat f_{\text{best}})
\end{equation}
by
\begin{align*}
    \frac{\sigma^4\tau(2-\tau)}{n^2}\mathbb{E}\Bigg[\min_{1 \leq k\leq M}\frac{(d_k-4k/(2-\tau))^2}{R_0-R_k}\Bigg] + \frac{2\sigma^2}{n}\text{sdf}(\hat f_{\text{best}}).
\end{align*}
Applying the bounds \eqref{eq:improvement}, \eqref{eq:improvement2}, and \eqref{eq:improve_final}, establishes \eqref{eq:main_inequality} when $ \lambda \leq \tau $.

\begin{remark}
The value of $ \tau $ that maximizes \eqref{eq:gap_other} is
\begin{equation*}
\tau^* = \max\Bigg\{\frac{\mathbb{E}\big[\sum_{\ell=1}^a \Delta d_{s_l}(\Delta d_{s_l}-2)\big]}{\mathbb{E}\big[\sum_{\ell=1}^a \Delta^2 d_{s_l}\big]},\; \lambda \Bigg\}. \qedhere
\end{equation*}
\end{remark}
\end{proof}
\begin{proof}[Proof of Theorem \ref{thm:l0}]
With equality \eqref{eq:lambda_small} and the assumption that $\tau=\lambda=1$, we have that 
\begin{align}
    &\mathbb{E}\big[\|\bff-\bhfstack\|^2\big]\\
    &=~ \E\qth{R_0+\sum_{l=1}^{\tilde a}\Delta R_{s_l}\pth{\hat\beta_{s_l}^2-1  }+\sum_{l=1}^{\tilde a}\frac{2\sigma^2}{n}\Delta d_{s_l}(1-\hat\beta_{s_l}) + \frac{4\sigma^2}{n}\sum_{l=1}^{\tilde a}\hat\beta_{s_l} } -\sigma^2\\
    &=~ \E\qth{R_0+\sum_{k=1}^M\drk\pth{\hat\beta_{k}^2-1  }+\sum_{k=1}^M\frac{2\sigma^2}{n}\ddk(1-\hat\beta_k) + \frac{4\sigma^2}{n}\sum_{k=1}^M\hat\beta_k\mathbf{1}(\hat\beta_k \neq \hat\beta_{k+1} )} - \sigma^2\\
    &\leq~ \E\qth{R_0+\sum_{k=1}^M\drk\pth{\beta_k^2-1  }+\sum_{k=1}^M\frac{2\sigma^2}{n}\ddk(1-\beta_k)} + \frac{4\sigma^2}{n}\E\qth{\sum_{k=1}^M\hat\beta_k\mathbf{1}(\hat\beta_k \neq \hat\beta_{k+1} )} - \sigma^2, 
\end{align}
where the last inequality holds for any deterministic or random  $\beta_1\leq\beta_2\leq\cdots\leq \beta_M \leq 1$, since $\hat{\boldsymbol{\beta}}$ is the solution of program \eqref{loss:1} and $\tau=\lambda=1$. For the second equality, we use the fact that $\hat\beta_k=1$ when $k>s_a$.  Recall that $\alpha_k = \beta_{k+1}-\beta_k \geq 0$. Note that with deterministic $\{\beta_k\}$ (and $\{\alpha_k\}$),
\begin{align*}
    \mathbb{E}\Bigg[\Bigg\|\bff-\sum_{k=1}^M\alpha_k \boldsymbol{\hat\mu}_k \Bigg\|^2\Bigg] &= \mathbb{E}\Bigg[\Bigg\|\by-\sum_{k=1}^M\alpha_k \boldsymbol{\hat\mu}_k \Bigg\|^2\Bigg] - \sigma^2 + \sum_{k=1}^M\frac{2\sigma^2}{n}d_k\alpha_k\\
    &= \E\qth{R_0+\sum_{k=1}^M\drk\pth{\beta_k^2-1  }+\sum_{k=1}^M\frac{2\sigma^2}{n}\ddk(1-\beta_k)} - \sigma^2.
\end{align*}
We thus have
\begin{align*}
    \mathbb{E}\big[\|\bff-\bhfstack\|^2\big] \leq \min_{\alpha_1\geq 0,\, \alpha_2 \geq 0,\, \dots,\, \alpha_M \geq 0}\mathbb{E}\Bigg[\Bigg\|\mathbf{f}-\sum_{k=1}^M\alpha_k \boldsymbol{\hat\mu}_k \Bigg\|^2\Bigg] + \E\qth{\frac{4\sigma^2}{n}\sum_{k=1}^M\hat\beta_k\mathbf{1}(\hat\beta_k \neq \hat\beta_{k+1} )}.
\end{align*}
Theorem \ref{thm:l0} follows from the fact that
$\sum_{k=1}^M\hat\beta_k\mathbf{1}(\hat\beta_k \neq \hat\beta_{k+1} ) = \sum_{k=1}^M\Big(1-\sum_{m=k}^M\hat\alpha_m \Big)\mathbf{1}(\hat\alpha_k \neq 0)$, and the inequality 
\begin{equation*}
    \E\qth{\sum_{k=1}^M\Bigg(1-\sum_{m=k}^M\hat\alpha_m \Bigg)\mathbf{1}(\hat\alpha_k \neq 0)} \leq \E\qth{\sum_{k=1}^M\mathbf{1}(\hat\alpha_k \neq 0)} = \E\qth{\|\hat{\boldsymbol{\alpha}}\|_{\ell_0}}. \qedhere
\end{equation*}
\end{proof}
\begin{proof}[Proof of Theorem \ref{thm:complexity}]
By setting $\lambda = \tau = 1$ in \eqref{loss:lasso}, and noting that $\|\boldsymbol{\alpha}\|_{\ell_0} = \sum_{k=1}^M\mathbf{1}(\hat\beta_k \neq \hat\beta_{k+1})$, we can establish that the solution to \eqref{loss:l0} is equivalent to the solution of the following program:
\begin{equation}
\begin{aligned}
& \text{minimize} \quad \sum_{k=1}^M w_k(z_k - \beta_k)^2 + \frac{4\sigma^2}{n}\sum_{k=1}^M \mathbf{1}(\beta_k \neq \beta_{k+1}) \\
&\text{subject to} \quad \beta_1 \leq \beta_2 \leq \cdots \leq \beta_M \leq 1,
\end{aligned}
\end{equation}
where $ w_k = \Delta R_k > 0  $ and $ z_k = (\sigma^2/n)(\Delta d_k/\Delta R_k) > 0 $.
Recall that we assume $r_q = \max\big\{k\in \{0,1,\dots,M\}: \hat\beta_k < 1\big\}$. According to Lemma \ref{lem:mallow}, to show  $\text{dim}(\hat f_{\text{stack}}) \geq \text{dim}(\hat f_{\text{best}})$, it suffices to show  $r_q \geq \tilde m$. If $r_q = M$, then $M=r_q\geq \tilde m$ is already satisfied. If $r_q<M$, we first show for any $r_q<j\leq M$, 
\begin{align}\label{ineq:larger_size1}
    \frac{\hat \beta_{r_q}+1}{2} \leq \frac{\sigma^2}{n} \frac{d_j -d_{r_q}}{R_{r_q}-R_j}. 
\end{align}
Consider another sequence $\boldsymbol{\tilde \beta}$ where $\tilde \beta_k = \hat\beta_k\mathbf{1}(1\leq k \leq r_q)+\hat\beta_{r_q}\mathbf{1}(r_q<k\leq j)+\mathbf{1}(j<k\leq M)$.   By the optimality of $\hat{\boldsymbol{\beta}}$, we have 
\begin{align*}
    &\sum_{k=1}^{M}\Delta R_k\Bigg(-\hat\beta_k +\frac{\sigma^2}{n}\frac{\Delta d_k}{\Delta R_k}\Bigg)^2 + \frac{4\sigma^2}{n}\sum_{k=1}^M \hat\beta_k\mathbf{1}(\hat\beta_k \neq \hat\beta_{k+1})\\ 
    &\leq~\sum_{k=1}^{M}\Delta R_k\Bigg(-\tilde\beta_k +\frac{\sigma^2}{n}\frac{\Delta d_k}{\Delta R_k}\Bigg)^2 + \frac{4\sigma^2}{n}\sum_{k=1}^M \tilde\beta_k\mathbf{1}(\tilde\beta_k \neq \tilde\beta_{k+1}).
\end{align*}
This implies 
\begin{align*}
    &\sum_{k = r_q+1}^j\drk\Bigg(1-\frac{\sigma^2}{n}\frac{\Delta d_k}{\Delta R_k}\Bigg)^2 + \sum_{k=j+1}^M\drk\Bigg(1-\frac{\sigma^2}{n}\frac{\Delta d_k}{\Delta R_k}\Bigg)^2\\
    &\leq~ \sum_{k = r_q+1}^j\drk\Bigg(\hat \beta_{r_q}-\frac{\sigma^2}{n}\frac{\Delta d_k}{\Delta R_k}\Bigg)^2 + \sum_{k=j+1}^M\drk\Bigg(1-\frac{\sigma^2}{n}\frac{\Delta d_k}{\Delta R_k}\Bigg)^2.
\end{align*}
After simplification, we have
\begin{align*}
    \pth{R_{r_q}-R_j}\pth{1-\frac{\sigma^2}{n} \frac{d_j -d_{r_q}}{R_{r_q}-R_j}}^2 &\leq \pth{R_{r_q}-R_j}\pth{\hat\beta_{r_q}-\frac{\sigma^2}{n} \frac{d_j -d_{r_q}}{R_{r_q}-R_j}}^2\\
    \Leftrightarrow \left|1- \frac{\sigma^2}{n} \frac{d_j -d_{r_q}}{R_{r_q}-R_j}\right| &\leq \left|\hat\beta_{r_q}- \frac{\sigma^2}{n} \frac{d_j -d_{r_q}}{R_{r_q}-R_j}\right|,
\end{align*}
which implies \eqref{ineq:larger_size1} since $\hat\beta_{r_q} < 1$.  Suppose $r_q < \tilde m$. Then
\begin{align*}
    \frac{1}{2}\geq \tilde \gamma_{\tilde m} = \max_{1\leq i\leq \tilde m}\frac{\sigma^2}{n}\frac{d_{\tilde m}-d_{i-1}}{R_{i-1}-R_{\tilde m}} \geq \frac{\sigma^2}{n} \frac{d_{\tilde m} -d_{r_q}}{R_{r_q}-R_{\tilde m}}.
\end{align*}
Noting that $\hat \beta_{r_q}>0$, we obtain a contradiction when taking $j = \tilde m$ in \eqref{ineq:larger_size1}. Therefore $r_q \geq \tilde m. $ 

To demonstrate $\sum_{k=1}^M\hat\alpha_k<1$, it suffices to show $\hat\beta_1>0$. Using an argument similar to \citep[Lemma 5.1]{gao2020estimation} to characterize the form of the solution to (weighted) reduced isotonic regression, we can show that there exists an index  $1\leq k \leq M$ for which $\hat\beta_1 = \frac{\sigma^2}{n}\frac{d_k}{R_0-R_k}>0$.  \qedhere

\end{proof}

\begin{proof}[Proof of the solution of program \eqref{loss:q-agg}]
The program proposed in \citep{bellec2018optimal, bellec2020cost} is 
\begin{equation}
\begin{aligned}\label{prog:bellec_yang}
& \text{minimize} \quad R(\boldsymbol{\alpha})
     + \frac{2\sigma^2}{n}\sum_{k=1}^M\alpha_kd_k+\eta \sum_{m=1}^M \alpha_m \frac{1}{n}\sum_{i=1}^n\Bigg(\hat\mu_m(\bx_i)-\sum_{k=1}^M \alpha_k \hat\mu_k(\bx_i)\Bigg)^2\\
&\text{subject to} \quad  \alpha_k \geq 0, \;\; k = 1, 2, \dots, M, \;\;\text{and}\;\; \sum_{k=1}^M\alpha_k=1,
\end{aligned}
\end{equation}
with $\eta = 1/2$. Thanks to a bias-variance decomposition, when $\sum_{k=1}^M \alpha_k=1$,
\begin{align}
&R(\boldsymbol{\alpha})
     + \frac{2\sigma^2}{n}\sum_{k=1}^M\alpha_kd_k+\eta \sum_{m=1}^M \alpha_m \frac{1}{n}\sum_{i=1}^n\Bigg(\hat\mu_m(\bx_i)-\sum_{k=1}^M \alpha_k \hat\mu_k(\bx_i)\Bigg)^2\\
& =~ (1-\eta)R(\boldsymbol{\alpha})
     + \frac{2\sigma^2}{n}\sum_{k=1}^M\alpha_kd_k + \eta\sum_{k=1}^M \alpha_kR_k. 
\end{align}
Thus, program \eqref{prog:bellec_yang}
corresponds to Lemma \ref{lem:bellec_yang}, which also demonstrates its equivalence to program \eqref{loss:q-agg}.
\qedhere
\end{proof}
\begin{lemma}\label{lem:bellec_yang}
    Consider the program
\begin{equation}
\begin{aligned}\label{prog:general}
& \text{minimize} \quad R(\boldsymbol{\alpha})
     + \sum_{k=1}^M\alpha_k\Bigg(\frac{\frac{2\sigma^2}{n} d_k+ \eta R_k}{1-\eta}\Bigg)\\
&\text{subject to} \quad  \alpha_k \geq 0, \;\; k = 1, 2, \dots, M.
\end{aligned}
\end{equation}

If $0<\eta<1$ and $\sum_{k=1}^M \alpha_k=1$, then the solution to \eqref{prog:general} is 
\begin{equation}
\check\alpha_k = \phi\Bigg(\frac{1-\eta/2-\check\gamma_k}{1-\eta}\Bigg) - \phi\Bigg(\frac{1-\eta/2-\check\gamma_{k+1}}{1-\eta}\Bigg),
\end{equation}
where $ \phi(z) = \min\{1, \max\{0, z\}\} $ is the clip function for $ z $ at 0 and 1, and 
\begin{equation}
    \check\gamma_1 = 0,\quad \check\gamma_k = \frac{\sigma^2}{n}\min_{k\leq i\leq M}\max_{1\leq j< k}\frac{d_i-d_j}{R_j-R_i},\quad k =2,3,\dots,M.
\end{equation}

\end{lemma}
\begin{proof}
    Recall that $\alpha_k = \beta_{k+1}-\beta_k$ and $c_k = 1-\beta_k$.  Using summation by parts and Lemma \ref{lem:training_error}, we can establish the equivalence between the
solution of \eqref{prog:general} and the solution of the following program:
\begin{equation}
\begin{aligned}\label{prog:simplify}
& \text{minimize} \quad \frac{\eta}{1-\eta}R_0c_1 + \sum_{k=1}^M\drk\Bigg( c_k-\frac{1}{2}\Bigg(\frac{2-\eta}{1-\eta}-\frac{2}{1-\eta}\frac{\sigma^2}{n}\frac{\ddk}{\drk}\Bigg)
 \Bigg)^2   \\
&\text{subject to} \quad  c_1 \geq c_2 \geq \cdots \geq c_M \geq 0.
\end{aligned}
\end{equation}
When $0<\eta<1$ and $\sum_{k=1}^M \alpha_k=1$, we note that $c_1 = \sum_{k=1}^M\alpha_k = 1$ and $\beta_1 = 1-c_1=0$. Thus, \eqref{prog:simplify} can be further reduced to \eqref{loss:q-agg}.
By Lemma \ref{lem:opt}, we have that  the solution is $\check \beta_1 = 0$ and  $\check\beta_k 
 =  1-\phi\Big(\frac{1-\eta/2-\check\gamma_k}{1-\eta}\Big)$ for $k=2,3,\dots,M$. Consequently,
\begin{equation}
\check\alpha_k = \phi\Bigg(\frac{1-\eta/2-\check\gamma_k}{1-\eta}\Bigg) - \phi\Bigg(\frac{1-\eta/2-\check\gamma_{k+1}}{1-\eta}\Bigg), \quad k = 1, 2, \dots, M.
\end{equation}
\qedhere
\end{proof}

\section{Additional Experiments} \label{app:numeric_additional}
In this section, we present additional experiments evaluating the performance of the stacked model under non-Gaussian noise and a random design. Figure \ref{fig:linear_first35_laplace} illustrates the results with Laplace noise, which has heavier tails compared to the Gaussian distribution (in fact, non-sub-Gausian). Specifically, we maintain the same setup as in Figure \ref{fig:linear_first60}, except that the noise now follows a Laplace distribution with location parameter $\mu = 0 $ and scale parameter $b = 3/\sqrt{2}$. Note that the noise variance remains $\sigma^2 = 9$, matching that of Figure \ref{fig:linear_first60}. 

Figure \ref{fig:linear_first60_random} explores a random design setting, where each entry of the data matrix $\bX$ is independently sampled from the standard Gaussian distribution. The noise follows a Gaussian distribution with mean zero and variance $\sigma^2 = 9$, again, consistent with the experiments in the main text. Instead of controlling the signal length $\|\mathbf{f}\|$, we now control $\sqrt{\E[\|\mathbf{f}\|^2]}$, which ranges from 1 to 5 in increments of 0.5. 

The results from both figures demonstrate that stacking continues to perform well under these different distributional settings.

\begin{figure}[H]
    \centering
    \begin{minipage}[b]{0.48\textwidth}
        \centering
        \includegraphics[width=\textwidth]{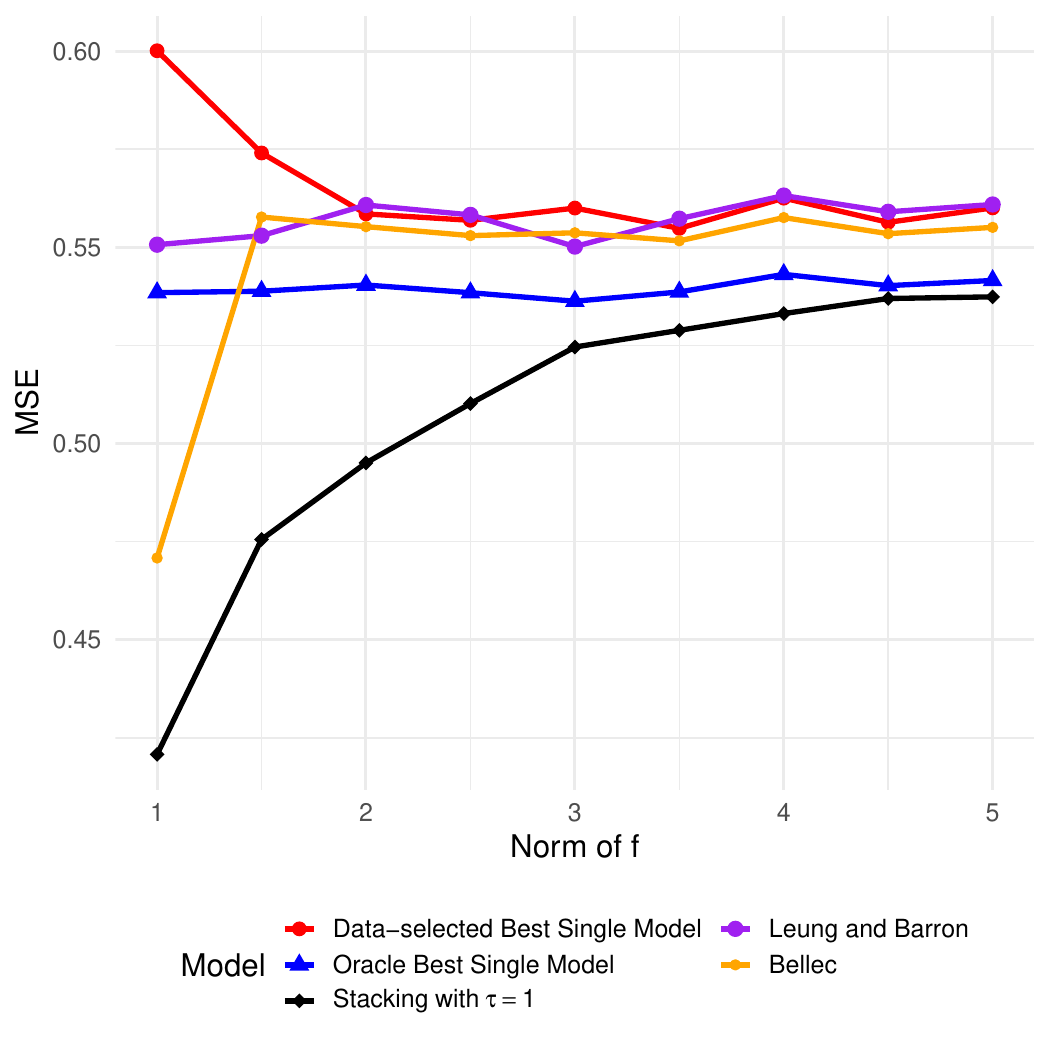}
        \caption{\small MSE comparison across different methods, where the function $f$ is linear and depends on the first 60 covariates, with noise drawn from a Laplace distribution.}
        \label{fig:linear_first35_laplace}
    \end{minipage}
    \hfill
    \begin{minipage}[b]{0.48\textwidth}
        \centering
        \includegraphics[width=\textwidth]{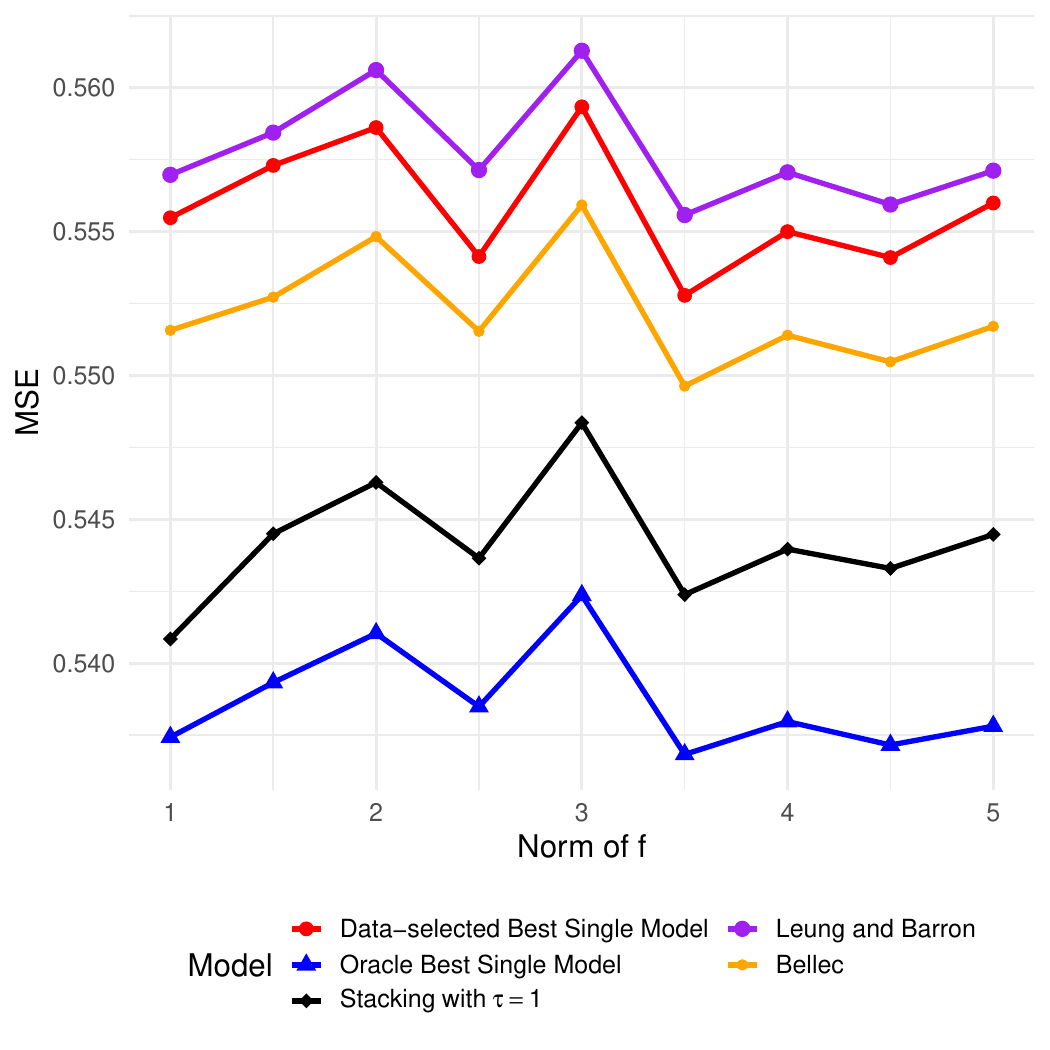}
        \caption{\small MSE comparison across different methods, where the function $f$ is linear and depends on the first 60 covariates The data matrix $\bX$ follows a random design.}
        \label{fig:linear_first60_random}
    \end{minipage}
\end{figure}

\bibliographystyle{plainnat}
\bibliography{stack.bib}

\end{document}